\documentclass[11pt,DIV=12,BCOR=0.7cm]{article}   % list options between brackets
\pdfoutput=1
\usepackage[left=2cm, right=2cm, top=2cm, bottom=3cm]{geometry}
\usepackage{graphicx,amssymb}
\usepackage{amsmath,amsfonts, amsthm}
\usepackage{bm}
\usepackage{mathtools,mathdots}
\usepackage{pgfplots}
    \pgfplotsset{
	layers/my layer set/.define layer set={
		background,
		main,
		foreground
	}{
	},
	set layers=my layer set,
}
\usepackage{subcaption}
\usepackage{algorithm,algorithmic}
\usepackage{setspace}
\usepackage{xcolor}
\usepackage{scalerel,stackengine}
\usepackage{mathtools}
\usepackage{tikz}
\usetikzlibrary{matrix,chains,positioning,decorations.pathreplacing,arrows}
\usetikzlibrary{positioning,calc}
\usepackage{url}
\usepackage{wrapfig, multirow}
\usepackage{tablefootnote}
\usepackage{textcomp}
\usepackage[hidelinks]{hyperref}

\newtheorem{lemma}{Lemma}

\newtheorem{theorem}{Theorem}[section]

\newcommand*\diff{\mathop{}\!\mathrm{d}}
\DeclarePairedDelimiterX{\infdivx}[2]{\big(}{\big)}{%
	#1\;\delimsize\|\;#2%
}
\newcommand{\infdiv}{\text{KL}\infdivx}
\DeclarePairedDelimiter{\norm}{\lVert}{\rVert}

\DeclareMathOperator{\E}{\mathbb{E}}
\newcommand\equalhat{\mathrel{\stackon[1.5pt]{=}{\stretchto{%
				\scalerel*[\widthof{=}]{\wedge}{\rule{1ex}{3ex}}}{0.5ex}}}}
			
\newcommand{\expnumber}[2]{{#1}\mathrm{e}{#2}}

\newcommand\Tstrut{\rule{0pt}{2.6ex}}         % = `top' strut
\newcommand\Bstrut{\rule[-0.9ex]{0pt}{0pt}}   % = `bottom' strut

\usepackage[normalem]{ulem}

\newenvironment{keywords}
{\bgroup\leftskip 20pt\rightskip 20pt \small\noindent{\bf Keywords:} }%
{\par\egroup\vskip 0.25ex}

\makeatletter
\DeclareRobustCommand{\Udots}{%
	\vcenter{\offinterlineskip
		\halign{%
			\hbox to .8em{##}\cr
			\hfil.\cr\noalign{\kern.2ex}
			\hfil.\hfil\cr\noalign{\kern.2ex}
			.\hfil\cr}%
	}%
}
\makeatother

\title{Robust Learning of Parsimonious Deep Neural Networks}
\author{Valentin Frank Ingmar Guenter\thanks{Graduate Student, Email:
vguenter@uci.edu}\ \ and\ \ Athanasios~Sideris\thanks{Professor, Email:
asideris@uci.edu}\\ \\  Department of Mechanical and
Aerospace Engineering,\\ University of California, Irvine,\\
Irvine, CA, 92697}

\begin{document}
\maketitle
\begin{abstract}%
We propose a simultaneous learning and pruning algorithm capable of identifying and eliminating irrelevant structures in a neural network during the early stages of training. Thus, the computational cost of subsequent training iterations, besides that of inference, is considerably reduced.
Our method, based on variational inference principles using Gaussian scale mixture priors on neural network weights, learns the variational posterior distribution of Bernoulli random variables multiplying the units/filters similarly to adaptive dropout.
Our algorithm, ensures that the Bernoulli parameters practically converge to either $0$ or $1$, establishing a deterministic final network.
We analytically derive a novel hyper-prior distribution over the prior parameters that is crucial for their optimal selection and leads to consistent pruning levels and prediction accuracy regardless of weight initialization or the size of the starting network. We prove the convergence properties of our algorithm establishing theoretical and practical pruning conditions.
We evaluate the proposed algorithm on the MNIST and CIFAR-10 data sets and the commonly used fully connected and convolutional LeNet and VGG16 architectures. The simulations show that our method achieves pruning levels on par with state-of the-art methods for structured pruning, while maintaining better test-accuracy and more importantly in a manner robust with respect to network initialization and initial size.
\end{abstract}
\section{Introduction}
Deep learning has gained tremendous prominence during recent years as it has been shown to achieve outstanding performance in a variety of Machine Learning tasks, such as natural language processing, object detection, semantic image segmentation and reinforcement learning \cite{girshick_fast_2015, noh_learning_2015, silver_mastering_2017}. However, Deep Neural Networks (DNNs) can be unnecessarily overparametrized resulting in excessive computational requirements both during training and inference, often making it infeasible to deploy them on systems with limited computational resources, e.g., low-powered mobile devices. Neural network pruning \cite{han_learning_2015,blalock_what_2020} has been one technique used to reduce the size of such over-parameterized Neural Networks (NNs) by appropriately eliminating weights and/or nodes from the network while essentially maintaining its prediction accuracy.
Thus, neural network pruning techniques can be divided into structured and unstructured methods. Unstructured or weight-based pruning aims to remove individual weights and therefore, connections in the neural network. Typically, a saliency or importance score, e.g., a weight’s magnitude is assigned to each weight \cite{han_learning_2015} and  network pruning is carried out by permanently removing weights with a score below a certain threshold \cite{han_learning_2015, guo_dynamic_2016, lecun_optimal_1990}. On the other hand, structured pruning aims to remove entire structures from the neural network. Thus, in fully connected NNs, the goal is to reduce the number of neurons in its hidden layers, while in convolutional NNs to remove entire filters. This makes structured pruning particularly attractive as it prunes the network to a smaller counterpart and allows accelerated inference with standard deep learning libraries; in contrast, the practical acceleration of DNNs achieved with unstructured or weight pruning may be limited by poor cache locality and jumping memory access caused by the ensuing random connectivity of the network and require specialized hardware \cite{wen_learning_2016}. In addition, when the dimensionality of a feature vector in a DNN has a specific interpretation, structured pruning methods, which effectively learn this dimension, inherently are more appropriate over unstructured pruning methods.

\paragraph{Review of structured pruning methods:} Next, we review in more detail some structured pruning methods since our proposed algorithm belongs to this category. Filter Thresholding(FT, \cite{li_pruning_2017})uses the Euclidean norm of the fan-out weight vector in fully-connected layers or the Frobenius norm of each kernel matrix in convolutional layers as a score to prune the corresponding unit or filter from the network; this method simply keeps the units/filters with the largest norm until the desired sparsity level is met, however,there is no active mechanism to promote the weights of the network to approach zero other than the typical $\mathcal{L}_2$ regularization terms in the cost function.
SoftNet \cite{he_soft_2018} is similar to FT \cite{li_pruning_2017} but using the $\mathcal{L}_1$ norm of the weights corresponding to a unit in a layer as the score and also allowing weights previously pruned, i.e., set to zero, to become non-zero again during its fine-tuning scheme. ThiNet \cite{luo_thinet:_2017} iteratively prunes feature maps or units in the network for which removal leads to the least absolute error in the pre-activation of the subsequent layer and until a desired pruning level is achieved.
Provable Filter Pruning (PFP, \cite{liebenwein_pruning_19}) uses the empirical sensitivity of \cite{baykal_data-dependent_2019} to construct importance sampling
distributions over feature maps and an iterative sampling scheme to prune the feature maps/units while keeping the output of the layer close to its original unpruned value.

Bayesian variational techniques together with sparsity promoting priors have also been employed for neural network pruning \cite{Nalisnick_2015, Louizos_2017, molchanov2017variational, srinivas_generalized_2016}.  Such methods typically employ Gaussian scale mixture priors, which are zero mean normal probability density functions (pdf's) with variance (scale) given by another random variable (RV). A notable prior in this class is the spike-and-slab prior, in which only two scales are used. In these methods,
%\paragraph{Related Work:}
\cite{Louizos_2017} uses variational inference by postulating parametrized posteriors over the weights and the scale RVs and finds the parameters of such posteriors by maximizing the Evidence Lower Bound (ELBO); weights are eliminated for which with scale RVs for which the posterior variance of their scale RVs
is greater than the posterior mean by a set threshold.
 To eliminate groups of weights, e.g., units in a layer, a common scale RV is used for the group. In their approach discrete distributions on the scale variables such as Bernoulli are problematic due to the need to apply the reparametrization trick \cite{kingma_variational_2015} in maximizing the ELBO.

In \cite{Nalisnick_2015}, an identifiable parametrization of the multiplicative noise
is used where the RVs are the product of NN weights and the scale variables and the
scale variables themselves. Then, estimates of the scale variables are obtained using the
Expectation-Maximization algorithm to maximize a lower bound on the log-likelihood.
A Bernoulli distribution on the scale variables is used, although the authors state that in principle, this is not justified since gradients with respect to (w.r.t.) discrete variables are considered. Also, the expectation step is accomplished using samples from the posterior of the NN weights via Monte Carlo (MC) simulations. Weights with a sum of posterior variance and mean less than a set threshold are pruned,
 a criterion that allows weights with high variance to survive in distinction to \cite{Louizos_2017}.

\cite{molchanov2017variational} also uses multiplicative Gaussian noise on the network weights, which receive an improper log-scale uniform distribution as prior. It postulates a normal posterior on the weights and proceeds to maximize the ELBO over its parameters. Due to the choice of prior, an approximation to the Kullback-Leibler-Divergence (KL-Divergence) term of the ELBO is necessary. The authors discuss that their method can be sensitive to weight initialization and extra steps must be taken to assure good initialization.

In \cite{Gal2016Uncertainty} Dropout \cite{hinton_improving_2012} has been interpreted as imposing a spike-and-slab pdf on the weights of a NN. Generalized Dropout \cite{srinivas_generalized_2016} also implicitly employs a spike-and-slab pdf on the product of scales and NN weights by placing a Bernoulli prior on the scale RVs and a normal prior on the NN weights. The posterior on the scale RVs is Bernoulli, sharing parameters with the prior and a Beta hyper-prior on these parameters is used; NN weights are treated as RVs with a delta posterior placed on them, leading to technical difficulties in gradient calculations.
Due to the choice of the parameters of the prior distributions to coincide with those of the variational posteriors, the method becomes highly sensitive to weight initialization and requires sensitive tuning to produce good results.

\paragraph{Overview of proposed algorithm:} In this paper, we propose a structured simultaneous learning and pruning algorithm capable of robustly identifying redundant or
irrelevant units/filters in a neural network during the early stages of the training process.
Thus, our algorithm allows to reduce the computational complexity of subsequent training
iterations besides doing so during inference. In our approach, we also use multiplicative
Bernoulli noise, i.e., we propose a spike-and-slab method and use variational techniques.
This can be interpreted as adding unit-wise dropout to the network, where each unit possesses
its own and adaptive dropout rate. Thus, during each training iteration via backpropagation,
a unit is active only with a certain probability and a different subnetwork is
realized formed from the active units. Based on variational Bayesian principles, we learn
parameterized posterior probability distributions for the Bernoulli random variables, which
determine the active units.

%\paragraph{Unique Contributions:}
While belonging to the same family of scale mixture of Gaussian priors for variational inference methods, our approach differs from the ones described above in several important aspects summarized next.
\begin{enumerate}
\item
A Bernoulli variational posterior is imposed only on the scale RVs. The weights receive a Gaussian prior as usual and are obtained via Maximum a posteriori (MAP) estimation, which is equivalent to having a delta posterior on the weights \cite{srinivas_generalized_2016} but with the technical difficulties and ill-posed KL-divergence terms circumvented.
The Bernoulli variational posterior on the scale RVs in our method naturally leads to deterministic, smaller networks and renders our method significantly different from \cite{Louizos_2017, molchanov2017variational}, which employ continuous scale pdf's.
\item
Expectations w.r.t. the scale RV's are explicitly computed in terms of the variational posterior Bernoulli parameter of a single scale RV and efficiently approximated w.r.t. the other scale RV’s using mini-batch samples
within the backpropagation algorithm. This formulation allows explicit calculation of the gradients
w.r.t. the Bernoulli parameter of each scale prior in terms of the error function values for the two forward passes corresponding to two discrete values of the scale RV. By using a 1st order Taylor series approximation, we avoid this computation, which can be significant for large network, and this turns out to be equivalent to the straight-through estimator
proposed on empirical grounds in \cite{bengio_estimating_2013} for approximating gradients of stochastic units w.r.t. the parameters of the noise pdf. This estimator is used also in \cite{srinivas_generalized_2016} but our analysis provides theoretical justification for the good properties of this approximation.
\item
We introduce a hyper-prior over the parameters of the Bernoulli prior on the scale RVs as in \cite{srinivas_generalized_2016}, which in effect makes our spike-and-slab approach multiscale.
However, unlike in \cite{srinivas_generalized_2016}, that poses a Beta hyper-prior, we
analytically derive the optimal form of this hyper-prior based on carefully examining the gradient of the ELBO w.r.t. the parameters of the posterior distribution on the scale RVs. This novel hyper-prior forces these parameters to either $0$ or $1$ in a manner that avoids premature pruning. Thus, our approach results in deterministic compressed networks that outperform state-of-the-art results. In addition our method is robust, i.e., insensitive w.r.t. the initial choice of weights and/or network structure. We remark that our analysis offers new insights on how to construct sparsifying priors, while in the literature such priors are typically selected based on their generic properties and the ability to perform needed computations.
\item
Most pruning approaches in the general framework considered apply pruning after training the full network, thus saving resources only during the prediction phase. In distinction, our approach effects simultaneous training and pruning and it can reduce training times and/or expended energy by $3-$ to $4$-fold in the case of training the VGG16 architecture \cite{simonyan_deep_2015} on the CIFAR-10 data set \cite{Krizhevsky_09}. Successful simultaneous training and pruning is challenging since aggressive pruning can save training resources but result in a poor network. Therefore, it is imperative to assure that units are pruned as soon as possible but not earlier. To this end, we develop analytical results establishing a region of attraction around $0$ for the dynamics of the posterior parameters on the scale RVs and the NN weights. That is, we provide provable conditions under which units that converge to their elimination cannot recover and survive.
\item
Our method does not require much more computation per training iteration than standard backpropagation. In particular there is no need for expensive MC simulations. In fact, because of the discrete variational posterior, during forward-, backpropagation and the update-phase only a part of the network is active, leading in principle to additional computational savings \cite{graham2015efficient}. Furthermore, in our method there is only one hyper-parameter besides the standard backpropagation ones to tune, which has a clear interpretation in terms of trade-offs between network compression and accuracy.
\end{enumerate}

The remainder of the paper is arranged as follows. In Section~\ref{sec:Model_definition}, we present the statistical modeling that forms the basis of our simultaneous pruning and learning approach. In Section~\ref{sec:Variational_Approach}, we give the analysis for fitting the parameters of the model and  in Section~\ref{sec:hyper_prior_selection}, we detail the optimal design of the hyper-prior distribution responsible for the robustness properties of our algorithm.  We provide convergence results supporting the pruning process in Section~\ref{sec:algorithm_convergence}, \ref{app:A}  and~\ref{app:B}.
%and some practical guidelines resulting from the theoretical analysis in Section~\ref{sec:practical_conditions}.
In Section~\ref{sec:algorithm}, we summarize the proposed simultaneous learning/pruning algorithm and in Section~\ref{sec:simulations}, we  present simulations on standard machine learning problems and comparison with state-of-the-art structured and variational inference-based pruning approaches. Section~\ref{sec:conclusions} concludes the paper.

\paragraph{Notation:}  We use the superscript $l$ to distinguish parameters or variables of the $l$th layer of a neural network and the subscript $j$ to denote dependence on the $j$th unit/filter in this layer. However, we also use the subscript $i$ to denote dependence on the $i$th sample in the given data set. We use $n$ in indexing such as $x(n)$ to denote iteration count. Also $\norm{\cdot}$ denotes the Euclidean norm of a vector, $M^\top$ transpose of a matrix (or vector), and $\E[\cdot]$ taking expectation with respect to the indicated random variables. Other notations are introduced in the following before their use.

\section{Problem Formulation}\label{sec:Model_definition}
We develop our algorithm for fully connected feed-forward neural networks (NNs) and comment on the simple extension to Convolutional Neural Networks (CNNs) in Section~\ref{sec:algorithm}. We consider NNs with $L-1$ hidden layers realizing mappings \mbox{$\hat y = NN(x;W, \Xi)$} through the hierarchy of functions
\begin{align}\label{eq:NN_def}
	\zeta^l = W^l \cdot \big(z^l\odot \xi^l\big), \quad z^{l+1} = a_l(\zeta^l), \quad  \xi^l \sim Bernoulli(\pi^l)
\end{align}
for $l=1,\dots, L$. Here, $z^{l+1}$ is the output of the $l^{th}$ layer and $W^1,\dots,W^L$ are the weights of the NN denoted collectively by $W$. To absorb the additive bias usually used in neural network architectures into the weight matrices $W^l$, we extend the features $z^l$ by a constant $1$ and each $W^l$ with a last row $\begin{bmatrix} 0 &\dots& 0 & 1 \end{bmatrix}$ of appropriate dimension. Then, we have $z_1 = \begin{bmatrix} x^\top & 1 \end{bmatrix}^\top$ and $z^{L+1} = \hat y$. The $\xi^l$ are (vector) Bernoulli random variables with parameters $\pi^l$, $l=1,\dots, L$ corresponding to the features $z^l$ and are denoted collectively by $\Xi$.  Therefore, each component of $\xi^l$  attains values $0$ or $1$.  The symbol $\odot$  denotes element-wise multiplication. The activation functions $a_l(\cdot)$ are assumed to be continuously differentiable nonlinearities applied to each component of their input to ensure this property for the overall NN mapping; this is the case for the bipolar sigmoidal and smoothed versions of the ReLU activation functions. The activation for the last component of $z^l$, $l=1,\ldots,L$ is taken to be the identity function to ensure that the last element of $z^{l}$, used for injecting the bias, is always equal to $1$. In the output layer, we use the linear or the softmax activation function for regression or multi-class classification problems, respectively.  We refer to each element of a hidden layer output $z^l$ as a unit. In a fully connected layer, the number of weights and therefore, the computational power needed to evaluate the layer, is proportional to the sum of the products of the number of units in two successive layers. The modeling of the NN in \eqref{eq:NN_def} with the additional Bernoulli RVs $\xi^l$ leads to the well-known dropout formulation  introduced and used in \cite{Srivastava_2014_Dropout} to regularize deep neural networks. Here, however, we aim to learn the appropriate number of units in each hidden layer and assume individual dropout RVs for each unit with learnable parameters $\pi^l$ so that the posterior distributions of the $\xi^l$  dictate the structure of the network.

\subsection{Statistical Model}
Given a data set $\mathcal{D}=\lbrace (x_i,y_i)\rbrace_{i=1}^N$, where $x_i \in \mathbb{R}^m$ are input patterns and $y_i \in \mathbb{R}^n$ are the corresponding target values, the goal is to learn the weights $W$ and appropriate parameters $\Pi=\{\pi^1,\ldots,\pi^L\}$ for the prior distributions of the RVs $\Xi$ for the NN in \eqref{eq:NN_def}.
\\
For regression tasks, the samples $(x_i,y_i)$ are assumed to be  drawn independently from the Gaussian statistical model
\begin{align}\label{eq:model_regression}
	p(Y\mid X,W,\Xi) \sim \prod_{i=1}^N \mathcal{N}\left(y_i;NN(x_i,W,\Xi),\frac{1}{\tau}\right)
\end{align}
with a variance hyper-parameter $\frac{1}{\tau}$. For a $K$-class classification problem, we assume the categorical distribution and write the statistical model as
\begin{align}\label{eq:model_classification}
%		p(Y\mid X,W,\Xi) \sim \prod_{i=1}^N \prod_{k=1}^{K} \left( NN(x_i,W,\Xi)\right)^{y_{i,k}},
		p(Y\mid X,W,\Xi) \sim \prod_{i=1}^N \prod_{k=1}^{K} (\hat y_{i,k})^{y_{i,k}},
\end{align}
where $y_{i,k}$ and $\hat y_{i,k}$ are the $k$th components of the one-hot coded target and NN output vectors, respectively for the $i$th sample point.  We note that $X,Y$ denote collectively the given patterns and targets, respectively and not the underlying RVs. Therefore, \eqref{eq:model_regression} and \eqref{eq:model_classification} give the conditional likelihood of the targets when the NN model is specified.\\

We have already assumed Bernoulli prior distributions $p(\xi^l\mid \pi^l)$ for the RVs $\xi^l$'s in $\Xi$. We will also treat the weights in $W$ as RVs with Normal prior distributions $p(W^l\mid \lambda)$
and choose the overall prior distribution to factorize as follows
\begin{align*}
	p(W,\Xi\mid \Pi, \lambda) = p(W\mid \lambda)\cdot p(\Xi \mid \Pi) = \prod_{l=1}^L p(W^l\mid \lambda) \cdot p(\xi^l\mid \pi^l)
\end{align*}
with
\begin{align}\label{eq:prior_distributions}
\begin{split}
p(W^l\mid \lambda) &\sim \mathcal N \big(0, \lambda^{-1} \textbf{I}\big)\\
p(\xi^l \mid \pi^l ) &\sim Bernoulli\big(\pi^l\big) \propto ({\pi^l})^{\xi^l} (1-\pi^l)^{1-\xi^l}.
\end{split}
\end{align}
Next, we place a hyper-prior $p(\pi^l \mid \Gamma)$
on the variables $\pi^l$, the selection of which will be crucial for the success of the proposed algorithm and will be discussed in detail in Section~\ref{sec:hyper_prior_selection}.
Combining the NN statistical model with the prior distributions gives
\begin{align*}%\label{eq:Bayes_model}
%\begin{split}
	p(Y\mid X, W, \Xi)\cdot p(W\mid \lambda)\cdot p(\Xi \mid \Pi )\cdot p( \Pi\mid \Gamma) &= p(Y,W,\Xi,\Pi\mid X,\lambda,\Gamma) \\
	&=p(W,\Pi,\Xi\mid Y,X,\lambda,\Gamma)\cdot p(Y\mid X,\lambda,\Gamma)
%\end{split}
\end{align*}
from which, we obtain the posterior
\begin{align*}
%\begin{split}
	p(W,\Pi \mid Y,X,\lambda,\Gamma) &= \int p(W,\Pi,\Xi \mid Y,X,\lambda,\Gamma) \diff \Xi\\
	&= \int p(Y\mid X,W,\Xi)\cdot p(W\mid \lambda)\cdot p(\Xi\mid \Pi)\cdot p(\Pi\mid \Gamma)\cdot \frac{1}{p(Y\mid X,\lambda,\Gamma)}\diff \Xi \\
	&\propto p(Y\mid X,W,\Pi)\cdot p(W\mid\lambda)\cdot p(\Pi\mid\Gamma).
%\end{split}
\end{align*}
Maximum a Posteriori (MAP) estimation selects the parameters $\pi$ and $W$  by maximizing the $\log$ posterior:
\begin{align}\label{eq:obj_pi}
	\max_{W,\Pi} \left[\log p(Y\mid X,W,\Pi) +\log p(W\mid \lambda) + \log p(\Pi \mid \Gamma)\right].
\end{align}
The authors of \cite{corduneanu_2001_model_sel} considered a similar formulation for the problem of choosing the appropriate number of components in a Gaussian mixture model. Their approach assumes a fixed number of potential components in the mixture and proceeds with optimal estimates of the mixing coefficients. Then, components with small mixing coefficients are eliminated from the mixture.
Here, we have a much more complex model in the form of NN with potential components being the NN units present in the initial structure. The goal is to set up an optimization problem, such that units not sufficiently contributing to the network performance on the given task are automatically identified and eliminated. Corresponding to the mixing coefficients in \cite{corduneanu_2001_model_sel}, we have the parameters $\Pi$ of the prior distributions \eqref{eq:prior_distributions} on the network's RVs $\Xi$. When during optimization, an element of $\pi^l$ converges to a small value near zero, it practically signals the removal of the associated unit from the NN.

\paragraph{Remark:} The Gaussian scale mixture priors on NN weights used in \cite{Louizos_2017, Nalisnick_2015, molchanov2017variational} lead to zero mean normal pdf's on the weights $W^l\sim\mathcal{N}\left(0,(\xi^l\odot\xi^l)\cdot\lambda^{-1}\textbf{I}\right)$ with scale $\xi^l \sim p(\xi^l\mid\pi^l)$. They are also referred to as multiplicative noise priors since we can equivalently express $W^l = \bar W^l \cdot \xi^l$ with $\bar W^l \sim \mathcal{N}\left(0,\lambda^{-1}\textbf{I}\right)$. While our approach falls into this same group using $p(\xi^l\mid\pi^l)\sim Bernoulli\big(\pi^l\big)$, an important distinction to existing works and key to the robustness properties of our method is the design of a novel hyper-prior on the $\pi^l$ parameters effectively transforming the $2$-scale Bernoulli prior into a multiscale, more flexible one (see Section~\ref{sec:flattening_hyper_prior}.)

\section{Model Fitting via a Variational Approximation Approach}\label{sec:Variational_Approach}
Given the data set, we aim to find parameter values $W, \Pi$ and in addition infer the posterior distribution on $\Xi$.
For deep neural networks the exact $p(Y \mid X,W,\Pi)$ in  \eqref{eq:obj_pi} is intractable, necessitating the approximation of the posterior distribution on $\Xi$. To this end, we employ variational methods and introduce the variational posterior
\begin{align*}
	q(\Xi \mid \Theta) \sim Bernoulli(\Theta),
\end{align*}
which factorizes to individual Bernoulli distributions for each component of $\Xi$ corresponding to a unit of the NN with parameters denoted collectively as $\Theta$. Next, we derive the Variational lower bound (see e.g. \cite[Chapter 10]{bishop_pattern_2006})
on \mbox{$p(Y \mid X,W,\Pi)$} in  \eqref{eq:obj_pi} as follows:
\begin{align*}%\label{eq:ELBO_relation}
%\begin{split}
	\log p(Y\mid X,W,\Pi) &=\log \int  p(Y,\Xi \mid X,W,\Pi) \diff \Xi%\\[5pt]
	= \log \int q(\Xi\mid \Theta) \frac{ p(Y ,\Xi\mid X,W,\Pi)}{q(\Xi\mid \Theta)}  \diff \Xi\\[5pt]
	&\geq \int q(\Xi\mid \Theta)  \log \frac{ p(Y ,\Xi\mid X,W,\Pi)}{q(\Xi\mid \Theta)}  \diff \Xi \equalhat \mathcal{L}(\Theta,W,\Pi),	
%\end{split}
\end{align*}
where Jensen's inequality was used. We proceed by replacing the intractable evidence $\log p(Y\mid X,W,\Pi)$ in \eqref{eq:obj_pi} with the Variational lower bound $\mathcal{L}(\Theta,W,\Pi)$  and obtain our main objective:
\begin{align}\label{eq:Main_obj}
	&\max_{W,\Pi,\Theta} \mathcal{L}(\Theta,W,\Pi) + \log p(W\mid \lambda)  + \log p(\Pi\mid \Gamma).
\end{align}
Further, we can express
\begin{align}\label{eq:ELBO}
\begin{split}
\mathcal{L}(\Theta,W,\Pi) &=\int q(\Xi\mid \Theta)  \log \frac{ p(Y \mid X,W,\Xi)p(\Xi\mid \Pi)}{q(\Xi\mid \Theta)}  \diff \Xi\\
&= \int q(\Xi\mid \Theta) \log p(Y\mid X,W,\Xi) \diff \Xi - \infdiv{q(\Xi\mid \Theta)}{p(\Xi\mid \Pi)}
\end{split}
\end{align}
and substituting \eqref{eq:ELBO} in \eqref{eq:Main_obj} gives:
\begin{align}\label{eq:Objective_with_KL}
\begin{split}
	\max_{W,\Pi,\Theta} \int q(\Xi\mid \Theta) \log p(Y\mid X,W,\Xi) \diff \Xi  - \infdiv{q(\Xi \mid \Theta)}{p(\Xi\mid \Pi)}
	+ \log p(W\mid \lambda)  + \log p(\Pi\mid \Gamma).
\end{split}
\end{align}
The integral term in (\ref{eq:Objective_with_KL}) represents an estimate of the loss over the given samples; its maximization leads to parameters $W$ and $\Theta$ that explain the given data set best by placing all probability mass of $q(\Xi\mid  \Theta)$ where $p(Y\mid X,W,\Xi)$ is highest. Maximizing the second part in (\ref{eq:Objective_with_KL}) or equivalently minimizing the KL-Divergence between the variational posterior and the prior distributions on $\Xi$  keeps the approximating distribution close to our prior. Finally, maximizing the last two terms in (\ref{eq:Objective_with_KL}) serves to reduce the complexity of the network by driving the update probabilities of nonessential units and their weights to zero.

Following  \cite{corduneanu_2001_model_sel}, we first define an optimization problem that can be solved explicitly for the parameters $\Pi$ in terms of $\Theta$ once the prior $p(\Pi\mid\Gamma)$ has been specified. This step amounts to a type~II MAP estimation.
Specifically, to maximize the main objective \eqref{eq:Objective_with_KL} with respect to $\Pi$, we can equivalently minimize
\begin{align}\label{eq:maxobj_pi}
	&\min_\Pi  \infdiv{q(\Xi\mid \Theta)}{p(\Xi\mid \Pi)} - \log p(\Pi\mid \Gamma).
\end{align}
Note that  \eqref{eq:maxobj_pi} factorizes over the units of the NN.
Then, for each unit $j$ in layer $l$ and after dropping indices from $\pi_j^l$, $\xi_j^l$ and $\theta_j^l$ for brevity, and using the definition of the discrete KL-Divergence,  we obtain the scalar minimization problem
\begin{align}\label{eq:opt_pi_star_full}
\begin{split}
	&\min_\pi  q(\xi=0)\log\frac{q(\xi=0)}{p(\xi=0\mid \pi)} + q(\xi=1)\log\frac{q(\xi=1)}{p(\xi=1\mid \pi)} - \log p(\pi\mid \Gamma)\\
	\Leftrightarrow\quad &\min_\pi J(\pi)=  (1-\theta) \log \frac{1-\theta}{1-\pi} + \theta \log\frac{\theta}{\pi} - \log p(\pi\mid \Gamma).
\end{split}
\end{align}
Given $p(\pi\mid \Gamma)$ and $\theta$, we define the optimum parameter $\pi^\star(\theta)$ as follows:
\begin{align}\label{eq:opt_pi_star}
\begin{split}
\pi^\star=\pi(\theta) = & \arg \min_\pi J(\pi) \\
&s.t. \quad \epsilon_1 \leq \pi \leq 1-\epsilon_2.
\end{split}
\end{align}
We restrict $\epsilon_1 \leq \pi \leq 1-\epsilon_2$ where $0<\epsilon_1, \epsilon_2 \ll 1$ to keep the $\log$ terms in  \eqref{eq:opt_pi_star_full} out of singularity and avoid infinite gradients during the learning process. The solution to this optimization problem for two particular choices of $p(\pi\mid \Gamma)$ is further discussed in Section~\ref{sec:hyper_prior_selection}. Furthermore, the regularization term on the network's weights \eqref{eq:Objective_with_KL} comes from the $\log$-probability of the Gaussian prior on $W$  and is
\begin{align}\label{eq:W-prior}
	\log p(W\mid \lambda) = \frac{N_W}{2}\log \lambda - \frac{N_W}{2} \log(2\pi)  -W^\top W \frac{\lambda}{2}
\end{align}
where $N_W$ is the total number of weights in the initial network and  we consider $W$ as a vector consisting of the network weights.

Let us express the negative of the integral term in  \eqref{eq:Objective_with_KL} as
\begin{align}\label{eq:Cost}
	C(W,\Theta)=\E_{\Xi\sim  q(\Xi\mid \Theta)}\big[-\log p(Y\mid X,W, \Xi)\big].
\end{align}
%\subsection{Learning the Network's Weights}\label{sec:LearningWeights}
Then by replacing each $\pi$ with the corresponding optimal $\pi^\star= \pi^\star(\theta)$ and using \eqref{eq:W-prior}, \eqref{eq:Cost} in \eqref{eq:Objective_with_KL}, dropping constant terms and  switching from maximization to equivalent minimization,  our optimization objective becomes
\begin{align}\label{eq:L_train_obj}
	\min_{W,\Theta} L(W,\Theta) \equalhat C(W,\Theta) + \frac{\lambda}{2}W^\top W +\infdiv{q(\Xi \mid \Theta)}{p(\Xi\mid \Pi^\star)} -\log p(\Pi^\star \mid \Gamma).
\end{align}
We employ (stochastic) gradient descent to minimize $L(W,\Theta)$ in \eqref{eq:L_train_obj}
with respect to the NN weights $W$ and variational parameters $\Theta$, as exact solutions are intractable. The required gradients are computed in the following subsections. In distinction to \cite{srinivas_generalized_2016}, we do not treat $W$ and $\Pi$ as part of the variational distribution and thus we avoid using the dirac-pdf on $W$, previously criticized in \cite{molchanov2017variational}, and the simplifying assumptions restricting $\pi=\theta$.

\subsection{Learning the Network's Weights $W$}\label{sec:LearningWeights}
From \eqref{eq:L_train_obj}, we obtain the gradient of $L(W,\Theta)$ with respect to weights $W^l$ of the $l^{th}$ layer as
\begin{align}\label{eq:grad_w}
	\frac{\partial L(W,\Theta)}{\partial W^l} =\frac{\partial C(W,\Theta)}{\partial W^l} + \lambda W^l.
\end{align}
To calculate the first term in \eqref{eq:grad_w}, we first estimate the expectation over the RVs $\Xi$ in \eqref{eq:Cost} with a sample $\hat\Xi\sim Bernoulli(\Theta)$ and since typically the data set is large, we also approximate the $\log$-likelihood in \eqref{eq:Cost} with a sub-sampled data set (mini-batch) $\mathcal{S} = \lbrace (x_i,y_i)\rbrace_{i=1}^B$ of size $B$. Thus, we have
\begin{align}\label{eq:MiniBatch_approx}
	C(W,\Theta) \approx \frac{N}{B}\sum_{i=1}^{B}\big[-\log p(y_i\mid x_i,W,\hat\Xi)\big]
\end{align}
and from the backpropagation algorithm (see for example \cite{bishop_pattern_2006}, Chapter~5)
\begin{align*}%\label{eq:C_grad}
	\frac{\partial C(W,\Theta)}{\partial W^l} \approx
\frac{N}{B}\sum_{i=1}^{B} {\delta_i^l}\cdot\big(z_i^l\odot \hat \xi^l_i\big)^\top,
\end{align*}
with $\delta_i^l$, a column vector, denoting the gradient estimate of the partial cost \linebreak${-\log p(y_i\mid x_i,W,\hat\Xi)}$ with respect to the activation input $\zeta_i^l$ in layer $l$  (see  \eqref{eq:NN_def}).
Then, the gradient of our objective $L(W,\Theta)$ with respect to the network weights $W^l$ in layer $l$ can be approximated by:
\begin{align}\label{eq:grad_m}
	\frac{\partial L(W,\Theta)}{\partial W^l} \approx \frac{N}{B}\sum_{i=1}^{B} {\delta_i^l}\cdot\big(z_i^l\odot \hat \xi^l_i\big)^\top +\lambda W^l.
\end{align}

\subsection{Learning the pruning parameters $\Theta$}\label{sec:Estimators}
To learn the parameters $\Theta$ of the variational posterior $q(\Xi\mid \Theta)$ simultaneously with learning the network weights $W$, we also need the derivative of the objective $L(W,\Theta)$ with respect to each component in $\Theta$.
In the following, we denote the scalar RV $\xi_j^l$ just as $\xi$ and all remaining RVs $\xi_{j'}^{l'}, l'\neq l$ or $j'\neq j$ as $\bar \Xi$.
%, i.e. $\bar \xi = \lbrace\xi_i^l, i\neq j\rbrace$.
We apply the same notation for $\theta_j^l$, i.e., we denote the scalar $\theta_j^l$ as $\theta$ and all remaining $\theta_{j'}^{l'}, l'\neq l$ or $j'\neq j$ as $\bar \Theta$.\\
Next using the fact that $q(\xi\mid \theta)$ is a Bernoulli distribution with parameter $\theta$, we express $C(W,\theta, \bar\Theta)$ from \eqref{eq:Cost} in a form that exposes its dependence on $\theta$ as follows:
\begin{align}\label{eq:C_def}
%\begin{split}
&C(W,\theta,\bar\Theta) = \E_{\Xi\sim  q(\Xi\mid \Theta)}\big[-\log p(Y\mid X,W, \Xi)\big]\\
&\hspace{-0.1cm}= \theta \underbrace{\E_{\bar \Xi\sim  q(\bar \Xi\mid \bar \Theta)}\hspace{-0.05cm}\big[\hspace{-0.13cm}-\log p(Y\mid X,W,\xi=1,\bar \Xi)\big]}_{\equalhat C_1(W,\bar\Theta)} \hspace{-0.03cm}+(1-\theta)\underbrace{\E_{\bar \Xi\sim  q(\bar \Xi\mid \bar \Theta)}\hspace{-0.05cm}\big[\hspace{-0.13cm}-\log p(Y\mid X,W,\xi=0, \bar \Xi)\big]}_{\equalhat C_0(W,\bar\Theta)}.\nonumber
%\end{split}
\end{align}
We can also write \eqref{eq:L_train_obj} in a similar way as
\begin{align*}%\label{eq:loss_single_unit_dep}
	 L(W,\theta,\bar\Theta)= \theta C_1 + (1-\theta)C_0 +(1-\theta)\log \frac{1-\theta}{1-{\pi}^\star} + \theta \log \frac{\theta}{{\pi}^\star}-\log p(\pi^\star \mid \Gamma) + \bar L(W,\bar \Theta),
\end{align*}
and readily compute the derivative with respect to $\theta$ as
\begin{align}\label{eq:grad_theta}
	\frac{\partial L(W,\Theta)}{\partial \theta} = C_1 - C_0 + \log \left(\frac{\theta(1-\pi^\star(\theta))}{(1-\theta)\pi^\star(\theta)}\right).
\end{align}
Note that $\frac{d L}{d \theta} = \frac{\partial L}{\partial \theta} + \frac{\partial L}{\partial \pi^\star}\frac{d\pi^\star}{d\theta} =\frac{\partial L}{\partial \theta} $, since $\pi^\star$ selected as in \eqref{eq:opt_pi_star} is such that either $\frac{\partial L}{\partial \pi^\star} = 0$ in case that $\pi^\star$ is a function of $\theta$ or $\frac{\partial \pi^\star}{\partial \theta}= 0$ in case that $\pi^\star= \epsilon_1\ {\rm or}\  \epsilon_2 = const.$

In \eqref{eq:grad_theta},  $C_1-C_0$ is the difference in the total cost with the particular unit switched on and off. A large negative value for the difference $C_1-C_0$ indicates high importance of the unit for the performance of the network.  Then in minimizing $L(W,\theta)$ via gradient descent, a negative value for $C_1-C_0$ will make the corresponding $\theta$ grow, leading to the unit being switched on and its weights being optimized more frequently. On the other hand, $C_1-C_0 > 0$ leads to smaller $\theta$ and less frequent optimization of the weights of the corresponding unit with a consequence that the weight decay term will drive these weights to zero.
Also the regularization $\log$-term in the gradient expression in \eqref{eq:grad_theta} is negative for $\theta<\pi^\star$, positive for $\theta>\pi^\star$, zero at $\theta=\pi^\star$ and $\mp \infty$ at the boundaries $\theta=0,1$, respectively. Therefore, this term drives $\theta$ towards the solution of $\theta=\pi^\star(\theta)$, the choice made in \cite{srinivas_generalized_2016}, which clearly relies only on the regularization term ignoring the data reflected in the difference $C_1-C_0$.

\subsection{$C_1-C_0$ Approximations}
Calculating the expected values $C_1$, $C_0$ is computationally not feasible for large networks, since requires evaluating the network with all combinations of units switched on and off. In the following, we discuss how to approximate the difference $C_1-C_0$ required in \eqref{eq:grad_theta} by methods based on (i) a first order Taylor approximation, (ii) a continuous relaxation of the discrete Bernoulli distribution (CONCRETE distribution), and (iii) sampling methods.
While sampling methods can provide unbiased estimates, in our experience such estimates suffer from either high variance or high computational effort in comparison to the other methods, which, however, result in biased estimates. Nevertheless, we show that the first order Taylor approximation can provide unbiased estimates in the limit when weights corresponding to a unit approach zero and can be then be particularly effective for units that are about to be pruned.
\subsubsection{Taylor Series Approximation of $C_1-C_0$}\label{sec:Taylor}
\begin{figure}[t!]
	\centering
	\resizebox{.7\linewidth}{!}{\includegraphics{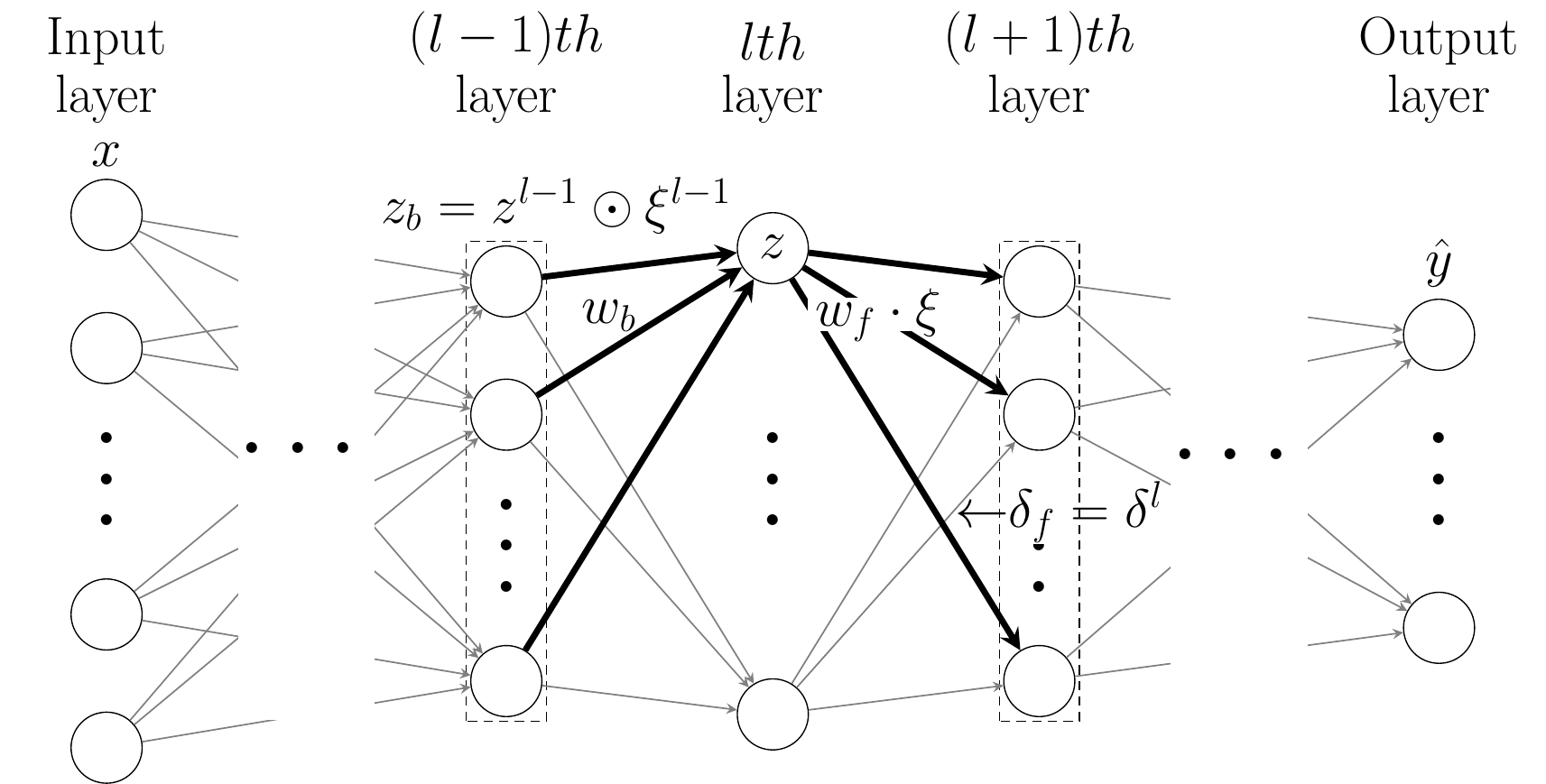}}
	\caption{Visualization of the fully connected neural network structure. $w_b$ and $w_f$ are the fan-in and fan-out weights corresponding to a single unit with output $z=z_j^l$ and selection RV $\xi=\xi_j^l$. $z_b=z^{l-1}\odot\xi^{l-1}$ is the componentwise product of the output of the $l-1$ layer and corresponding selection RVs and $\delta_f=\delta^l$ denotes the gradient of the cost with respect to the activation input $\zeta^l$ in layer $l$.}\label{fig:NN_struct}
\end{figure}
Let us first extend the simplifying notation
$z^l_j=z$, $\xi^l_j=\xi$ with $\delta^l_i=\delta_f$, $W^l_{:j} = w_f$, $z^{l-1}\odot \xi^{l-1} =z_b$, and $(W^{l-1}_{j:})^\top = w_b$
as illustrated in Figure \ref{fig:NN_struct}. Next, it will be convenient to defer taking expectation with respect to the unit's RV $\xi$ in the definition of the cost $C(W,\Theta)$ and consider it as a function also of $\xi$. Thus, we denote {(now with bar notation):
\begin{align}\label{eq:Costxi}
	\hspace{-0.25cm}\bar C(W,\Theta,\xi)=\E_{\bar\Xi\sim  q(\bar\Xi\mid \bar\Theta)}\hspace{-0.1cm}\big[\hspace{-0.1cm}-\log p(Y\mid X,W,\xi, \bar \Xi)\big]
&\hspace{-0.1cm}\approx N \E_{\substack{\mathcal{D}\sim p(\mathcal{D})\\ \bar\Xi\sim  q(\bar\Xi\mid \bar\Theta)}}\hspace{-0.1cm}\left[\hspace{-0.05cm}-\log p(y_i\mid x_i,W,\xi, \bar \Xi)\right]\hspace{-0.05cm}.
\end{align}
In the last expression of \eqref{eq:Costxi}, we view $\bar C(W,\Theta,\xi)$ as an expectation over the data $\cal D$ and the RVs $\bar\Xi$.
%over the data $\cal D$ in addition to the RVs $\bar\Xi$.
In what follows, we denote dependence of $\bar C$ only on the variables of interest. Let us consider $\tilde w_f\equalhat w_f\cdot\xi$ and
the Taylor expansion of $\bar C(\tilde w_f)$ at \mbox{a point $\tilde w_f^*$:}
\begin{align*}
	\bar C(\tilde w_f) = \bar C(\tilde w_f^*) + \frac{\partial \bar C(\tilde w_f)}{\partial \tilde w_f}\Bigr|_{\tilde w_f^*} (\tilde w_f - \tilde w_f^*) + \mathcal{O}(\norm{\tilde w_f - \tilde w_f^*})).
\end{align*}
Then, let $\tilde w_f^* = 0$ (for $\xi=0$) and evaluate at $\tilde w_f=w_f$ (for $\xi=1$) or alternatively let $\tilde w_f^*=w_f$  and evaluate at $\tilde w_f = 0$ to obtain in both cases
\begin{align}\label{eq:taylor}
% 	C_0 = C_1 - \frac{\partial C(\xi)}{\partial \xi}\Bigr|_{\xi= 1}  + h.o.t \quad \text{or}\quad  C_1 = C_0 + \frac{\partial C(\xi)}{\partial \xi}\Bigr|_{\xi= 0}   + h.o.t.
  C_1 - C_0 = w_f^\top\frac{\partial \bar C(\tilde w_f)}{\partial \tilde w_f}\Bigr|_{\tilde w_f= 0\, \text{or}\, w_f}   + \mathcal{O}(\norm{ w_f }))
\end{align}
since clearly it holds $C_0\equiv \bar C(W,\Theta,\xi=0)$ and $C_1\equiv \bar C(W,\Theta,\xi=1)$.
Furthermore, we obtain from \eqref{eq:Costxi} via backpropagation
\begin{align*}
\frac{\partial \bar C(\tilde w_f)}{\partial \tilde w_f}  \approx N \E_{\bar\Xi, \mathcal{D}}\left[z\cdot \delta_f \right] = N\E_{\bar\Xi,\mathcal{D}}\left[a_{l-1}(w_b^\top z_b) \delta_f \right]
\end{align*}
where $\delta_f=\delta_f(\xi)$ depends on $\xi$.
At each iteration, we estimate the above expectation and from \eqref{eq:taylor} we approximate the difference $C_1-C_0$ via sample means over the mini-batch data and samples $\hat\Xi$ in which $\hat\xi=0$ or $\hat\xi=1$ as follows:
\begin{align}\label{eq:Taylor_estimator}
	C_1-C_0  \approx w_f^\top \frac{\partial \bar C(\tilde w_f)}{\partial \tilde w_f} \approx  \frac{N}{B}\sum_{i=1}^B a_{l-1}(w_{b}^\top z_{b,i}) \delta_{f,i}^\top w_{f}.
\end{align}
The estimator in \eqref{eq:Taylor_estimator} is biased after dropping the higher order terms in the Taylor series expansion. However, for small weight values and as $w_f\rightarrow 0$ the estimator becomes unbiased. This perspective can explain the good success of the Taylor series approximation over the NN weights in other works such as \cite{bengio_estimating_2013}, since typically NN weights tend to be small. Indeed, the straight-through estimator in \cite{bengio_estimating_2013} approximates the expected gradient through stochastic binary neurons by back-propagating through the hard threshold function as if it had been the identity function and was found empirically to produce good results. Although motivated differently, the straight-through estimator leads to the same approximation of the difference $C_1-C_0$ as our Taylor approximation.

\subsubsection{CONCRETE Approximation}
Instead of using the exact Bernoulli distribution on $\xi$, we can relax it to its corresponding CONCRETE distribution \cite{Maddison_CONCRETE}.
\begin{align*}
	\xi (\theta, u)= 1-\sigma\left(\rule{0cm}{0.5cm}\right.\underbrace{\frac{1}{t} \left[\log(1-\theta) - \log  \theta + \log u - \log(1-u)\right]}_{h(\theta,u)}\left.\rule{0cm}{0.5cm}\right),
\end{align*}
with $\sigma(\cdot)$ the sigmoidal function, $1>t>0$ the approximation temperature and $u\sim \mathcal{U}(0,1)$ a uniformly distributed random variable.
Its derivative is given by
\begin{align*}
\frac{\partial \xi}{\partial \theta} = \sigma\big(h(\theta,u)\big) \bigg(\sigma\big(h(\theta,u)\big)-1\bigg) \frac{1}{t} \frac{1}{\theta(1-\theta)}.
\end{align*}
Then, we can express
\begin{align*}
C(W,\Theta)=\E_{\Xi\sim  q(\Xi\mid \Theta)}\hspace{-0.05cm}\big[\hspace{-0.1cm}-\log p(Y\mid X,W,\Xi)\big]  \hspace{-0.1cm}\approx \hspace{-0.1cm}\int_0^1 \hspace{-0.05cm}\underbrace{\E_{\bar\Xi\sim q(\bar\Xi\mid \bar\Theta)}\hspace{-0.05cm}\big[\hspace{-0.1cm}-\log p(Y\mid X,W,\xi(\theta,u), \bar \Xi)\big]}_{=g(\xi(\theta,u))} \diff u
\end{align*}
and obtain the approximation
\begin{align}\label{eq:concrete_estimator}
	\frac{\partial C}{\partial \theta} \approx \int \frac{\partial g(\xi)}{\partial \xi} \frac{\partial \xi}{\partial \theta} \diff u.
\end{align}
The above expectations over $u$ and $\bar \Xi$ can be approximated by samples, for example a single sample estimate.
In \cite{gal2017concrete}, this approximation is used to learn layer-wise Dropout probabilities  with parameter $t=0.1$. Our experience suggests that setting $t=0.1$ also works well for estimating $\frac{\partial C}{\partial \theta}$, striking a good compromise between approximation quality and sampling efficiency.

\subsubsection{Sampling Method}
Consider a Monte Carlo estimation over data $\mathcal{D}$ and $\Xi$ of the cost in \eqref{eq:MiniBatch_approx}. We employ the same sample $\hat\Xi$ for all mini-batch samples, which results in computational savings \cite{graham2015efficient}. Then, each mini-batch computation gives an unbiased estimate of $C_1$ or $C_0$ depending on whether $\hat \xi=1$ or $0$:
\begin{align}\label{eq:Chats_sampling}
	C_{1,0} \approx \hat C_{1,0} =\frac{N}{B} \sum_{i=1}^B -\log p(y_i \mid x_i,W,\hat \xi=\lbrace1,0\rbrace, \hat{\bar\Xi}).
\end{align}
Subsequently, the network is evaluated a second time switching the value of $\hat \xi$ while the sample $\hat{\bar\Xi}$ is kept the same. This approach requires  $M+1$ evaluations of the network to obtain estimates of $C_1$ and $C_0$ for all units, where $M$ is the number of hidden units in the network.
Then, the estimator for the difference is:
\begin{align}\label{eq:Sampling_estimator}
\begin{split}
	C_1-C_0 \approx& \frac{N}{B} \sum_{i=1}^B -\log p(y_i \mid x_i,W,\hat \xi=1, \hat{\bar\Xi}) + \frac{N}{B} \sum_{i=1}^B - \log p(y_i \mid x_i,W,\hat \xi=0, \hat{\bar\Xi})\\
	&=\frac{N}{B} \sum_{i=1}^B \log \left(\frac{p(y_i \mid x_i,W,\hat \xi=0, \hat{\bar\Xi})}{p(y_i \mid x_i,W,\hat \xi=1, \hat{\bar\Xi})}\right).
\end{split}
\end{align}
While this unbiased estimator was found to have relatively low variance in practice and the $M+1$ evaluations of forward passes of the network can be done in parallel, it may still be computationally unattractive for large networks even though the required number of forward network evaluations is linear in the number of hidden units of the network.\\
To trade-off computational efficiency vs. variance of the estimator, one may consider switching the $\hat \xi$ values of $m$ units at a time.
This reduces the number of required forward network evaluations approximately $m$-times and the resulting
unbiased estimator is
\begin{align*}
	C_1-C_0 \approx \frac{N}{B} \sum_{i=1}^B -\log p(y_i \mid x_i,W,\hat \xi=1, \hat{\bar\Xi})\frac{q(\hat\xi=1,\hat{\tilde\Xi} \mid \Theta)}{q(\hat \xi=1, \hat{\bar \Xi}\mid \Theta)} 
	 +\log p(y_i \mid x_i,W,\hat \xi=0, \hat{\bar\Xi})\frac{q(\hat\xi=0,\hat{\tilde\Xi} \mid \Theta)}{q(\hat \xi=0, \hat{\bar \Xi}\mid \Theta)},
\end{align*}
where $\hat{\tilde\Xi}$ contains the same values as $\hat {\bar \Xi}$ except $m-1$ of them are switched.
Unfortunately, the weights $\frac{q(\hat\xi=0,\hat{\tilde \Xi} \mid \Theta)}{q(\hat \xi=0, \hat{\bar \Xi}\mid \Theta)}$  can be much greater than one even for modest values of $m$ resulting in a high variance estimator. Therefore, this approach was not pursued further.

\subsubsection{Hybrid Approach}
We can utilize both, the Taylor approximation and the Sampling method to estimate the differences $C_1-C_0$ of the units in the network.
Due to the fact that the Taylor approximation estimate is asymptotically unbiased as the weights corresponding to a unit approach zero, we expect it to provide good estimates whenever the weights are small. For large DNNs trained with $\mathcal{L}_2$-regularization, generally small weights are observed.
The idea is then to use the cheap Taylor approximation for most of the network's units and use the Sampling method only for the units with the larger weights for which the Taylor estimate may have a larger bias.  One can choose the number of units to apply the more expensive sampling method based on the available computational budget during training.

\section{Selection of the Hyper-prior Distribution $p(\pi\mid\Gamma)$}\label{sec:hyper_prior_selection}
In this section, we discuss the choice of the hyper-prior $p(\pi\mid \Gamma)$ and its implications to the optimization problem \eqref{eq:opt_pi_star} and the learning of parameters $W$ and $\Theta$.
To promote strong pruning, $p(\pi\mid \Gamma)$ should be larger for values of $\pi$ close to 0.
To avoid committing what units to prune or keep before they had a change to learn, it is desirable to keep the log-term in (\ref{eq:grad_theta}) as flat as possible such that early changes in $\theta$ do not lead to stronger/weaker regularization and cause the pruning/survival of the unit prematurely. Also, flatness of this term helps to decouple the effect of how $\theta$ is initialized from the pruning of the network, thus eliminating the need to tune such initialization carefully, which is an issue with \cite{molchanov2017variational,srinivas_generalized_2016}.
Furthermore, the value of the $\log$-term in \eqref{eq:grad_theta} acts as a threshold for the first term ($C_1-C_0$), which reflects the ``usefulness'' of the corresponding unit in representing the data, for increasing the updating rate $\theta$ of the unit and thus allowing it to further learn and survive. In the following, we first discuss the choice of Beta distribution for $p(\pi\mid \Gamma)$, which is typically used when representing random variables taking values in $[0,\ 1]$. Then, we analytically derive a novel prior $p(\pi\mid \Gamma)$ to exactly meet the flatness requirement  on the log-term in (\ref{eq:grad_theta}) from the previous considerations.

\subsection{Beta hyper-prior}\label{sec:beta_hyper_prior}
Selecting $p(\pi \mid \Gamma=\{\alpha,\beta\})\sim Beta(\alpha,\beta)$ leads to
\begin{align*}
	J^{Beta}(\pi) = \underbrace{\big(1-\alpha-\theta\big)}_{=A}\log \pi +\underbrace{\big(\theta-\beta\big)}_{=B}\log( 1-\pi) + const.
\end{align*}
in  \eqref{eq:opt_pi_star} and obtaining the optimum prior on $\xi$ boils down to solving:
\begin{align*}
\min_\pi  J^{Beta}(\pi) = A\log \pi + B \log (1-\pi) \\
\quad s.t. \quad \epsilon_1 \leq \pi \leq 1-\epsilon_2
\end{align*}
for each unit of the neural network separately with potentially different $\alpha$ and $\beta$ in each layer of the network.
%Here we do not use separate values $\epsilon_1, \epsilon_2$ but rather set $\epsilon=\epsilon_1=\epsilon_2$.
In the remainder of this work, we are led to the choice $\beta >1$ and
we can summarize the choice of $\pi^\star$ in this case as follows:
\begin{align}\label{eq:pibeta}
	\pi^\star = \begin{cases}
		\frac{\theta_1+\alpha-1}{\alpha+\beta-1}=\epsilon_1\, , \quad &0\leq \theta \leq \theta_1
\equalhat{(1-\epsilon_1)(1-\alpha)+\epsilon_1\beta} \\[5pt]
		\frac{\theta+\alpha-1}{\alpha+\beta-1}\, , \quad &\theta_1< \theta < \theta_2
\equalhat{\epsilon_2(1-\alpha)+(1-\epsilon_2)\beta} \\[5pt]
		\frac{\theta_2+\alpha-1}{\alpha+\beta-1}=1-\epsilon_2\, , \quad &\theta_2\leq \theta \leq 1
	\end{cases},
\end{align}
where the last case in \eqref{eq:pibeta} is vacuous if $\theta_2>1$.
 Using this formulation of $\pi^\star$,
we can write the gradient of the loss function in \eqref{eq:grad_theta} more explicitly as
\begin{align}\label{eq:grad_theta_beta}
\begin{split}
	\frac{\partial L(W,\theta)}{\partial \theta}
	%&=\red{C_1-C_0+\frac{\partial E(\theta)}{\partial \theta}+\log\left(\frac{\beta-\theta_c}{\theta_c+\alpha-1}\right)}
	&= C_1-C_0+\begin{cases}
	\log\left(\frac{\theta(\beta-\theta_1)}{(1-\theta)(\theta_1+\alpha-1)}\right)\, ,\quad &0\leq\theta\leq \theta_1\\[5pt]
	\log\left(\frac{\theta(\beta-\theta)}{(1-\theta)(\theta+\alpha-1)}\right)\, ,\quad &\theta_1< \theta< \theta_2\\[5pt]
	\log\left(\frac{\theta(\beta-\theta_2)}{(1-\theta)(\theta_2+\alpha-1)}\right)\, ,\quad &\theta_2\leq \theta\leq 1
	\end{cases}.
\end{split}
\end{align}
We note that the pruning scheme in \cite{srinivas_generalized_2016} uses a Beta hyper-prior with both $\alpha,\beta<1$ in addition to the suboptimal choice $\pi^\star \equiv \theta$ (independent of $\alpha,\beta$).

\subsection{Flattening hyper-prior}\label{sec:flattening_hyper_prior}
Here, we propose a more careful choice for the hyper-prior $p(\pi \mid \Gamma)$ motivated by the considerations discussed earlier.
More specifically, we seek $p(\pi \mid \Gamma)$ such that $\pi^\star=\arg\min_\pi\ J(\pi)$ in \eqref{eq:opt_pi_star_full} (without the constraints $\epsilon_1\leq\pi\leq 1-\epsilon_2$) makes the $\log$-term in  \eqref{eq:grad_theta} flat for $\theta\in [0,\ 1]$, i.e.,
\begin{align}\label{eq:flat_log_term}
	\log \left(\frac{\theta(1-\pi^\star)}{(1-\theta)\pi^\star}\right) = - \log \gamma,\quad \gamma>0.
\end{align}
The solution to this optimization problem is characterized by
\begin{align*}
	\frac{\partial J(\pi)}{\partial \pi} = \frac{1-\theta}{1-\pi} - \frac{\theta}{\pi} - \frac{p'(\pi \mid \Gamma)}{p(\pi \mid \Gamma)} = 0 \quad \Leftrightarrow \quad \theta = \pi \left[1-(1-\pi)\frac{p'(\pi \mid \Gamma)}{p(\pi \mid \Gamma)}\right].
\end{align*}
Then, requiring \eqref{eq:flat_log_term} yields the following ordinary differential equation (ODE) for $p(\pi \mid \Gamma)$:
\begin{align*}
	\frac{(1-\theta)\pi}{\theta(1-\pi)} = \gamma \quad \Leftrightarrow\quad \frac{p(\pi\mid \Gamma)+ \pi p'(\pi\mid \Gamma)}{p(\pi\mid \Gamma) - (1-\pi)p'(\pi\mid \Gamma)} = \gamma \quad \Leftrightarrow \quad \frac{p'(\pi\mid \Gamma)}{p(\pi\mid \Gamma)} = \frac{\gamma-1}{1+(\gamma-1)(1-\pi)}.
\end{align*}
 The solution to this ODE is given by
\begin{align}\label{eq:pi-flat}
	p(\pi \mid \Gamma=\gamma) = \frac{c}{1+(\gamma-1)(1-\pi)}
\end{align}
with constant $c = \frac{\gamma-1}{\log(\gamma)}>0$ for the above to be a valid probability density function (PDF) on $\pi \in [0,1]$. This PDF is monotone increasing for $\gamma >1$ and monotone decreasing for $\gamma <1$. Therefore, to emphasize the lower values of $\pi$ as discussed earlier, we consider only values $\gamma <1$.

Selecting the Flattening hyper-prior in \eqref{eq:pi-flat} leads to solving \eqref{eq:opt_pi_star} with
\begin{align*}
	J^{Flattening}(\pi) = (1-\theta) \log \frac{1-\theta}{1-\pi} + \theta \log\frac{\theta}{\pi} + \log \left(1+(\gamma-1)(1-\pi)\right)
\end{align*}
which yields
\begin{align}\label{eq:piflat}
	\pi^\star = \pi(\theta) = \begin{cases}
	\epsilon_1\, , \quad &0\leq \theta \leq \theta_1\equalhat\frac{\epsilon_1}{\epsilon_1+\gamma(1-\epsilon_1)}\\
	\frac{\gamma\theta}{1+\theta(\gamma-1)}\, , \quad &\theta_1 \leq \theta \leq \theta_2\equalhat\frac{1-\epsilon_2}{1+\epsilon_2(\gamma-1)}\\
	1-\epsilon_2\, , \quad &\theta_2 \leq \theta\leq 1
	\end{cases}\,
\end{align}
Using this choice for $p(\pi \mid \gamma)$ leads to the
following explicit form for the gradient of the loss function with respect to $\theta$:
\begin{align}\label{eq:grad_theta_flattening}
	\frac{\partial L(\theta,W)}{\partial \theta} = C_1-C_0 + \begin{cases}
	\frac{\partial E(\theta)}{\partial \theta} - \log\left(\frac{\theta_1}{1-\theta_1}\gamma\right)\, ,\quad &\theta \leq \theta_1 \\[5pt]
	-\log(\gamma)\, , \quad &\theta_1 < \theta <\theta_2\\[5pt]
	\frac{\partial E(\theta)}{\partial \theta} - \log\left(\frac{\theta_2}{1-\theta_2}\gamma\right)\, ,\quad &\theta_2 \leq \theta
	\end{cases}
\end{align}
where $E(\theta)=\theta \log \theta+ (1-\theta)\log(1-\theta)$ is the negative binary entropy function. Therefore, the regularization $\log$-term in $\frac{\partial L(\theta,W)}{\partial \theta}$ has a constant value determined by the hyper-parameter $\gamma$ for almost all $\theta$ since by choice of $\epsilon_1$ and $\epsilon_2$, $\theta_1$ and $\theta_2$ are as close to $0$ and $1$, respectively, as desired.
%For $\theta_1\approx 0$ and $\theta_2\approx 1$ the gradient is almost everywhere $\frac{\partial L(\theta,W)}{\partial \theta} = \log(\gamma)$, without inducing any extra computation.
We also note that the gradient in  \eqref{eq:grad_theta_flattening} depends on $\log(\gamma)$ rather than explicitly on $\gamma$, which makes possible to numerically tolerate small values of $\gamma$ for achieving appropriate regularization as the network size increases.

\subsection{Discussion on Hyper-prior choice}\label{sec:dishyp}
In Figure~\ref{fig:regu_theta_all_three}, we plot the $\log$-term in \eqref{eq:grad_theta} as a function of $\theta$ for the Beta hyper-prior with parameters $\alpha=0.9$, $\beta=10$ and the Flattening hyper-prior with parameter $\gamma=\expnumber{1}{-2}$. We also plot the corresponding regularization term in the Stochastic Architecture Learning (SAL) \cite{srinivas_generalized_2016}, which is a Beta distribution, using $\alpha=0.099$ and $\beta=0.99$ so that $\frac{\beta}{\alpha}=10$ as chosen in  \cite{srinivas_generalized_2016}. In the case of the Flattening hyper-prior, $-\log(\gamma)$ has the clear interpretation as the minimum difference $C_0-C_1$ that a single unit must achieve to survive during the training process. $C_0-C_1$ is the reduction in the cost when the unit is activated, therefore, it can be thought as the ''usefulness'' of the unit in representing the data. Consequently, only units strongly contributing to the performance survive the training process. Indeed, if $C_0-C_1$ converges to a value less than  $-\log(\gamma)$,  $\frac{\partial L(\theta,W)}{\partial \theta}>0$  and $\theta$ will converge via the gradient descent update to its lower equilibrium point, which can only be in $[0,\theta_1]$ due to the flat nature of the curve and the fact that in both intervals $[0,\theta_1]$ and $[\theta_2,1]$ the curve is decreasing; in turn, the weights updated with rate $\theta$ do not participate in the learning process and are forced to zero by the weight decay term resulting in the automatic pruning of the corresponding unit. On the other hand, if $C_0-C_1$ converges to a value larger than $-\log(\gamma)$, then the equilibrium point for $\theta$ has to be greater than $\theta_2$. When $\epsilon_1,\ \epsilon_2$ are small enough, we have that $\theta_1\approx 0$ and $\theta_2\approx 1$ and thus the $\theta$ parameters converge to their equilibrium points at virtually $0$ or $1$, resulting in an optimally pruned deterministic network.

It is important to note that the flat shape of the regularization term in the case of the Flattening hyper-prior is responsible for the robust learning of the $\theta$ parameters in a manner insensitive to weight initialization. If the weights corresponding to a unit are initialized poorly and the corresponding $C_0-C_1$ value is small initially, the $\theta$ value decreases, but at the same time the regularization term $-\log(\gamma)$ stays constant, giving the unit a chance to recover by adapting its weights to increase the difference $C_0-C_1$ above the threshold $-\log(\gamma)$. Also, the choice of the hyper-parameter $\gamma$, along with the weight decay rate $\lambda$,  dictate the level of pruning and can be clearly used to trade-off network size versus performance. On the other hand, it is difficult to achieve the desired flat shape and appropriate level with the Beta hyper-prior and even more so under the restrictive choice $\theta=\pi$ made in \cite{srinivas_generalized_2016}. Thus, units with initial $C_0-C_1$ small are not given a chance to recover, so the final network depends heavily on weight initialization. Further, to achieve good pruning performance, the initial value of $\theta$ needs to be tuned carefully in SAL. This is not the case in our method, since we can adjust freely the level of the curve by varying $\log(\gamma)$.

\begin{figure}[t]
	\centering
	\resizebox{.55\linewidth}{!}{\includegraphics{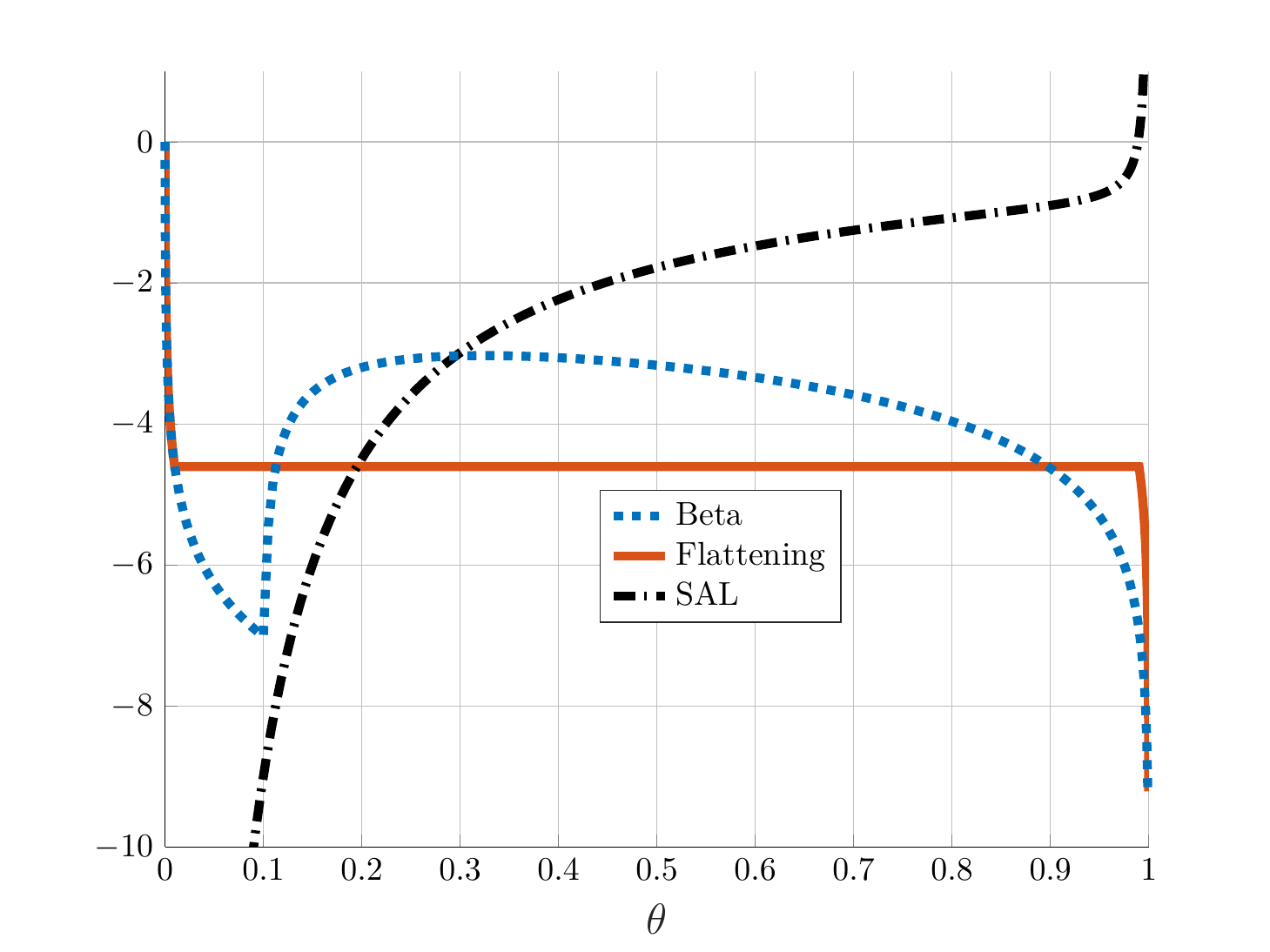}}
	\caption{The $\log$-regularization part of the gradient of the objective $L(W,\theta)$ w.r.t. $\theta$ in  \eqref{eq:grad_theta} for values of $\alpha=0.9, \beta=10$, $\gamma=\expnumber{1}{-2}$, $\epsilon_1=\epsilon_2=\expnumber{1}{-4}$ and $\alpha=0.099$, $\beta=0.99$ in the case of SAL \cite{srinivas_generalized_2016}.}\label{fig:regu_theta_all_three}
\end{figure}

\section{Convergence Results}\label{sec:algorithm_convergence}
In this section, we derive convergence results for the algorithm proposed in Section  \ref{sec:algorithm} by employing the continuous-time ordinary differential equations (ODE) underlying the gradient descent process on which the learning algorithm is based. More specifically, we show that this ODE system has equilibria points for which the weights of certain units are zero.
We then tie the convergence of the ODE system to the convergence of the learning algorithm using stochastic approximation results.

The ODE system describing the continuous learning (gradient descent) dynamics for the fan-in weights $w_b$, fan-out weights $w_f$ and update rate parameter $\theta$ of a typical hidden unit is expressed as
\begin{equation}\label{eq:ODE}
\left\{	\dot w_f = -\frac{\partial L}{\partial w_f}, \hspace{1cm} \dot w_b = -\frac{\partial L}{\partial w_b}, \hspace{1cm} \dot \theta =-\frac{\partial L}{\partial \theta} \right\}
\end{equation}
where $L(W,\Theta)$ is the objective function defined in \eqref{eq:L_train_obj}.
Next, \eqref{eq:grad_w} and expressing $C(W,\Theta$) as in \eqref{eq:C_def} yields
\begin{align}\label{eq:dLdw}
\begin{split}
	\frac{\partial L}{\partial w_f} =\theta \frac{\partial C_1}{\partial w_f} +  (1-\theta)\frac{\partial C_0}{\partial w_f} + \lambda w_f  \\
    \frac{\partial L}{\partial w_b} =\theta \frac{\partial C_1}{\partial w_b} +  (1-\theta)\frac{\partial C_0}{\partial w_b} + \lambda w_b.
\end{split}
\end{align}
Furthermore, with the notation of Section~\ref{sec:Taylor}, we derive
\begin{align*}%\label{eq:C_wf-grads}
\frac{\partial C(\tilde w_f)}{\partial w_f}=\frac{\partial C(\tilde w_f)}{\partial \tilde w_f}\cdot \frac{\partial \tilde w_f}{\partial w_f}=\E_{{\bar\Xi,\mathcal{D}}}\left[a_{l-1}\left(w_b^\top z_b\right) \delta_f \right]\cdot \xi
\end{align*}
and obtain for $\xi=0$ and $\xi=1$, respectively:
\begin{align}\label{eq:dCdwf}
\frac{\partial C_0}{\partial w_f} = 0 \quad \text{and} \quad
&\frac{\partial C_1}{\partial w_f} =  \E_{{\bar\Xi,\mathcal{D}}}\left[a_{l-1}\left(w_b^\top z_b\right) \delta_f(\xi=1) \right].
\end{align}
Also, from the backpropagation algorithm and for fixed $\xi$,
\begin{align*}%\label{eq:C_wb-grads}
\frac{\partial C(w_b)}{\partial w_b}=\E_{{\bar\Xi,\mathcal{D}}}\left[z_b a_{l-1}'\left(w_b^\top z_b\right) w_f^\top\delta_f\right]\cdot \xi
\end{align*}
and obtain for $\xi=0$ and $\xi=1$, respectively:
\begin{align}\label{eq:dCdwb}
\frac{\partial C_0}{\partial w_b} = 0 \quad \text{and} \quad
&\frac{\partial C_1}{\partial w_b} =  \E_{{\bar\Xi,x_i,y_i}}\left[a_{l-1}\left(w_b^\top z_b\right) \delta_f(\xi=1)\right]
\end{align}
We now substitute \eqref{eq:dCdwf} and \eqref{eq:dCdwb} in \eqref{eq:dLdw} and also use \eqref{eq:grad_theta} derived earlier to write the ODE system \eqref{eq:ODE} describing the continuous learning (gradient descent) dynamics for a single unit as follows:
\begin{align}\label{eq:Sys_GrDes}
\begin{split}
	\dot w_f &= -\frac{\partial L}{\partial w_f} =-\theta \underbrace{\E\left[\frac{a_{l-1}(w_b^\top z_b)}{w_b^\top z_b} \delta_f z_b^\top\right]}_{\equalhat M_1}w_b -\lambda w_f    \\
	\dot w_b &= -\frac{\partial L}{\partial w_b} = -\theta \underbrace{\E\left[a_{l-1}'(w_b^\top z_b) z_b \delta_f^\top\right]}_{\equalhat M_2}w_f - \lambda w_b \\   	
	\dot \theta &=- \frac{\partial L}{\partial \theta}= -(C_1 - C_0) + \log \left(\frac{(1-\theta)\pi^\star}{\theta(1-\pi^\star)}\right)
\end{split}
\end{align}
We note that the expectations in \eqref{eq:Sys_GrDes} are taken over the RV $\bar \Xi$ as well as the data set $\mathcal{D}$ and that the backpropagation error $\delta_f$ is computed with the unit being active, i.e., $\xi=1$.

Clearly, $\{w_f=0, w_b=0, \theta=\pi^\star\}$ is a valid equilibrium point of the dynamical system \eqref{eq:Sys_GrDes} since $C_1-C_0 = 0$ when the weights of the unit are zero. The following result establishes conditions under which this equilibrium point is asymptotically stable.
\begin{theorem}\label{thm:stability}
	Consider the dynamical system \eqref{eq:Sys_GrDes} written as
	\begin{align}\label{eq:dyn_sys}
	\begin{split}
	\begin{bmatrix}
	\dot w_f \\ \dot w_b
	\end{bmatrix} &= \begin{bmatrix}
	-\lambda I &  -\theta M_1 \\ -\theta M_2 & -\lambda I
	\end{bmatrix}\begin{bmatrix}
	w_f \\  w_b
	\end{bmatrix}\\
	\dot \theta &= -(C_1-C_0) + \log\left[\frac{(1-\theta)\pi^\star}{\theta(1-\pi^\star)}\right]
	\end{split}
	\end{align}
	with $w_b \in \mathbb{R}^p$, $w_f \in \mathbb{R}^q$, $\theta \in \mathbb{R}$ and matrices $M_1, M_2$ of appropriate dimension as well as scalar functions $C_1,C_0$ of $w_f,w_b$ and $\lambda >0$ a fixed parameter. Assume that the difference $C_1-C_0$ satisfies the inequality
	\begin{align}\label{eq:kbnd}
	|C_1-C_0| \leq \kappa\cdot \phi, \quad 0<\kappa<\infty,
	\end{align}
	where $\phi\equalhat \frac{1}{2}\left[{\norm{w_f}^2}+ {\norm{w_b}^2}\right]$ and also the maximum singular value of $M = M_1+M_2^\top$ is bounded as follows:
\begin{align}\label{eq:ebnd}
\bar\sigma( M_1+ M_2^\top)\leq \eta < \infty.
\end{align}
Lastly, let the hyper-prior $p(\pi\mid \Gamma)$ be such that the solution to $\pi^\star=\rm{argmin}_{\pi} J(\pi)$ under the constraints $0<\epsilon_1\leq\pi\leq 1-\epsilon_2<1$ with $J(\pi)$ defined in \eqref{eq:opt_pi_star_full} satisfies
	\begin{align}\label{eq:cond_prior}
	\pi^\star > \theta \quad \text{if}\ \theta<\epsilon_1 \quad \text{and}\quad  \pi^\star < \theta \quad \text{if}\ \theta>\epsilon_1.
	\end{align}
	Then, $\{w_f=0,w_b=0,\theta=\epsilon_1\}$ is an equilibrium of the system \eqref{eq:dyn_sys} and it is locally asymptotically stable if
	\begin{align}\label{eq:cond_stability}
	0 < \epsilon_1 < \frac{1}{2}\frac{\lambda}{\eta + \kappa}.
	\end{align}
	Moreover, if \eqref{eq:cond_stability} holds \begin{align}\label{eq:RoA}
	\mathcal{A}= \left\{w_b \in \mathbb{R}^p, w_f \in \mathbb{R}^q, \theta \in \mathbb{R} \quad \middle|\quad {\norm{w_f}^2}+{\norm{w_b}^2} + (\theta-\epsilon_1)^2  < \left(\frac{\lambda}{\eta + \kappa}-\epsilon_1\right)^2 \right\}
	\end{align}
	belongs to the region of attraction to the equilibrium point.
\end{theorem}
\begin{proof}
	Clearly given \eqref{eq:kbnd}, $(w_f=0,w_b=0,\theta=\epsilon_1)$ is an equilibrium point of the dynamics of $w_f$ and $w_b$ and $\theta$ in \eqref{eq:dyn_sys}. To establish its local asymptotic stability, we consider the Lyapunov candidate function
	\begin{align*}%\label{eq:lyap_function}
	V(\phi,\theta) =\phi + \frac{1}{2} (\theta-\epsilon_1)^2,
	\end{align*}
	which satisfies $V>0$ for $\{w_f, w_b,\theta\}\not=\{0,0,\epsilon_1\}$ and show that $\dot V<0$ in the region $\mathcal A$, which is clearly a level set of $V$.

 First, let
	\begin{align}\label{eq:th-eps1}
	1 >\theta_{o}\equalhat\frac{\lambda}{\eta+\kappa} > 2 \epsilon_1
	\end{align}
	and assume that
	$0\leq\theta < \theta_{o}$.
	Next, pick $e>0$ such that
	\begin{align}\label{eq:e_bounds}
	\kappa (\theta_{o}-\epsilon_1) < e < \kappa \theta_{o}
	\end{align}
	holds for $\epsilon_1$ satisfying \eqref{eq:cond_stability}. Using the right part of the above inequality we get
	\begin{align*}
	e < \kappa \theta_{o} = \kappa \frac{\lambda}{\eta+\kappa} = \lambda - \theta_{o} \eta
	\end{align*}
	implying that
	\begin{align*}
	\lambda > \theta \bar \sigma(M) + e \quad \text{or aternatively} \quad  \lambda>e \quad \text{and}\quad  (\lambda-e) I - \frac{1}{\lambda-e} \theta^2 M^\top M \succ 0
	\end{align*}
	where $M\equalhat M_1+M_2^\top$.
	Then, by the Schur-condition for positive definiteness \cite[Appendix A.5.5]{boyd2004convex} the following matrix $M_\theta$ is positive definite:
	\begin{align*}
	M_\theta = \begin{bmatrix}
	(\lambda-e) I & \theta M \\ \theta M^\top & (\lambda-e) I
	\end{bmatrix} \succ 0,
	\end{align*}
and we obtain
	\begin{align}\label{eq:phi_bound}
	\dot{\phi}=
	\begin{bmatrix}
	{w_f}^\top & {w_b}^\top
	\end{bmatrix}
	\begin{bmatrix}
	-\lambda I & -\theta M \\ -\theta M^\top & -\lambda I
	\end{bmatrix}
	\begin{bmatrix}
	w_f \\ w_b
	\end{bmatrix}
	\leq
	\begin{bmatrix}
	{w_f}^\top & {w_b}^\top
	\end{bmatrix}
	\begin{bmatrix}
	-e I & 0 \\ 0 & -e I
	\end{bmatrix}
	\begin{bmatrix}
	w_f \\ w_b
	\end{bmatrix}
	=- e \phi.
	\end{align}
	Next, the Lie-Derivative of the candidate Lyapunov function $V$ along the dynamics of the system \eqref{eq:dyn_sys} satisfies
	\begin{align} \label{eq:Lie_der}
	\dot V =  \dot \phi +(\theta-\epsilon_1)\dot \theta \leq
	-e\phi + |\theta-\epsilon_1|\kappa \phi + (\theta-\epsilon_1) \log\left(\frac{(1-\theta)\pi^\star}{\theta(1-\pi^\star)}\right)
	\end{align}
	using \eqref{eq:phi_bound} and  $\dot\theta$ from \eqref{eq:dyn_sys}.
	Now note that from \eqref{eq:e_bounds} and since $\theta_{o}>2\epsilon_1$, it holds
	\begin{align}\label{eq:ebnds2}
	e> \kappa |\theta_{o}-\epsilon_1| > \kappa |\theta-\epsilon_1|.
	\end{align}
Also note that assumption \eqref{eq:cond_prior} implies that
\begin{align*}
\frac{(1-\theta)\pi^\star}{\theta(1-\pi^\star)}>1 \quad \text{if}\, \theta < \epsilon_1 \quad \text{and} \quad
\frac{(1-\theta)\pi^\star}{\theta(1-\pi^\star)}<1 \quad \text{if}\, \theta > \epsilon_1
\end{align*}
and this yields
\begin{align}\label{eq:theta_ineq_2}
(\theta-\epsilon_1)\log\left(\frac{(1-\theta)\pi^\star}{\theta(1-\pi^\star)}\right) < 0 \quad \forall \theta \in [0,1],\ \theta\not=\epsilon_1.
\end{align}
Using \eqref{eq:ebnds2} and \eqref{eq:theta_ineq_2} in \eqref{eq:Lie_der}
gives that $\dot V<0$ away from $\theta=\epsilon_1$, $\phi=0$.	
Finally, observe that the assumption $\theta<\theta_{o}$, which was used to obtain $\dot V<0$, holds in the region $\mathcal A$ since $\phi>0$. Hence, the analyzed equilibrium point is locally asymptotically stable with region of attraction encompassing $\mathcal{A}$.
\end{proof}
Next, we discuss the required assumptions in Theorem~\ref{thm:stability}. We show in \ref{app:A} that \eqref{eq:kbnd} and \eqref{eq:ebnd} hold true, if the NN weights $W$ remain bounded during the training process, i.e., $\phi\leq\phi_{max}$. This condition can be assured by introducing the projection step in the proposed algorithm. However, we remark that in our simulation experiments by choosing a large $\phi_{max}$, we never had to exercise the projection step. The next result shows that condition \eqref{eq:cond_prior} is satisfied for monotonically decreasing priors $p(\pi\mid \Gamma)$, a property that is desirable for effective pruning as we argued earlier. Indeed, both choices for $p(\pi\mid \Gamma)$ discussed in Section~\ref{sec:hyper_prior_selection} enjoy this property and, therefore, satisfy the required condition for the applicability of Theorem \ref{thm:stability}.

\begin{lemma}\label{lemma:hyper_prior_characterization}
Let $\pi^\star=\rm{argmin}_{\pi} J(\pi)$ under the constraints $0<\epsilon_1\leq\pi\leq 1-\epsilon_2<1$ with $J(\pi)$ defined in \eqref{eq:opt_pi_star_full} using a hyper-prior $p(\pi\mid \Gamma)$ that is monotonically decreasing in $[0,1]$. Then, it holds
\begin{align*}%\label{eq:theta_ineq}
%\red{(\theta-\epsilon_1)\log\left(\frac{(1-\theta)\pi^\star}{\theta(1-\pi^\star)}\right) < 0 \quad \forall \theta \in [0,1],\ \theta\not= \epsilon_1.}
%\end{align}
%\begin{align}
\pi^\star > \theta \quad \text{if}\ \theta<\epsilon_1 \quad \text{and}\quad  \pi^\star < \theta \quad \text{if}\ \theta>\epsilon_1.
\end{align*}
\end{lemma}
\begin{proof}
	Clearly, if $\theta<\epsilon_1$, we have that $\pi^\star>\theta$ since we directly constrained $\pi\geq \epsilon_1$. Further, if $\theta>\epsilon_1$, notice that from
\begin{align*}%\label{eq:Jp_der}
		\frac{ dJ(\pi)}{d\pi}=\frac{\pi-\theta}{\pi(1-\pi)}-\frac{d\log p(\pi\mid\Gamma)}{d\pi}
\end{align*}
and since $p(\pi\mid\Gamma)$ is assumed to be monotonically decreasing, we have that $\frac{ dJ(\pi)}{d\pi}>0$ for $\pi\geq\theta$. Thus, it follows that $J(\pi)$ is minimized under the considered constraints for $\epsilon_1\leq\pi^\star<\theta$ and the proof is complete.
\end{proof}
The following theorem connects Theorem~\ref{thm:stability} to our simultaneous pruning/training algorithm proposed in Section \ref{sec:algorithm}. This algorithm generates discrete-time sequences $\{w_f(n), w_b(n), \theta(n)\}_{n\geq0}$ for each unit of the network to minimize the objective \eqref{eq:L_train_obj} based on the discretization of the ODE system \eqref{eq:Sys_GrDes} and stochastic gradient descent. More specifically,
\begin{equation}\label{eq:DE}
%\blue{\left\{
w_f(n+1) = w_f(n)-a(n)\widehat{\frac{\partial L}{\partial w_f}},\
w_b(n+1) = w_b(n)-a(n)\widehat{\frac{\partial L}{\partial w_b}},\
\theta(n+1) =\theta(n)-a(n)\widehat{\frac{\partial L} {\partial \theta}}
%\right\}}
\end{equation}
where the estimates of the gradient of the objective function $L(W(n),\Theta(n))$ are obtained from \eqref{eq:grad_m} and \eqref{eq:grad_theta} and the stepsize $a(n)$ satisfies the Robbins-Monro conditions
\begin{equation}\label{eq:RM}
\sum_{n=0}^\infty a(n) = \infty  \quad \text{and}\quad \sum_{n=0}^\infty a(n)^2 = 0.
\end{equation}
\begin{theorem}\label{thm:Alg_conv}
Consider the sequence $x(n)\equiv\{w_f(n),w_b(n), \theta(n)\}_{n\geq0}$ of the weights and their update rate corresponding to a single unit/filter of the network as generated by \eqref{eq:DE} and \eqref{eq:RM}
%the algorithm presented in Section \ref{sec:algorithm}
and assume that this sequence enters and remains within a region of attraction for the asymptotically stable equilibrium point $x^*\equiv\{w_f=0,w_b=0, \theta=\epsilon_1\}$ of the ODE system \eqref{eq:Sys_GrDes} contained in $\{\phi=\|w_f\|^2+\|w_b\|^2\leq 2\phi_{max},\,\newline \theta\in[\theta_l,\ \theta_h]\}$ with $0<\theta_l<\epsilon_1<\theta_h<1$; in particular such a region is
\begin{align}\label{eq:ROAd}
	\mathcal{A}_d= {\cal A}\bigcap\left\{\norm{w_f}^2+\norm{w_b}^2<2\phi_{max}\right\}\bigcap
 \left\{\theta \in (
 \theta_l, \theta_h)\right\},
\end{align}
where $\cal A$ is defined in \eqref{eq:RoA}.
Then, $x(n)$ converges to $x^*$ almost surely.
\end{theorem}
The proof of Theorem~\ref{thm:Alg_conv} is based on well-established stochastic approximation results in \cite{borkar2009stochastic} and is relegated to \ref{app:B}.

\section{Learning/Pruning Algorithm}\label{sec:algorithm}
We now present the proposed simultaneous learning and pruning algorithm based on minimizing the objective in  \eqref{eq:L_train_obj} via gradient descent.
Pseudo-code of the algorithm is given in Algorithm \ref{alg:learning_algo}.
The algorithm assumes a given data set: $\mathcal{D}=\lbrace (x_i,y_i)\rbrace_{i=1}^N$, mini-batch size $B$ and an initial network structure characterized by the number of layers and the number of units in each layer. We select hyper-parameters $0<\alpha<1$ and $\beta>1$ or $0<\gamma<1$ for the Beta or Flattening hyper-priors, respectively, $\lambda>0$, $\tau>0$ for weight regularization.
We remark that as the data set size $N$ increases and because $C_0$ and $C_1$ are total expected errors over the data set, $C_0-C_1$ roughly increases linearly with $N$. This behavior is consistent with the Bayesian approach whereby using more data samples to estimate the posterior distribution leads to a diminishing influence of prior information. Therefore, to induce sufficient pruning and for the previous pruning conditions to be meaningful, the hyper-parameters $\gamma$ for the Flattening hyper-prior and $\alpha$, $\beta$ for the Beta hyper-prior need to be matched appropriately to the size of the data set. For the Flattening hyper-prior, the choice of $\gamma$ is straightforward, since it directly compares to $C_0-C_1$. For the Beta hyper-prior, we typically pick $\alpha=0.1$ and control the level of the log-term in the $\dot\theta$ equation by picking $\beta>1$. For large $N$, $\beta$ needs to be considerably larger than one, making the use of the Beta hyper-prior less effective than the Flattening hyper-prior for large data sets.
%the tuning the Beta hyper-prior is more delicate and in the end less effective than the Flattening hyper-prior.
We also pick $0<\epsilon_1, \epsilon_2\ll 1$ for guiding the convergence of the variational parameters $\Theta$.
The convergence results established in Section~\ref{sec:algorithm_convergence} require that
$\epsilon_1<\frac{\lambda}{2(\kappa+\eta)}$ where $\kappa$ and $\eta$ are constants such that \eqref{eq:kbnd} and \eqref{eq:ebnd} hold. The existence of such constants is established in Appendix~\ref{sec:kappabound} and \ref{sec:deltabound}, however, the analysis there provides rather conservative estimates for $\kappa$ and $\eta$ to be of practical value. Therefore, we simply let $\epsilon_1$ and $\epsilon_2$ be user-defined parameters. It is also shown in Appendix~\ref{sec:kappabound} that $\kappa$ and $\eta$ scale linearly with the data set size $N$. Then, the weight regularization parameter $\lambda$  should be also chosen to scale with $N$ to avoid having $\frac{\lambda}{\kappa+\eta}$ too small and allow reasonable values for $\epsilon_1$.
% and which can be estimated cheaply during the course of the algorithm as discussed Section~\ref{sec:practical_conditions}.
Equivalently, we may define directly $\theta_1$, $\theta_2$ instead of $\epsilon_1$ and $\epsilon_2$ (see \eqref{eq:pibeta} and \eqref{eq:piflat} for the Beta and Flattening hyper-priors, respectively), which are used to obtain the gradients in \eqref{eq:grad_theta_beta} and \eqref{eq:grad_theta_flattening}.
The theoretical analysis also requires to choose a bound $\phi_{max}$ used to scale the weights of each unit so that $\|w_b\|^2+\|w_f\|^2\leq 2\phi_{max}$ (see Figure~\ref{fig:NN_struct}) and interval $[\theta_l,\theta_h]$ with $0<\theta_l<\epsilon_1$, $1-\epsilon_2<\theta_l<1$ used to clip the variational parameters.  Picking a large value for $\phi_{max}$ should make the weight projection step practically unnecessary as our experiments show. The initialization phase is completed by setting the iteration counter to $n=0$, and initializing the network weights $W(0)$  in a standard manner and the parameters $\Theta(0)$ to $0.5$ in a neutral fashion.

Each iteration of the algorithm consists of five main steps. In Step~1, a mini-batch $\mathcal{S}=\lbrace ( x_i, y_i)\rbrace_{i=1}^B$ is sampled from the data $\mathcal{D}$ with replacement and a sample $\hat \Xi \sim Bernoulli(\Theta(n))$ is obtained and used to predict the network output and to approximate needed expectations. A single realization $\hat \Xi$ is used for all the data in the mini-batch to allow for a more efficient implementation without significant increase in the variance of realized estimates \cite{graham2015efficient}. This step is equivalent to the forward calculation in the well-known Dropout formulation in \cite{Srivastava_2014_Dropout}. In Step~2, the gradients of the objective \eqref{eq:L_train_obj} with respect to the network parameters $W$ are computed via standard backpropagation and $\frac{\partial C}{\partial\theta}=C_1-C_0$ is approximated for each unit using one of the methods described in Section~\ref{sec:Estimators}.  Then, the gradients of the objective with respect to $\theta$ are obtained from \eqref{eq:grad_theta_beta} or \eqref{eq:grad_theta_flattening} for the case of the Beta or Flattening prior, respectively. Step~3 constitutes the learning phase in which the network parameters $W$ and variational probabilities $\Theta$ are updated via gradient descent utilizing the previously computed gradients. The convergence results of Section~\ref{sec:algorithm_convergence} require that the stepsize $a(n)$ for the gradient descent update satisfies the Robbins-Monro conditions \eqref{eq:RM}
%$\sum_{n=0}^\infty a(n) = \infty  \quad \text{and}\quad \sum_{n=0}^\infty a(n)^2 = 0$,
and after the update step the weights of each unit  are scaled to satisfy $\|w_b\|^2+\|w_f\|^2\leq 2\phi_{max}$ and parameters $\Theta(n+1)$ are clipped in $[\theta_l,\theta_h]$.

In Step~4, we identify the units that can be pruned away and remove them from the network to reduce the computational cost in further training iterations. Based on Theorem~\ref{thm:Alg_conv}, units can be safely removed from the network if the weights $w_f$, $w_b$ and update rate $\theta$ of a unit enter and remain in the region $\mathcal A_d$ defined in \eqref{eq:ROAd}, since then $\{w_f\rightarrow 0,\  w_b\rightarrow 0,\ \theta\rightarrow\epsilon_1\}$ for this unit.
This result is of clear theoretical value but difficult to utilize in practice. However, it points to more practical conditions for unit removal that we have found to work well. Specifically, notice that from \eqref{eq:dyn_sys}, it is sufficient that $\theta\rightarrow 0$, for  $\{w_f\rightarrow 0,\  w_b\rightarrow 0\}$. Then we monitor $\theta$ for each unit and assume that convergence has been achieved once $\theta\leq\theta_{tol}$, a user-defined parameter.
We also utilized an alternative pruning criterion that can speed up pruning considerably as shown by our simulation experiments. More specifically, we observed (see Figure~\ref{fig:results_FCNN_thetas}) that once the $\theta$-curve for a unit has dropped sufficiently from its running maximum, it keeps decreasing towards zero. 
\begin{algorithm}
	\caption{Learning/Pruning Algorithm}
	\label{alg:learning_algo}
	\begin{algorithmic}[1]	
		\REQUIRE  Data set $\mathcal{D}=\lbrace (x_i,y_i)\rbrace_{i=1}^N$ and mini-batch size $B$, Initial Network Structure with Initial Weights $W(0)$ and parameters $\Theta(0)$;
		Hyper-parameters: $0<\alpha<0, \beta>1$, or $0<\gamma<1$, $\lambda>0$, $\tau>0$, $0<\epsilon_1, \epsilon_2\ll 1$, $\phi_{max}>0$, $0<\theta_l<\epsilon_1$, $1-\epsilon_2<\theta_h<1$. If using Case~(i) in Step~4, select $\theta_{tol}$; if using Case~(ii) in Step~4, select $\theta_{per}$, $n_0$ and set $\Theta_{max}=\theta(0)$.
		\\ \hrulefill \\
		\STATE initialize n=0
		\WHILE{Training has not converged or exceeded the maximum number of iterations}
		\STATE %\hrulefill not allowed directly after \WHILE
		\hrulefill \\
		\textbf{STEP 1: Forward Pass}\\
		%\hrulefill \\
		\STATE Sample $B$ times with replacement from the data set $\mathcal{D}$ to obtain samples $\mathcal{S}=\lbrace ( x_i, y_i)\rbrace_{i=1}^B$.
		\STATE Sample one realization of the network by sampling $\hat \Xi \sim Bernoulli(\Theta(n))$.
		\STATE Using current weights $W(n)$, predict the network's output $\hat y_i = NN(x_i;W,\hat \Xi) \ \forall i=1\dots B$ as in \eqref{eq:NN_def}.
		\\ \hrulefill \\
		\textbf{STEP 2: Backpropagation Phase}\\
		%\hrulefill \\
		%		\STATE Select the prior parameter $\pi^\star$ given $\theta(n)$ according to cases i)-ii) in Section \ref{sec:prior_sel}
		\STATE Approximate the gradient w.r.t. weights $g^l_W =\frac{\partial L(W,\Theta)}{\partial W^l} \ \forall l=1\dots L$ using  \eqref{eq:grad_m} with the single sample $\hat \Xi$ and $\mathcal{S}$.
		\STATE Approximate $\frac{\partial C}{\partial\theta}=C_1-C_0$ for each unit in the network using one of the methods described in Section \ref{sec:Estimators}.
		\STATE Approximate the elements of the gradient w.r.t. $\theta$, $g_\theta = 	\frac{\partial L(W,\Theta)}{\partial \theta}$ with  \eqref{eq:grad_theta_flattening} or \eqref{eq:grad_theta_beta} and using the previous approximations of $C_1-C_0$.
		\\ \hrulefill \\
		\textbf{STEP 3: Learning Phase}\\
		%\hrulefill \\
		\STATE  Take gradient steps $W^l(n+1) = W^l(n) - a(n) g^l_W$ and $\theta(n+1) = \theta(n) - a(n) g_\theta$ for appropriate step size $a(n)$ satisfying: $\quad \sum_{n=0}^\infty a(n) = \infty  \quad \text{and}\quad \sum_{n=0}^\infty a(n)^2 = 0.$\\
		\STATE Scale the weights of each unit so that $\|w_b(n+1)\|^2+\|w_f(n+1)\|^2\leq 2\phi_{max}$ (see Figure~\ref{fig:NN_struct}) and clip each  $\theta(n+1) \in [\theta_l, \theta_h]$. If using Case~(ii) in Step~4, set $\theta_{max}=\max\{\theta_{max},\ \theta(n+1)\}$.
		\\ \hrulefill \\
		\textbf{STEP 4: Network Pruning Phase}\\
		%\hrulefill \\
		\FOR{each unit on a hidden layer $l$}
		\IF {(i) $\theta(n+1)<\theta_{tol}$ or (ii) $\theta(n+1)<\theta_{max}(1-\theta_{per})$ and $n>n_0$}
		\STATE prune the  unit by setting its weights to zero.
		\ENDIF
		\ENDFOR
		\\ \hrulefill \\
		\textbf{STEP 5: Check for Convergence}
		%\\ \hrulefill \\
		\STATE If the magnitude of the gradients $g^l_W$ and $g_\theta$ is less than a specified tolerance, the algorithm has converged.
		\\ \hrulefill \\
		\STATE set $n=n+1$
		\ENDWHILE
	\end{algorithmic}
\end{algorithm}
Therefore, we store for each unit the maximum value of its $\theta$ during training and if the current value has dropped by at least by a percentage $\theta_{per}$, a user-defined parameter, the unit is removed from the network. This rule may prune prematurely some units during the early stages of training.
Therefore, it is applied only after the first $n_0$ training iterations, where $n_0$ is also a user-defined parameter. In this manner, this rule does not interfere with the principle that all units should have the opportunity to adapt, which is made possible by our careful selection of the hyper-prior as discussed in Subsection~\ref{sec:dishyp}.

Finally in Step~5, if the algorithm has converged i.e., the gradients of the objective \eqref{eq:Main_obj} with respect to $W$ and $\Theta$ are small, or a maximum number of iterations has been reached, we exit the training process. Otherwise, we set $n=n+1$ and repeat all five steps.

We remark that our algorithm is applicable to both, fully connected and convolutional networks. Although, we present details for the fully connected case for reasons of brevity, extension of our algorithm to the case of a convolutional layer simply entails the introduction of a Bernoulli random variable $\xi$ for each filter matrix. Then, $\xi$ multiplies all elements of the filter and if $\xi$ is zero, the corresponding filter is inactive. Only minor modifications are required for the gradient computation of the performance objective with respect to the variational parameters; these gradients are used for learning the posterior distributions over the random variables $\xi$ and in turn for selecting which filters of the convolutional layer to prune and which to keep. In Section~\ref{sec:simulations},  we apply the algorithm to the LeNet5 convolutional neural network and compare its performance with competing methods.

\section{Simulation Experiments}\label{sec:simulations}
We evaluate our simultaneous learning and pruning algorithm on the MNIST data set \cite{lecun-mnisthandwrittendigit-2010} and on the CIFAR-10 data set \cite{Krizhevsky_09} starting from the commonly used neural network architectures LeNet300-100, LeNet5 and VGG16, respectively. The goal is to learn the size of all hidden layers simultaneously with the network's weights and obtain significantly smaller networks having performance on par with that of the trained unpruned networks. We compare 3 versions of our method combining (i) Flattening hyper-prior with Taylor approximation of $C_0-C_1$, (ii) Flattening hyper-prior with CONCRETE approximation of $C_0-C_1$, and (iii) Beta hyper-prior with Taylor approximation of $C_0-C_1$. In addition we train the considered network architectures without any pruning algorithm to obtain a baseline run. Each experiment consists of training on the full training data set and evaluating the found models on the test images. We do not cross-validate during training to select the final network. Each experiment is run 10 times starting from random weight initializations based on the Xavier normal initializer (Glorot normal) \cite{glorot10a_init}, which we  keep the same for all different versions of our method compared. We evaluate the robustness of our method, i.e., its sensitivity with respect to weight initialization, by reporting the mean and standard deviation of the results over the 10 runs. Throughout this experiments section we use leaky ReLU activation functions with leakage parameter $\expnumber{1}{-3}$ in all networks. We initialize all $\theta$ parameters to $0.5$. Unless stated otherwise, we employ pruning condition (i) as defined in Algorithm \ref{alg:learning_algo}, Step 4 with parameter $\theta_{tol}=\expnumber{1}{-3}$ in our experiments. The remaining training parameters vary for the different data sets and are given in the following subsections. The networks resulting after the last epoch of training are saved and evaluated in terms of their structure, accuracy and  pruning ratio, the latter defined as the percentage of pruned weights from the number of total weights in the starting architecture.
All simulations were run with Python using several TensorFlow/KERAS libraries.

Although our code does not utilize the most efficient dropout implementation, it adds only moderately to the training time without the simultaneous pruning. More specifically, for the MNIST experiments, performed on a laptop computer with an 8th Gen Intel\textsuperscript{\tiny\textregistered} Core\texttrademark\ i5 CPU, 8GB RAM and a low-end NVIDIA\textsuperscript{\tiny\textregistered} GeForce\textsuperscript{\tiny\textregistered} MX150 GPU, the average times per epoch during the first couple epochs of the training process for the LeNet5 experiments are compared as follows. Training without our algorithm and without dropout leads to a baseline time per epoch of about 7 seconds. Using our method with the Taylor approximation and the Flattening hyper-prior increases the time per epoch by about 28\% to approximately 9 seconds while using the Beta hyper-prior leads to an increase of about 85\% as in this case gradients are harder to compute. Using the CONCRETE  instead of the Taylor approximation adds an additional 11\% to the corresponding times. We remark that these times were obtained without removing units/filters from the network once their $\theta$ is less than the pruning threshold to maintain  a fair comparison of the computational effort involved in the different versions of our method.

However, by neglecting the computations involving activation functions as a small fraction of the total computational load, we can estimate the computational load of the baseline network per iteration by $\rho\cdot\sum_{l=1}^{L} \rho^l n^{l}n^{l+1}$ where $n^l$ is the number of units on the $l$th layer and $\rho$ is proportionality constant; also for a fully connected layer $\rho^l=1$  and for a convolutional layer $\rho^l = d_F^2\cdot d_W \cdot d_H$, where $d_F$ is the size of a square filter (kernel) matrix and $d_W \times d_H$ is the dimension of the resulting 2-D feature map on the $l$th layer.

For the dropout networks used in our algorithm, the expected computational load per iteration is upper-bounded by $\rho\cdot\sum_{l=1}^{L} \rho^l n^{l}n^{l+1}$ where $n^l$ is the current number of units on the $l$th layer. This number will decrease as we train and prune the network and hence the computational load per iteration decreases.
On the other hand, the computational load for the dropout network is lower-bounded by $\rho\cdot\sum_{l=1}^{L} \left(\rho^l \sum_{j} \theta_j^{l} \sum_{j}\theta_j^{l+1}\right)$, where $\theta^l$ is the vector of dropout probabilities for the units of the $l$th layer.
Using these formulas with the data obtained from our MNIST experiments, we can estimate that training the baseline LeNet300-100 \cite{le_cun_1998_doc_rec} for 32,830 iterations (35 epochs) is between $2.05$ and $2.75$ times or between $2.49$ and $2.96$ times computationally more expensive than training with our algorithm for 46,900 iterations (50 epochs) and fine-tuning the resulting networks for an additional 9,380 iterations (10 epochs), when pruning conditions (i) and (ii) are used, respectively. For the convolutional LeNet5 \cite{le_cun_1998_doc_rec} network, the baseline network is between $1.15$ and $1.33$ times or between $1.31$ and $1.42$ times more expensive than our algorithm for pruning conditions (i) and (ii), respectively.
When training VGG16 \cite{simonyan_deep_2015} for 300 epochs on the CIFAR-10 data set, the baseline network is between $3.38$ and $4.17$ times more expensive than our algorithm using pruning condition (i).

\subsection{MNIST Experiments}
We use the standard fully connected LeNet300-100 and convolutional LeNet5 architectures \cite{le_cun_1998_doc_rec} as the starting networks for the MNIST data set, which was designed for character recognition of handwritten digits (0-9) and consists of $60000$ 28x28 grayscale images for training and an additional $10000$ alike images for evaluation.
Each experiment consists of training on the full 60000 training images of the MNIST data set and evaluating the found models on the 10000 test images. We use for all experiments the following training and hyper-parameter values: Data set size $N=60000$, Mini-batch size $B=64$ and weight $\mathcal{L}_2$-Regularization parameter $\lambda=20$ (Note that this corresponds to a typically reported regularization parameter of $\frac{\lambda}{N}=\expnumber{3.\bar3}{-4}$). We train using the Adam optimizer \cite{kingma_adam:_2014} with a learning rate of $\expnumber{1}{-3}$ for 50 epochs and follow-up with a fine-tuning phase by training the pruned network for an additional 10 epochs with a learning rate $\expnumber{1}{-4}$. Prior to the fine-tuning phase, all $\theta$ values less than $\expnumber{1}{-3}$ are set to $0$ and all other to $1$, thus specifying the final deterministic network architecture. The networks resulting after the fine-tuning phase are saved and evaluated in terms of their structure, accuracy and  pruning ratio, the latter defined as the percentage of pruned weights from the number of total weights in the starting architecture.
\subsubsection{LeNet300-100}
LeNet300-100 is a fully connected network with 300 units and 100 units in the first and second hidden layer, respectively.  We choose $\log(\gamma)=-25$ and $\alpha=0.9$, $\beta=\expnumber{1}{10}$ for the Flattening and Beta hyper-priors, respectively as described in Section \ref{sec:hyper_prior_selection}. These values are chosen such that the flat parts of both regularization curves lie approximately on the same level (see Fig.~\ref{fig:regu_theta_all_three})  and which reflects that a unit needs to achieve a difference $C_0-C_1>25$ (over the $N=60,000$ samples) at the converged state for it to remain in the network.
In the case of LeNet300-100, in addition to the structured pruning performed by Algorithm \ref{alg:learning_algo} we set weights of the input layer (weights of $W^1$) to zero if their absolute value is less than $\expnumber{1}{-4}$ after training.

Table~\ref{tab:LeNet300-100} summarizes the found architectures and corresponding test accuracies and pruning ratios. Both versions of our method using the Taylor approximation with the Flattening or the Beta-hyper-prior perform similarly well achieving accuracy of over $98.1\%$ while pruning the network to only about 80 total hidden units or pruning about $87.5\%$ of the network's weights. The CONCRETE approximation generates less aggressive pruning but a slightly higher test accuracy.
Using pruning condition (ii) in step 4 of Algorithm \ref{alg:learning_algo} with $\theta_{per}=0.1$ and $n_0=2814$ (3 epochs) produces a network with similar accuracy, pruning ratio and learned architecture as using pruning condition (i). All reported standard deviations are small indicating robust pruning of the network to a consistently small size, independent of weight initialization.
\begin{table}[!htbp]
	\begin{center}
		\begin{tabular}{lclclclc}
			Method & Learned Architecture &Test Accuracy [\%]& Pruning Ratio [\%] \\
			\hline
			Baseline & $300-100$ & $98.46\pm0.09$ & -\\
			\hline
			Flattening: &&&\\
			Taylor & $49.5$\mbox{\tiny$\pm1.86$}$-29.5$\mbox{\tiny$\pm1.50$} & $98.13\pm 0.07$ & $87.59\pm0.43$\\
			Taylor, Condition (ii) & $49.9$\mbox{\tiny$\pm1.92$}$-30.0$\mbox{\tiny$\pm1.00$} & $98.17\pm 0.06$ & $87.48\pm0.46$\\
			CONCRETE  & $59.9$\mbox{\tiny$\pm2.43$}$-38.8$\mbox{\tiny$\pm2.68$} & $98.25\pm 0.07$ & $84.84\pm0.59$\\
			\hline
			Beta, Taylor  & $49.5$\mbox{\tiny$\pm1.75$}$-27.8$\mbox{\tiny$\pm1.47$} & $98.13\pm 0.08$ & $87.62\pm0.41$\\
			%\red{SAL}, $\theta_{init}=0.08$ & $66.3$\mbox{\tiny$\pm9.12$}$-54.1$\mbox{\tiny$\pm7.31$} & $98.20\pm 0.09$ & $82.93\pm2.25$\\
			\hline
		\end{tabular}
	\end{center}
\caption{Resulting architecture and test accuracy of the learning/pruning on the LeNet300-100 architecture using different forms of hyper-prior and methods to approximate the difference $C_0-C_1$.\label{tab:LeNet300-100}}
\end{table}
In Figure \ref{fig:results_FCNN_training}, we depict the evolution of the mean values of total number of hidden units in the network (Fig.~\ref{fig:results_FCNN_units}), the training loss, where to compare convergence rates more accurately, each curve is shifted by the value it converges to  (Fig.~\ref{fig:results_FCNN_loss}), and the test accuracy (Fig.~\ref{fig:results_FCNN_tacc}) during the 50 epochs of training. Plots for the fine tuning phase are not provided.
It takes about 25 epochs for the baseline network without pruning and without dropout to converge with respect to the test accuracy.
Our method prunes only very few units during the first 5 and 10 epochs in the case of the Taylor and CONCRETE approximations, respectively.  This initial period acts as a grace period and provides a chance to the vast majority of the units and weights in the network to learn useful behavior. It serves the purpose of decoupling the pruning process from weight initialization and is crucial for the robustness properties of our algorithm.
After these first few iterations, the most important units in the network have asserted themselves and the pruning process speeds up drastically, significantly reducing the size of the network early during training. Much of the pruning is done by epoch 15 in case of the Taylor approximation, leaving a small network to train for the remaining epochs. The network reaches its final size after about 35 epochs.
Both, the Flattening and the Beta hyper-prior perform well in this experiment. The Taylor approximation enjoys a faster convergence rate when compared to the CONCRETE approximation. When using pruning condition (ii), the pruning process is sped up by about 7 epochs.
%The resulting network is similar sized and performs equally to the one obtained with pruning condition (i).
\begin{figure}[t]%!htb
	\centering
	\begin{subfigure}{0.32\textwidth}
		\resizebox{\linewidth}{!}{\includegraphics{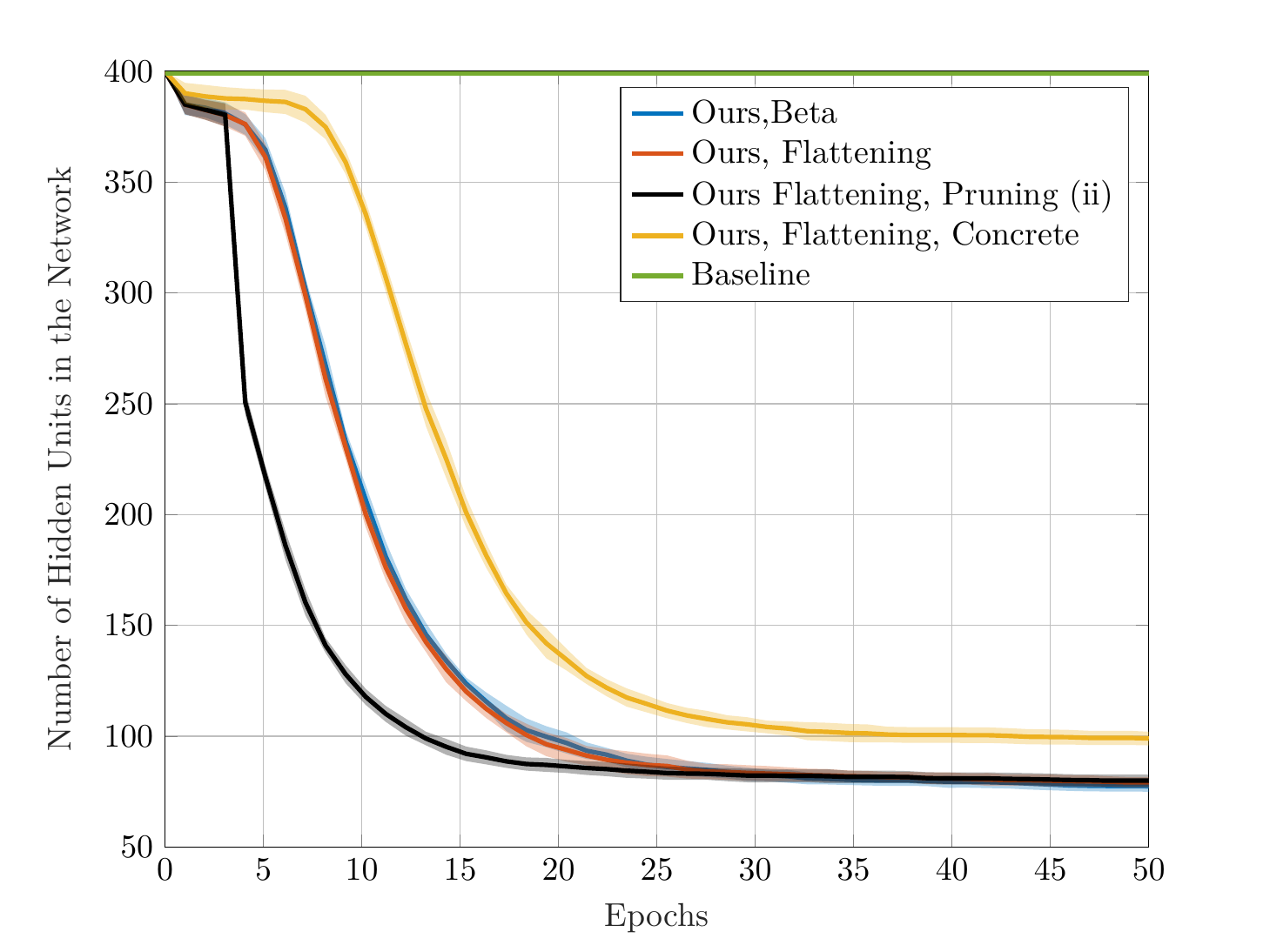}}
		\caption{Number of Units} \label{fig:results_FCNN_units}
	\end{subfigure}
	\begin{subfigure}{0.32\textwidth}
		\resizebox{\linewidth}{!}{\includegraphics{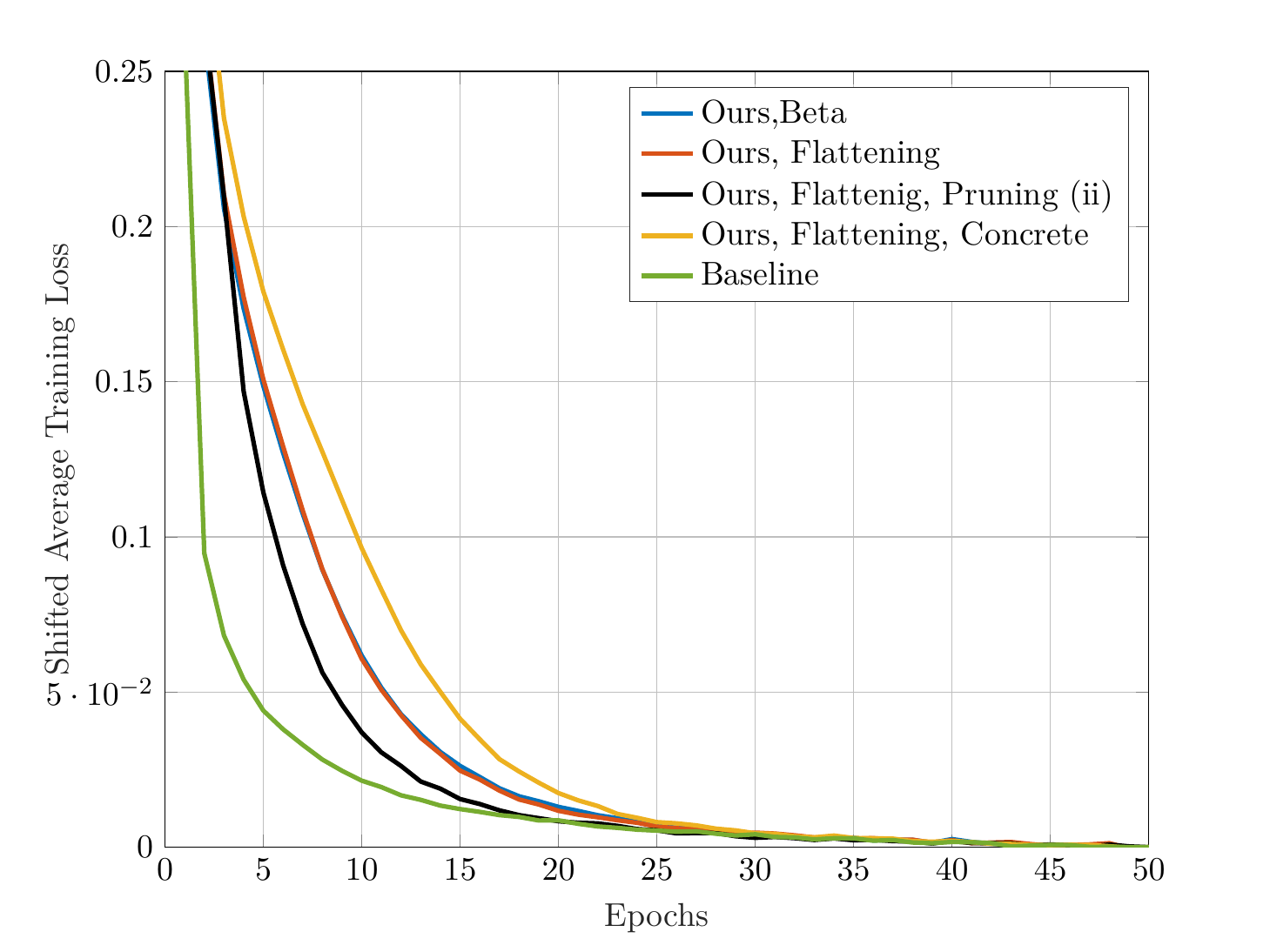}}
		\caption{Loss} \label{fig:results_FCNN_loss}
	\end{subfigure}
	\begin{subfigure}{0.32\textwidth}
		\resizebox{\linewidth}{!}{\includegraphics{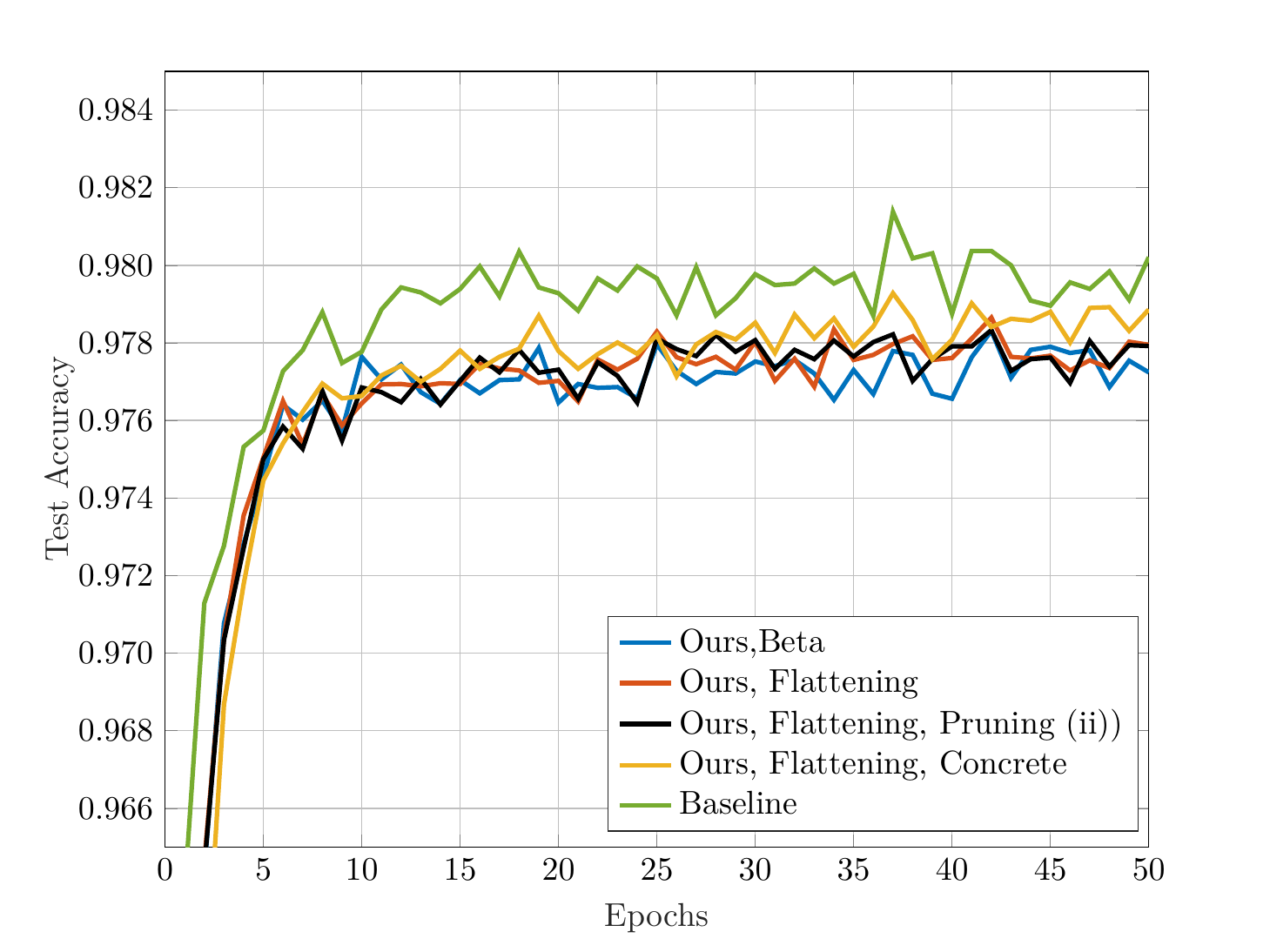}}
		\caption{Accuracy} \label{fig:results_FCNN_tacc}
	\end{subfigure}
	\begin{subfigure}{0.32\textwidth}
		\resizebox{\linewidth}{!}{\includegraphics{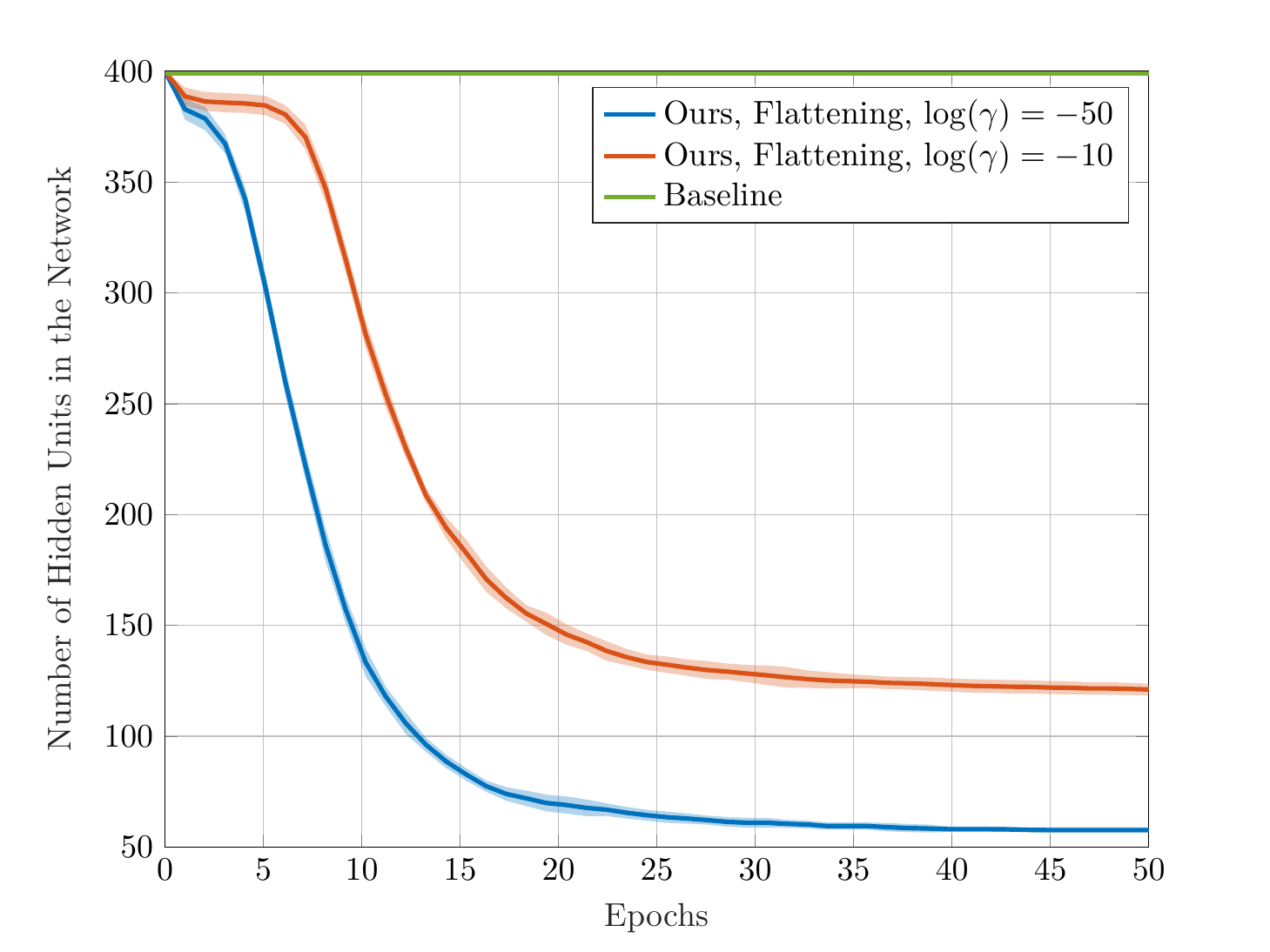}}
		\caption{Number of Units} \label{fig:results_FCNN_units_vary}
	\end{subfigure}
	\begin{subfigure}{0.32\textwidth}
		\resizebox{\linewidth}{!}{\includegraphics{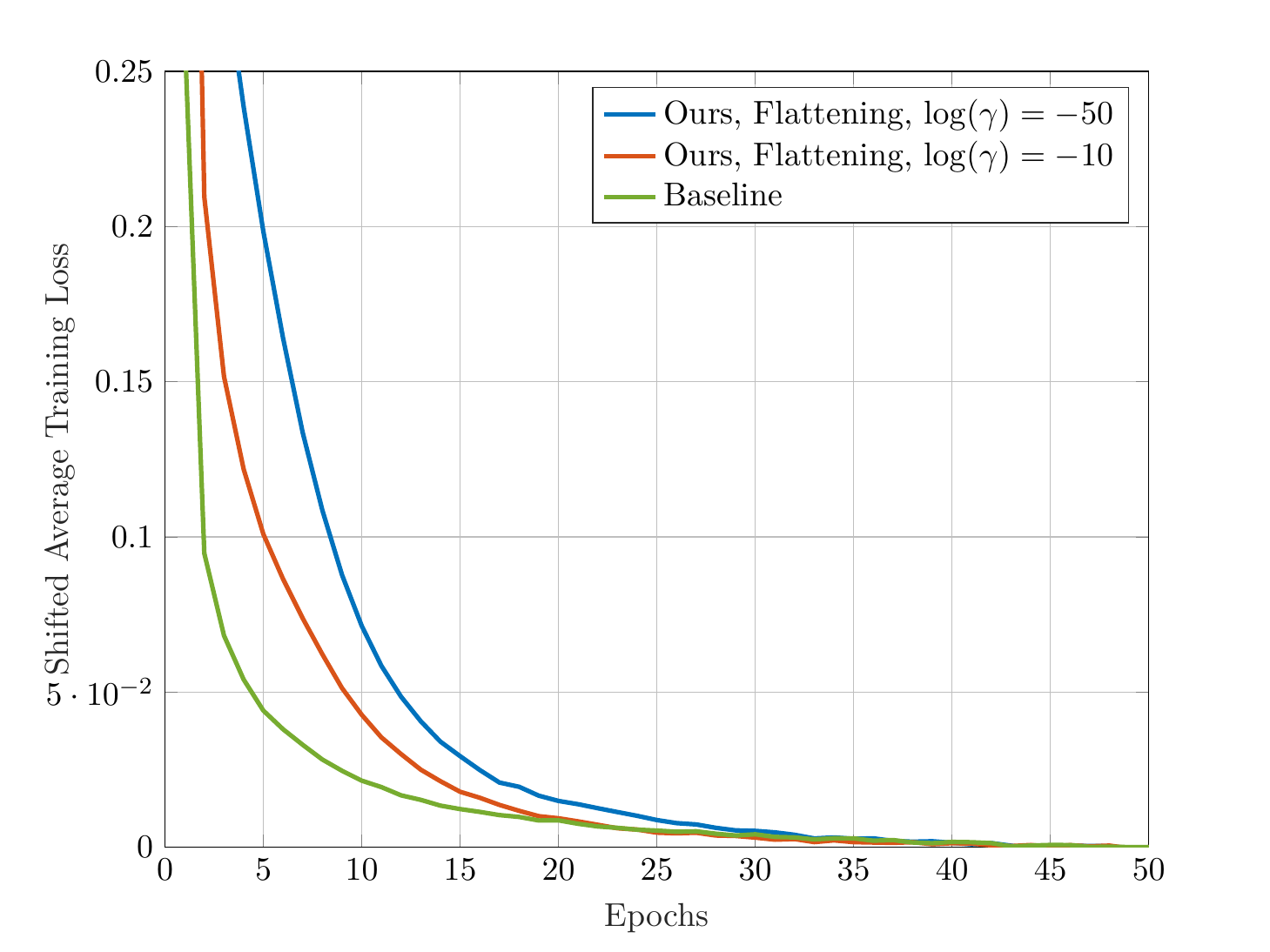}}
		\caption{Loss} \label{fig:results_FCNN_loss_vary}
	\end{subfigure}
	\begin{subfigure}{0.32\textwidth}
		\resizebox{\linewidth}{!}{\includegraphics{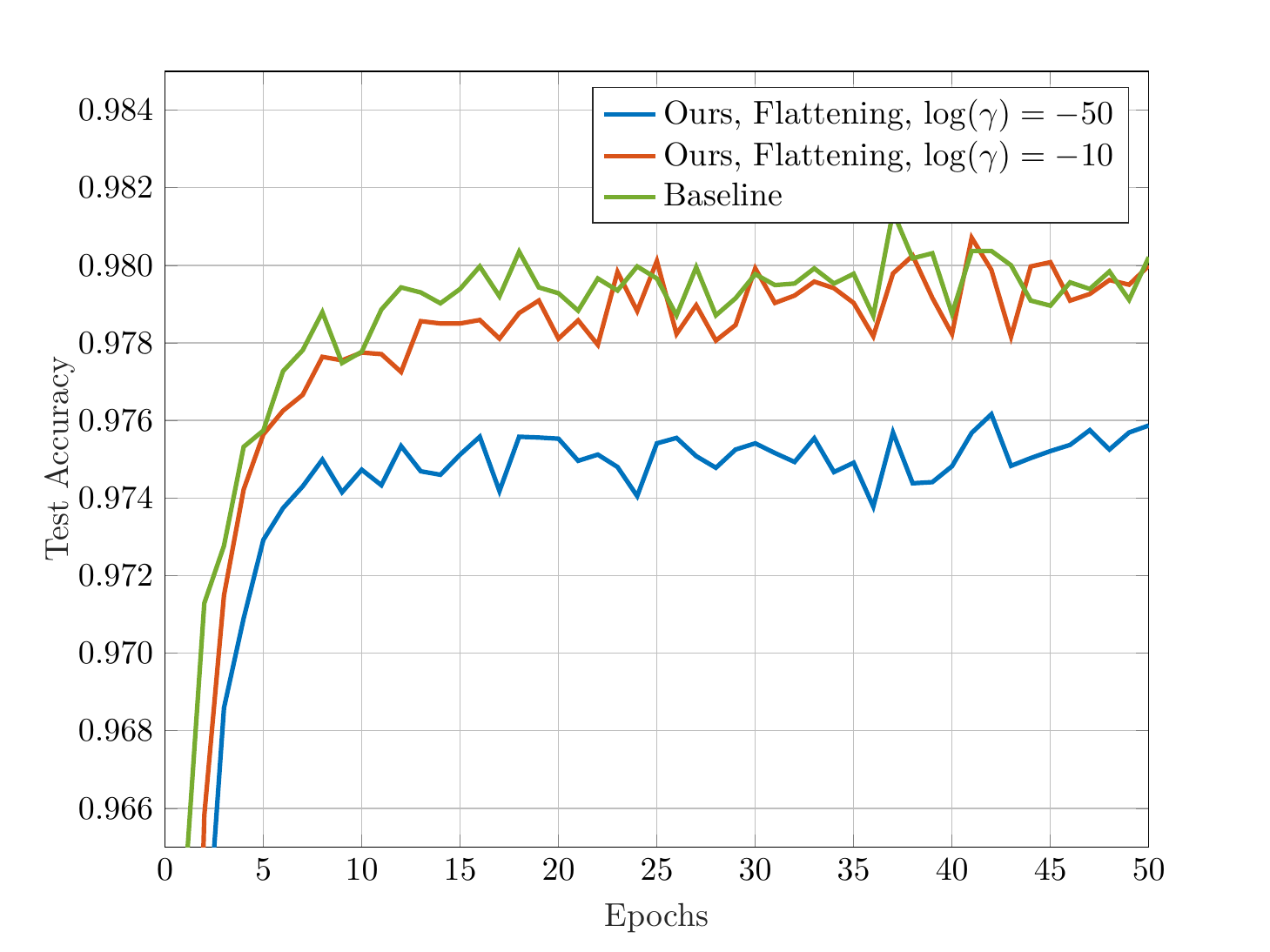}}
		\caption{Accuracy} \label{fig:results_FCNN_tacc_vary}
	\end{subfigure}
	\caption{Total number of hidden units in the network (a),(d), training loss (b),(e) and test accuracy (c),(f) during the 50 epochs of training/pruning on LeNet300-100. Upper row: comparison of the different approximation methods and forms of hyper-prior. Lower row: different choice of hyper-prior parameter.}\label{fig:results_FCNN_training}
\end{figure}
Figures~(\ref{fig:results_FCNN_units_vary}), (\ref{fig:results_FCNN_loss_vary}) and (\ref{fig:results_FCNN_tacc_vary}) show average results over 10 runs for the Flattening hyper-prior using parameter values $\log(\gamma)=-50,-10$ and the Taylor approximation.
Using smaller values for $\log(\gamma)$ leads to networks being pruned more aggressively but our method always keeps an initial phase during which all units receive training, albeit shorter than with higher values of $\log(\gamma)$. Thus, the choice of $\log(\gamma)$ can be used effectively to control/trade-off the resulting network size and prediction accuracy.

\begin{figure}[t]%!htb
	\centering
	%include as pdf
	\resizebox{.49\linewidth}{!}{\includegraphics{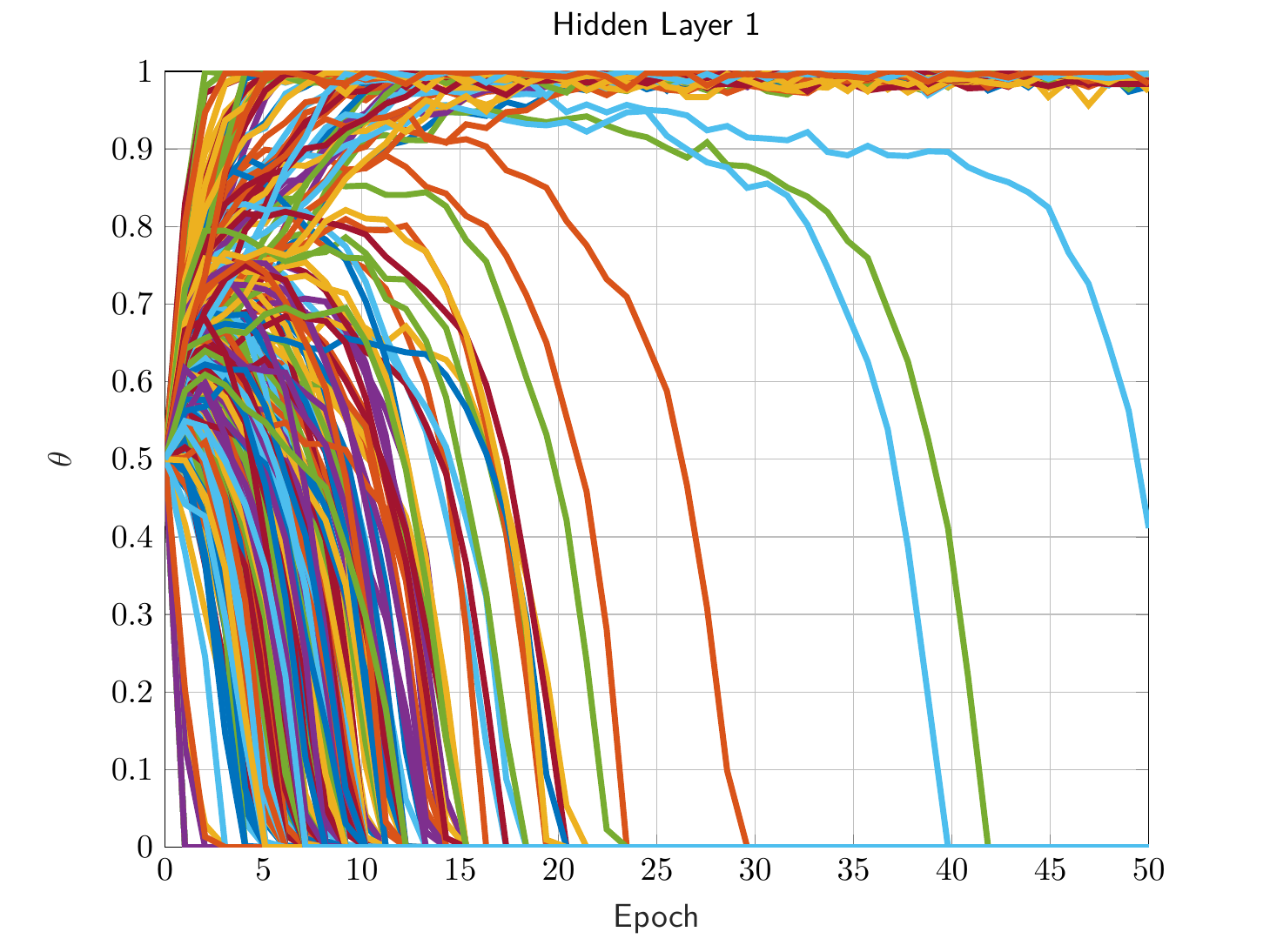}}
	\resizebox{.49\linewidth}{!}{\includegraphics{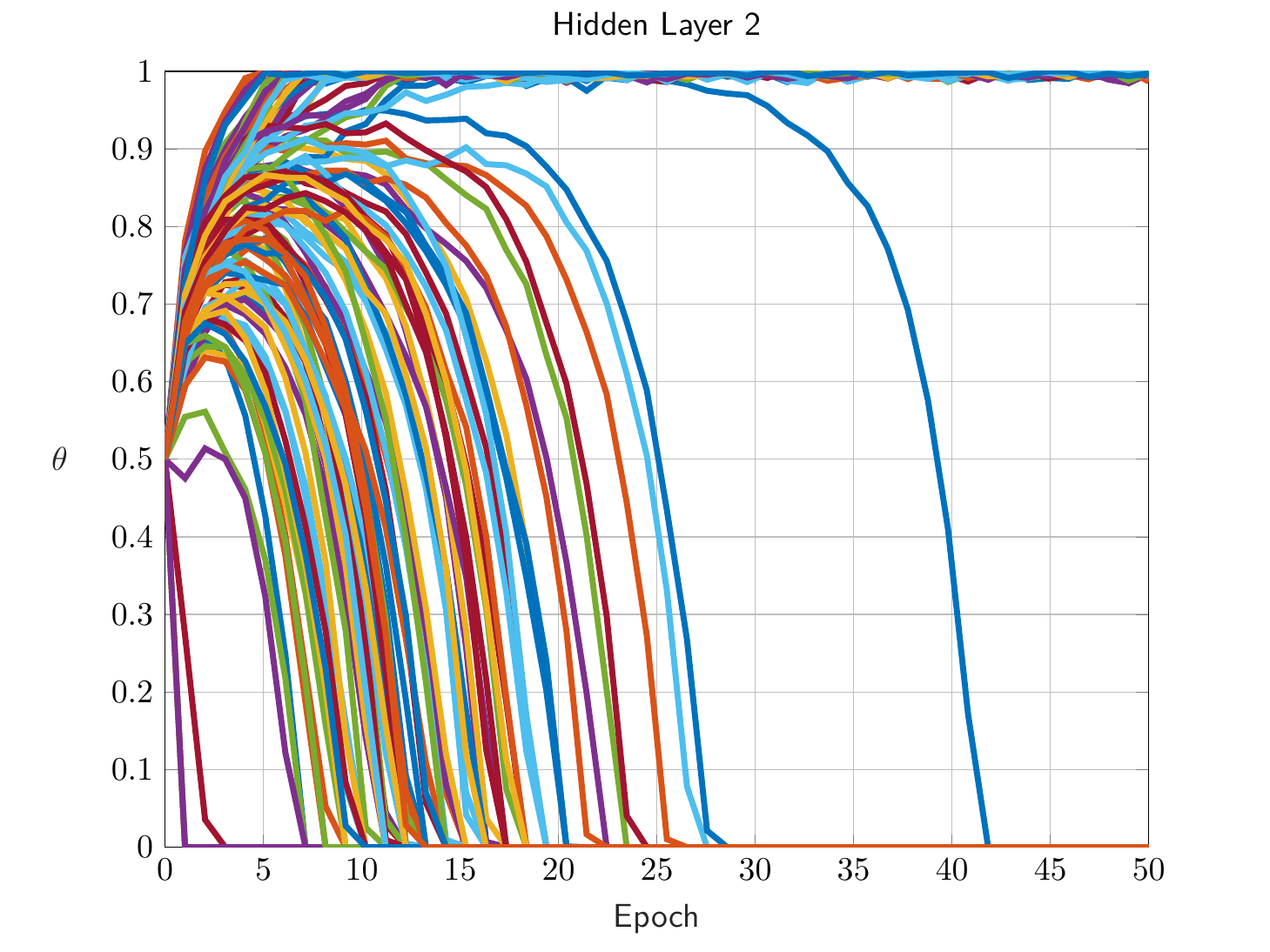}}
	\caption{Evolution of the $\theta$ values in both hidden layers during training/pruning on LeNet300-100. Most values increase initially, making the units active more often and providing them with a chance to adapt. }\label{fig:results_FCNN_thetas}
\end{figure}

In Figure~\ref{fig:results_FCNN_thetas}, we show the typical behavior of the $\theta$ values in the first and second hidden layer of the network during training. It is worth noticing that for most units and therefore, even the ones with an unfavorable weight initialization, $\theta$ increases at first indicating that the weights of their units learn to become useful. The optimal selection of $\pi^\star$ together with the careful choice of hyper-prior distribution allows units to recover after an unfavorable initialization, rather than being pruned prematurely. Eventually, as more units emerge to represent the data and the backpropagation error gets smaller, the rate of adaptation slows down and units with $C_0-C_1<-\log\gamma$  experience their $\theta$ converging to $0$ leading to their pruning.

Next, we run our algorithm with different starting hidden layer widths to evaluate the effect of initial network overparametrization to the size of the resulting pruned network.
All training and other hyper-parameters remain unchanged during this experiment.
Table \ref{tab:archs_FCNN} compares the resulting pruned network sizes and their test accuracy for stating sizes 150-50 and 50-30 with the LeNet300-100 network. Our method is able to robustly prune the networks to about the same size and accuracy with little dependence on the initial size of its layer.
This consistency demonstrates again the robustness of our approach in its ability to find the appropriate size network for a given desired accuracy.
In contrast, SAL fails to prune the network to consistent sizes as can be seen by the much higher standard deviations and the vastly varying mean values for different initial network sizes. Since SAL does not decouple the pruning process from initialization of weights, $\theta$ values and the starting architecture well, we see that the resulting networks are smaller when the initial architecture was smaller. In the extreme case of the $50 - 30$ network this leads to a poorly performing final network.
\begin{table}[h!]%h!
	%\vspace{0.25cm}
	\centering
	\begin{tabular}{lclclclc}
			Method & \hspace{-0.05cm}Start Architecture & \hspace{-0.15cm}End Architecture & \hspace{-0.25cm}Test Accuracy [\%]& \hspace{-0.15cm}Pruning Ratio [\%] \\
		\hline\parbox[t]{0mm}{\multirow{2}{*}{\rotatebox[origin=c]{0}{Baseline}}} &$300-100$  & $300-100$& $98.46\pm0.09$& -\Tstrut\\
		&$50-30$ &\centering$50-30$ & $97.99 \pm 0.08$& -\Bstrut\\
		\hline
		{\multirow{3}{*}{\parbox{0cm}{Ours,\\(Flattening,\\Taylor)}}}
        &$300-100$&	$49.5$\mbox{\tiny$\pm1.86$}$-29.5$\mbox{\tiny$\pm1.50$}&$98.13\pm 0.07$&$87.59\pm0.43$ \\
        & $150-50$  &$49.5$\mbox{\tiny$\pm0.67$}$-30.4$\mbox{\tiny$\pm1.28$}&$98.11 \pm 0.08$&$73.55\pm0.37$\\
		&$50-30$ & $45.6$\mbox{\tiny$\pm1.20$}$-26.8$\mbox{\tiny$\pm1.33$}& $97.98 \pm 0.08$&$24.85\pm2.02$\\
		\hline
		\parbox[t]{8mm}{\multirow{3}{*}{\rotatebox[origin=c]{0}{SAL}}}&$300-100$  & $66.3$\mbox{\tiny$\pm9.12$}$-54.1$\mbox{\tiny$\pm7.31$}&$98.20\pm0.09$&$82.93\pm2.25$ \Tstrut\\
		&$150-50$  &$44.2$\mbox{\tiny$\pm10.88$}$-33.1$\mbox{\tiny$\pm3.36$}&$97.83\pm0.21$&$76.45\pm7.18$\\
		&$50-30$ &$22.0$\mbox{\tiny$\pm3.19$}$-17.8$\mbox{\tiny$\pm1.94$}& $96.96\pm0.22$&$72.80\pm6.31$\\
		\hline
	\end{tabular}
	\centering
	\caption{Resulting architectures and test accuracy of the learning/pruning method starting from different sized initial networks similar to LeNet300-100.\label{tab:archs_FCNN}}
\end{table}

\subsubsection{LeNet5}
LeNet5 is a convolutional neural network consisting of two convolutional layers with 6 and 16 5-by-5 filters, respectively, pooling layers and two fully connected layers with 120 and 84 units \cite{le_cun_1998_doc_rec}.
We choose $\log(\gamma)=-100$ in case of the Flattening hyper-prior and $\alpha=0.9, \beta=\expnumber{1}{33}$ in case of the Beta hyper-prior in all layers of the network.
Table \ref{tab:LeNet5} summarizes the resulting networks, their test accuracies and pruning ratios. Again, our method is able to robustly prune the network to virtually the same size starting from different weight initializations. All reported standard deviations are small. We achieve a pruning ratio of over $92\%$ while maintaining test accuracy of around $99\%$, which is only a $0.2\%$ decrease from the accuracy of the full LeNet5 architecture. The Flattening hyper-prior with $\log(\gamma)=-100$ prunes the network to smaller sizes as compared to the Beta hyper-prior with $\beta=\expnumber{1}{33}$. Indeed, this value for $\beta$ is the highest we can tolerate before encountering numerical problems when evaluating gradients in our implementation. Higher values of the $\beta$ hyper-parameter are necessary for the Beta hyper-prior to better approximate the ``flat'' character of the Flattening hyper-prior and prevent aggressive pruning initially. The difficulty in achieving the required level for $\beta$ becomes more severe as the data set size increases and the level of the regularization curve in Figure \ref{fig:regu_theta_all_three} needs to decrease. Therefore, the effectiveness of the Beta hyper-prior could be limited in certain problems. We also observe in the LeNet5 experiment that the CONCRETE approximation to the difference $C_0-C_1$ is showing less aggressive pruning when compared to the Taylor approximation. Pruning condition (ii) in Algorithm \ref{alg:learning_algo} with $\theta_{per}=0.1$ and $n_0=2814$ (3 Epochs) leads to slightly more aggressive pruning as compared to pruning condition (i). Again, SAL leads to much higher standard deviations and fails to prune especially the last layer to minimal size, while lacking considerably in test accuracy.
\begin{table}[!thb] %!thb
	\begin{center}
		\begin{tabular}{lclclclc}
			Method & \hspace{-0.25cm}Learned Architecture & \hspace{-0.5cm}Test Accuracy [\%]& \hspace{-0.2cm}Pruning Ratio [\%] \\
			\hline
			Baseline & $6-16-120-84$ & $99.20\pm0.08$ & -\\
			\hline
			Flattening:&&&\\
			Condition (i) &  $3.7$\mbox{\tiny$\pm0.64$}$-8.8$\mbox{\tiny$\pm1.16$}$-15.8$\mbox{\tiny$\pm1.16$}$-9.5$\mbox{\tiny$\pm0.50$} & $98.97\pm 0.07$ & $92.44\pm0.93$\\
			Condition (ii) &  $3.5$\mbox{\tiny$\pm0.67$}$-8.4$\mbox{\tiny$\pm0.92$}$-14.6$\mbox{\tiny$\pm0.92$}$-9.5$\mbox{\tiny$\pm0.67$} & $98.93\pm 0.09$ & $93.23\pm0.63$\\
			CONCRETE &  \hspace{-0.3cm}$4.1$\mbox{\tiny$\pm0.54$}$-10.8$\mbox{\tiny$\pm0.98$}$-16.2$\mbox{\tiny$\pm0.98$}$-10.3$\mbox{\tiny$\pm0.90$} & $98.98\pm 0.07$ & $90.42\pm1.05$\\
			\hline
			Beta &  $4.0$\mbox{\tiny$\pm0.63$}$-9.7$\mbox{\tiny$\pm1.01$}$-16.3$\mbox{\tiny$\pm1.42$}$-9.8$\mbox{\tiny$\pm0.60$} & $98.99\pm 0.06$ & $91.39\pm0.63$\\
			SAL &\hspace{-1.5cm} $5.43$\mbox{\tiny$\pm0.49$}$-10.14$\mbox{\tiny$\pm3.18$}$-18.71$\mbox{\tiny$\pm13.20$}$-22.71$\mbox{\tiny$\pm13.55$} & $92.95\pm 14.67$ & $87.33\pm8.63$\\% $\theta_{init}=0.9, 0.16$
			\hline
		\end{tabular}
	\end{center}
\caption{Resulting architecture and test accuracy of the learning/pruning on the LeNet5 architecture using different forms of hyper-prior and methods to approximate the difference $C_0-C_1$.\label{tab:LeNet5}}
\end{table}
Next, Figures~(\ref{fig:results_CNN_units}), (\ref{fig:results_CNN_loss}) and (\ref{fig:results_CNN_tacc}) show how the total number of hidden units/features, the training loss and the test accuracy evolve during training. The left most plot shows again that our method maintains a high number of units/features for about 4 epochs. During this time, these network structures are allowed to adapt and become useful. In the case of LeNet5, where each layer carries a relatively small number of units, the need to provide poorly initialized units a grace period to adapt and not prune them immediately becomes especially critical. Otherwise, the process becomes unstable and good networks can not be found consistently. Our method is able to reduce the total number of hidden units and features of the network by a factor of 6 reliably while maintaining high test accuracy. The majority of units/filters have been pruned after about 15 epochs when using pruning condition~(i) or about 8 epochs when using pruning condition~(ii), leaving only a small network to train on for about 20 more epochs until convergence. In comparison, training of the baseline network (no pruning) takes about 25 epochs to converge.\\
\begin{figure}[t]%!htb
	\centering
	\begin{subfigure}{0.32\textwidth}
		\resizebox{\linewidth}{!}{\includegraphics{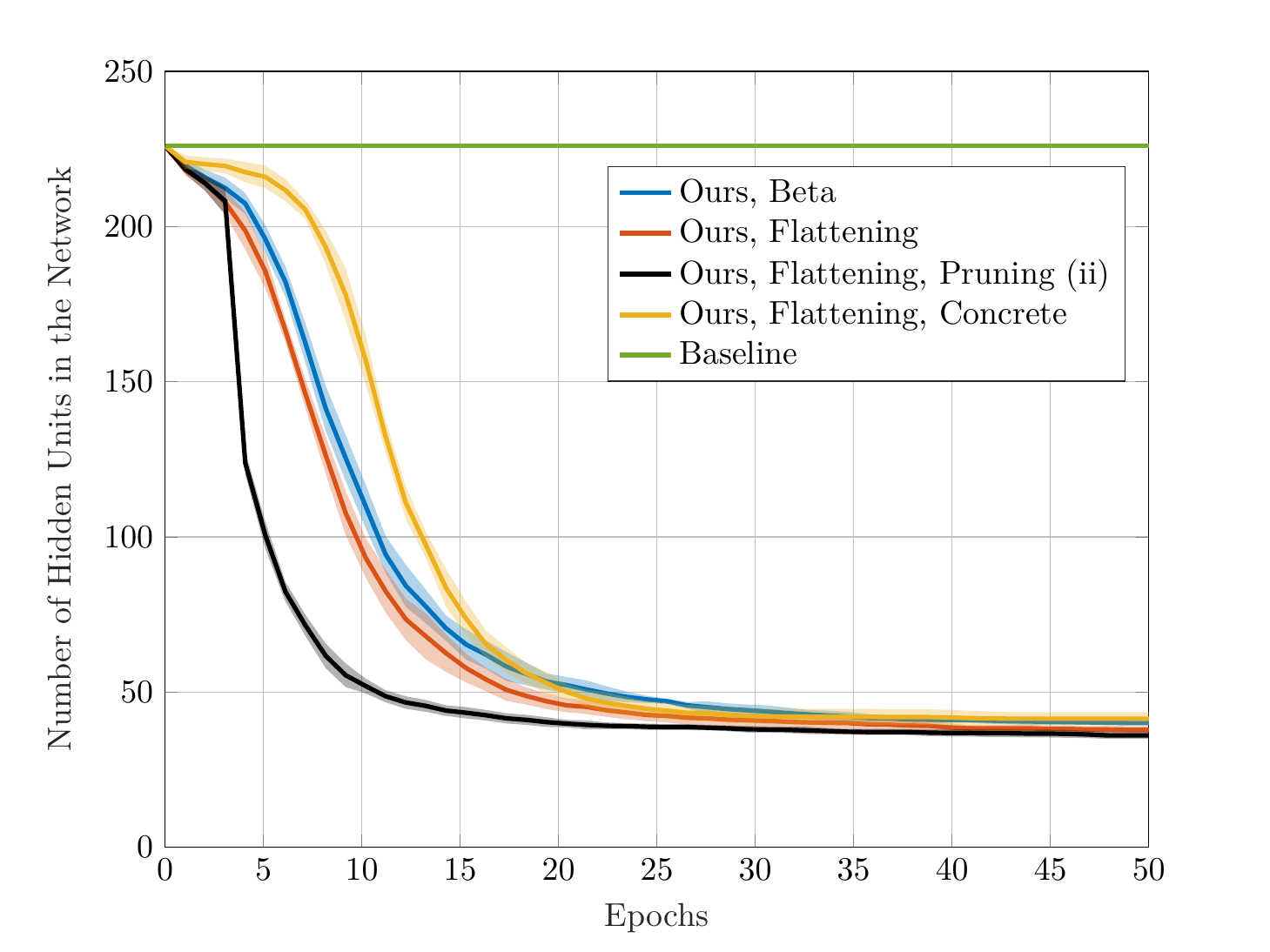}}
		\caption{Number of Units} \label{fig:results_CNN_units}
	\end{subfigure}
	\begin{subfigure}{0.32\textwidth}
		\resizebox{\linewidth}{!}{\includegraphics{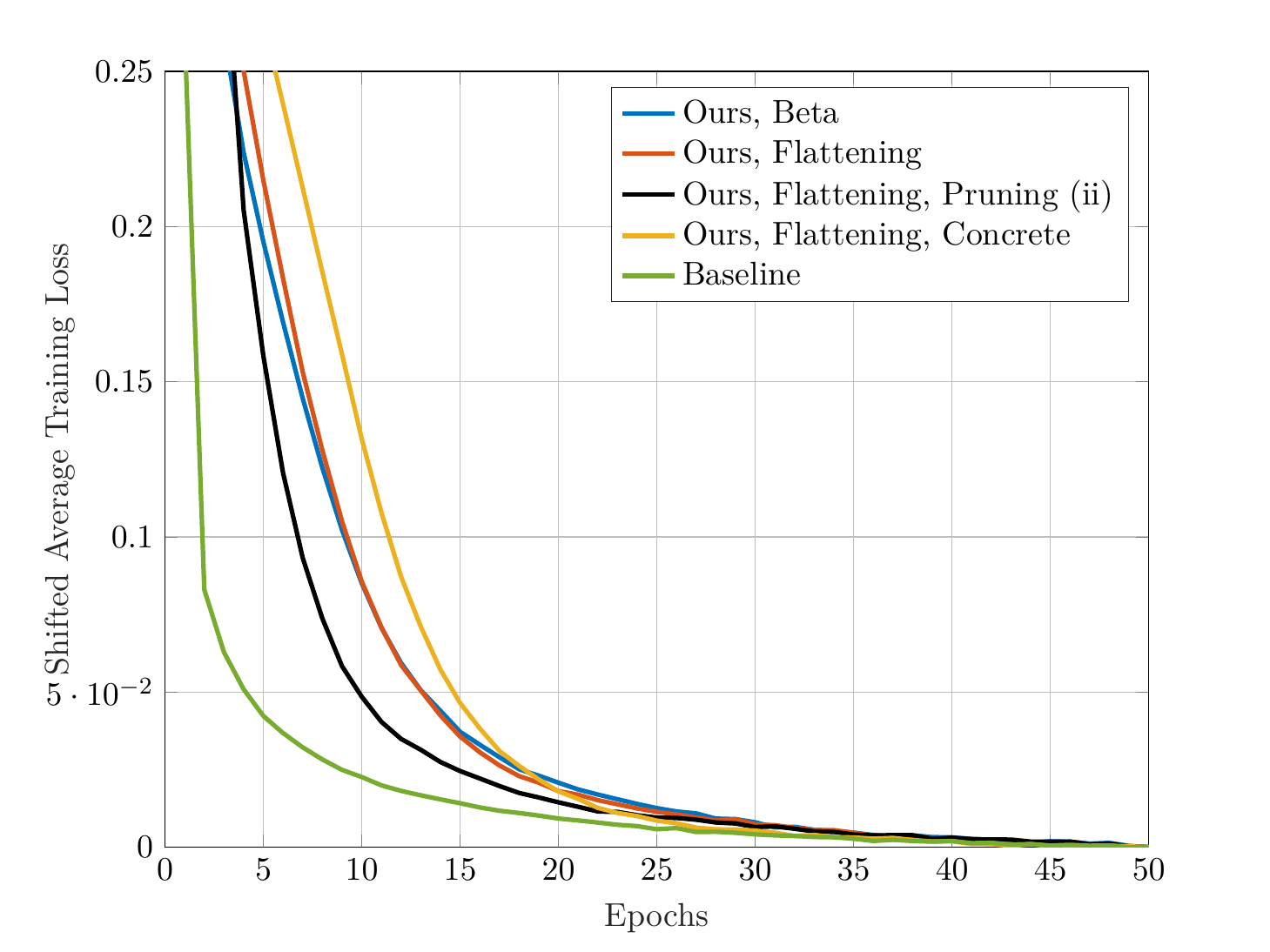}}
		\caption{Loss} \label{fig:results_CNN_loss}
	\end{subfigure}
	\begin{subfigure}{0.32\textwidth}
		\resizebox{\linewidth}{!}{\includegraphics{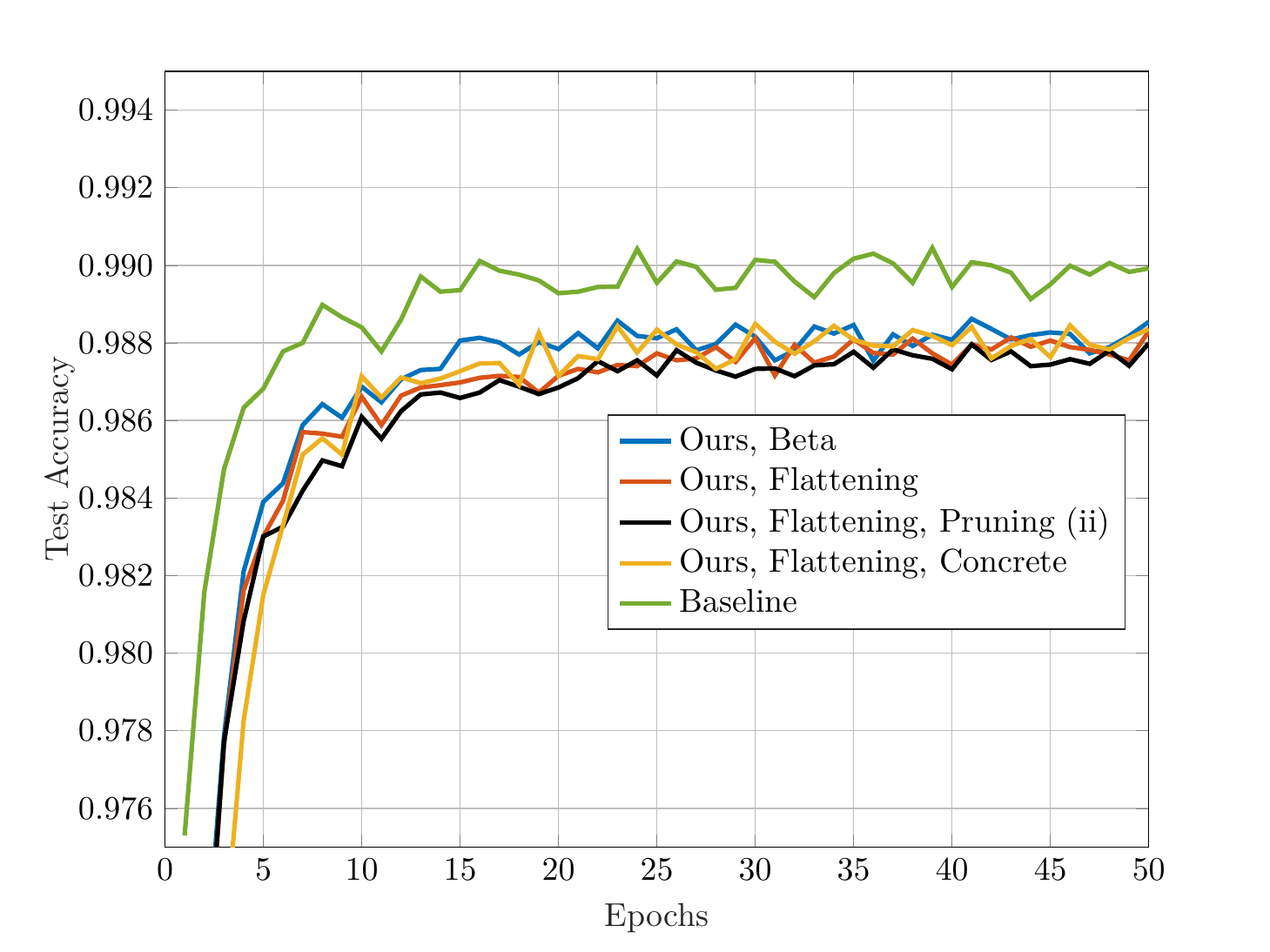}}
		\caption{Accuracy} \label{fig:results_CNN_tacc}
	\end{subfigure}
	\caption{Total number of hidden units in the network (a), training loss (b) and test accuracy (c) during the 50 epochs of training/pruning on LeNet5.}\label{fig:results_CNN_training}
\end{figure}

\subsubsection{Comparison with state-of-the-art methods}
In this section, we compare our approach with some recent, state-of-the-art methods for structured pruning on the MNIST data set. More specifically, we evaluate our method against the Stochastic Architecture Learning (SAL) in \cite{srinivas_generalized_2016}, Provable Filter Pruning (PFP) in \cite{liebenwein_pruning_19}, Filter Thresholding (FT) in \cite{li_pruning_2017} and SoftNet in \cite{he_soft_2018}. The comparison results  are shown in Table~\ref{tab:CNN_FCNN_pruning_results}. Results for PFP, FT and SoftNet are listed as reported in \cite{liebenwein_pruning_19}, where no standard deviations were given. To our knowledge, results for SAL were not available in the literature for the network structures considered and were obtained from our own implementation of SAL.  More specifically, SAL is also based on a variational approach to adapt the dropout probabilities $\theta$ and utilizes a Beta hyper-prior. However, a key difference between our approach (when using the Beta hyper-prior) and SAL is how the $\pi$ parameters are selected; instead of solving for the optimal $\pi^*$ as in Section~\ref{sec:beta_hyper_prior}, SAL takes the arbitrary choice $\pi^*=\theta$. In our implementation of SAL, besides adhering to this choice, we took  $\alpha=0.099$ and $\beta=10\alpha$ as the hyper-parameters of the Beta distribution satisfying $\beta/\alpha=10$ as recommended in \cite{srinivas_generalized_2016}. Furthermore, to achieve competitive pruning levels with SAL, we initialized $\theta=0.08$ for the LeNet300-100 network and $\theta=0.9$ in the first convolutional layer and $\theta=0.16$ in all other layers for the LeNet5 network.
We observed that the performance of SAL was sensitive to initialization resulting in uncompetitive results for a variation of less than $0.05$ in the above initialization values of $\theta$. We also observed that SAL was sensitive to weight initialization and for the LeNet5 network it could find successful networks only in 7 out of the 10 runs (using the same initial weights in each run as for our method.) Thus,  the reported statistics for SAL in Table~\ref{tab:CNN_FCNN_pruning_results} come from the 7 successful runs.

Table~\ref{tab:CNN_FCNN_pruning_results} shows that our method is able to robustly prune the network to smaller sizes while maintaining higher accuracy compared to the best of previous methods for both, the fully connected LeNet300-100 and the convolutional LeNet5 architectures.
We emphasize that a key characteristic of our simultaneous pruning and training approach is its robustness, i.e., its ability to deliver consistent results in pruning effectiveness and accuracy regardless of the starting architecture (number of units per layer) and initial values of weights. SAL in \cite{srinivas_generalized_2016}, while also being based on Bayesian principles and relying on variational inference techniques to unify the learning and pruning process, does not provide this robustness in our experiments as can be seen by the high standard deviations reported in Table~\ref{tab:CNN_FCNN_pruning_results}. Unfortunately, we do not have data to judge the robustness properties of the other methods.
We do not report results for BC-GNJ and BC-GHS from \cite{Louizos_2017} here, because in their experiments with LeNet-300-100 the input layer was also subject to pruning in addition to the hidden layers, and with LeNet-5 a different starting architecture of $20-50-800-500$ filters/units is used, making their MNIST results not comparable with ours. Also
in \cite{liebenwein_lost_2021}, the architectures resulting after pruning for PFP, FT and SoftNet are not reported.
\begin{table}[h]
	%\vspace{0.25cm}
	\centering
	\begin{tabular}{c| c |c c}
		[\%] & Method & Accuracy & Pruning Ratio\\
		\hline
		\parbox[t]{3mm}{\multirow{5}{*}{\rotatebox[origin=c]{90}{LeNet-300-100}}} &Unpruned  & $98.46\pm0.09$& -\Tstrut\\
		&Ours  &$98.13\pm0.07$ &$87.59 \pm 0.43$\\
		&SAL &$98.20\pm0.09$ &$82.93\pm2.25$\\
		&PFP & $98.00$&$84.32$\\
		&FT & $98.06$ &$81.68$\\
		&SoftNet &$98.00$&$81.69$\Bstrut\\
		\hline
		\parbox[t]{3mm}{\multirow{5}{*}{\rotatebox[origin=c]{90}{LeNet-5}}}&Unpruned  &$99.20\pm 0.08$ & -\Tstrut\\
		&Ours  & $98.97 \pm 0.07$ & $92.44 \pm 0.93$\\
		&SAL &$\,\,\,92.95\pm14.67$ &$87.33\pm8.63$\\
		&PFP& $98.93$& $92.37$\\
		&FT &  $98.81$ &$85.04$\\
		&SoftNet  & $98.88$ & $80.57$ \\
	\end{tabular}
	\centering
\caption{Accuracy and pruning ratio of different state of the art techniques for structured pruning evaluated on the MNIST digit-classification problem. All quantities are in [\%], higher is better. Comparing our method using the Flattening hyper-prior, Taylor approximation and pruning condition (i) from Algorithm \ref{alg:learning_algo} to SAL \cite{srinivas_generalized_2016}, PFP \cite{liebenwein_pruning_19}, FT \cite{li_pruning_2017} and SoftNet \cite{he_soft_2018}. The results for PFP, FT, and SoftNet are quoted from \cite{liebenwein_pruning_19}.} \label{tab:CNN_FCNN_pruning_results}
\end{table}

\subsection{VGG16 on CIFAR-10 Results}
We use our algorithm to train and prune the VGG16 network structure \cite{simonyan_deep_2015}, which consists of 13 convolutional and 3 dense layers, summing up to over $15$ million parameters for the CIFAR-10 32x32 color images. We also use batch normalization layers with the ``center=false'' option between each two layers.
All $50000$ training images are used during training and the network is evaluated on the $10000$ test images of the 10-class data set. Here, our method is only used with the flattening hyper-prior with Taylor approximation of $C_0-C_1$ and pruning condition (i) with parameter $\theta_{tol}=\expnumber{1}{-3}$. The following training parameters and hyper-parameters are used: Data set size $N=50000$, Mini-batch size $B=128$, $\mathcal{L}_2$-Regularization parameter $\lambda= 25$. For our method, we initialize all $\theta$ parameters to $0.5$ and for SAL, we initialize all $\theta$ parameters to $0.1$. We train the network for 300 epochs using stochastic gradient descent with momentum (with momentum parameter of $0.9$), starting at a learning rate of $0.05$ and reducing it every 30 epochs by a factor of $2$. The networks obtained after 300 epochs are saved and evaluated on the test data set in terms of their structure, accuracy and pruning ratio.

Table \ref{tab:CIFAR10_pruning_results} shows the training and pruning results of our method, SAL and several other state-of-the-art methods for structured pruning: Provable Filter Pruning (PFP) in \cite{liebenwein_pruning_19}, Filter Thresholding (FT) in \cite{li_pruning_2017}, SoftNet in \cite{he_soft_2018} as well as BC-GNJ and BC-GHS in \cite{Louizos_2017}.
Our method is able to robustly prune the network to smaller sizes while maintaining higher accuracy compared to the best of previous methods for VGG16.
Again, we emphasize that our robust method yields 4 times smaller standard deviation for the pruning ratio when compared to SAL. Here one should note that a network with pruning ratio $95.5\%$ has only half the number of parameters as a network with a pruning ratio $91\%$. Our method finds networks consisting of about $700000$ parameters after 300 epochs of training. However, after the first 30 epochs of training the network was already pruned to about one tenth of its initial size or to about $1.5$ million parameters, leading to large computational savings in all following training iterations and of course during inference.
\begin{table}[h]
	%\vspace{0.25cm}
	\centering
	\begin{tabular}{c| c |c |c |c}
		[\%] & Method & Accuracy & Pruning Ratio& Baseline Accuracy\\
		\hline
		\parbox[t]{3mm}{\multirow{5}{*}{\rotatebox[origin=c]{90}{VGG16}}}
		%&Unpruned  &$92.78\pm 0.25$ & -\Tstrut\\
		&Ours  & $92.66 \pm 0.24$ & $95.63 \pm 0.15$ & $92.92\pm 0.25$\Tstrut\\
		&SAL &$90.30\pm0.34$ &$91.05\pm0.61$ & $92.92\pm 0.25$\\
		&PFP& $92.39$& $94.32$& $92.89$\\
		&FT &  $91.78$ &$80.09$& $92.89$\\
		&SoftNet  & $92.08$ & $63.95$& $92.89$ \\
		&BC-GNJ  & $91.40$ & $93.30$ & $91.60$\\
		&BC-GHS  & $91.00$ & $94.50$ & $91.60$
	\end{tabular}
	\centering
	\caption{Accuracy and pruning ratio of different state of the art techniques for structured pruning evaluated on the CIFAR-10 image-classification problem using the VGG16 network architecture \cite{simonyan_deep_2015}. All quantities are in [\%], higher is better. Comparing our method using the Flattening hyper-prior and pruning condition (i) from Algorithm \ref{alg:learning_algo} to SAL \cite{srinivas_generalized_2016}, PFP \cite{liebenwein_pruning_19}, FT \cite{li_pruning_2017}, SoftNet \cite{he_soft_2018} and BC-GNJ as well as BC-GHS \cite{Louizos_2017}. The results for PFP, FT, and SoftNet are quoted from \cite{liebenwein_pruning_19}. The Results for BC-GNJ and BC-GHS are quoted from \cite{Louizos_2017}.} \label{tab:CIFAR10_pruning_results}
\end{table}

\newpage
Table \ref{tab:CIFAR10_arch} shows the resulting networks from of our method, SAL, as well as the BC-GNJ and BC-GHS methods in \cite{Louizos_2017}. Our network is able to prune each layer to appropriate size independent of its starting size. The resulting network structure is narrower towards both ends and wider in the intermediate layers. Using SAL, we find that the size of each layer is dependent on the initial size and roughly pruned to one third of it. Due to the careful choice of the flattening hyper-prior, the need to initialize $\theta$ at particular values (e.g. at $\theta=0.1$ for SAL) to achieve good pruning ratios, is eliminated and the resulting algorithm decouples the pruning from the initial value of $\theta$ and the initial layer sizes. BC-GNJ and BC-GHS yield networks with tails heavily pruned while the first few layers, which are responsible for the bulk of the computational load, are not narrowed much. As a consequence, although the network architecture using BC-GHS has $24\%$ more parameters than ours, the computational load to forward-propagate through this network is $71\%$ higher than with our network.

\begin{table}[!thb] %!thb
	\begin{center}
		\begin{tabular}{lclclclc}
			Method & \hspace{-0.25cm}Learned Architecture  \\
			\hline
			Baseline &\hspace{-0.6cm} $64\,\text{-}\,64\,\text{-}\,128\,\text{-}\,128\,\text{-}\,256\,\text{-}\,256\,\text{-}\,256\,\text{-}\,512\,\text{-}\,512\,\text{-}\,512\,\text{-}\,512\,\text{-}\,512\,\text{-}\,512\,\text{-}\,512\,\text{-}\,512$ \\
			\hline
			Ours &\hspace{-0.6cm}
			$16.7\mbox{\tiny$\pm1.49$}\,\text{-}\,47.9\mbox{\tiny$\pm3.35$}\,\text{-}\,99.6\mbox{\tiny$\pm5.74$}\,\text{-}\,104.6\mbox{\tiny$\pm6.42$}\,\text{-}\,160.6\mbox{\tiny$\pm3.95$}\,\text{-}\,123.2\mbox{\tiny$\pm4.10$}\,\text{-}\,79.3\mbox{\tiny$\pm4.72$}\,\text{-}\,73.4\mbox{\tiny$\pm4.67$}\,\text{-}\,$\Tstrut\\
			&$\,\text{-}\,41.3\mbox{\tiny$\pm2.75$}\,\text{-}\,22.7\mbox{\tiny$\pm1.57$}\,\text{-}\,25.3\mbox{\tiny$\pm2.67$}\,\text{-}\,18.7\mbox{\tiny$\pm1.42$}\,\text{-}\,20.0\mbox{\tiny$\pm1.41$}\,\text{-}\,23.0\mbox{\tiny$\pm2.00$}\,\text{-}\,30.1\mbox{\tiny$\pm1.60$}$ \\
			\hline
			SAL &\hspace{-0.6cm}
			$23.8\mbox{\tiny$\pm3.66$}\,\text{-}\,19.5\mbox{\tiny$\pm4.58$}\,\text{-}\,44.5\mbox{\tiny$\pm5.25$}\,\text{-}\,40.8\mbox{\tiny$\pm9.26$}\,\text{-}\,85.7\mbox{\tiny$\pm12.37$}\,\text{-}\,77.9\mbox{\tiny$\pm6.31$}\,\text{-}\,79.8\mbox{\tiny$\pm7.81$}\,\text{-}\,153.4\mbox{\tiny$\pm19.02$}\,\text{-}\,$\Tstrut\\
			&$\,\text{-}\,145.6\mbox{\tiny$\pm12.76$}\,\text{-}\,145.2\mbox{\tiny$\pm12.26$}\,\text{-}\,154.6\mbox{\tiny$\pm8.98$}\,\text{-}\,156.2\mbox{\tiny$\pm14.43$}\,\text{-}\,140.0\mbox{\tiny$\pm8.18$}\,\text{-}\,169.4\mbox{\tiny$\pm9.50$}\,\text{-}\,227.9\mbox{\tiny$\pm9.99$}$ \\
			\hline
			BC-GNJ &\hspace{-0.6cm} $63\,\text{-}\,64\,\text{-}\,128\,\text{-}\,128\,\text{-}\,245\,\text{-}\,155\,\text{-}\,63\,\text{-}\,26\,\text{-}\,24\,\text{-}\,20\,\text{-}\,14\,\text{-}\,12\,\text{-}\,11\,\text{-}\,11\,\text{-}\,15$ \\
			BC-GHS &\hspace{-0.6cm} $51\,\text{-}\,62\,\text{-}\,125\,\text{-}\,128\,\text{-}\,228\,\text{-}\,129\,\text{-}\,38\,\text{-}\,13\,\text{-}\,9\,\text{-}\,6\,\text{-}\,5\,\text{-}\,6\,\text{-}\,6\,\text{-}\,6\,\text{-}\,20$ \\
			\hline
		\end{tabular}
	\end{center}
	\caption{Resulting architecture and test accuracy of the learning/pruning on the VGG16 architecture for our method using the Flattening hyper-prior and pruning condition (i) and SAL. The Results for BC-GNJ and BC-GHS are quoted from \cite{Louizos_2017}.\label{tab:CIFAR10_arch}}
\end{table}

\section{Conclusions}\label{sec:conclusions}
Deep Neural Networks often require excessive computational requirements during training and inference. To address this issue, we have proposed a novel structured pruning algorithm that operates simultaneously with the weight learning process. Based on Bayesian variational inference peinciples, our method learns the distributions over Bernoulli random variables multiplying structures such as units in fully connected or filters in convolutional networks and acting like unit-wise adaptive dropout. In this way, automatic pruning is effected during the training phase and  is signaled by the parameters of the variational Bernoulli distribution converging to $0$ rendering the corresponding unit/filter permanently inactive. The Bernoulli parameters of surviving structures converge to $1$ resulting is a smaller, deterministic network.

An important consideration for the pruning and prediction accuracy performance is that these parameters do not converge prematurely, e.g.,  due to an unfavorable initialization of weights while at the same time it is desirable to prune irrelevant structures from the network as early as appropriate to save computational effort in future training iterations.
To this end, we establish desired properties of the hyper-prior distributions over the parameters controlling the prior distributions of the Bernoulli random variables based on the dynamics of the learning process and analytically derive a novel ``Flattening''  hyper-prior distribution possessing these properties; this hyper-prior has only one parameter that can be transparently used for trading-off pruning levels vs. prediction accuracy.
In this manner, consistent pruning results are achieved regardless of the initialization of network weights and the level of overparametrization in the starting network.

We showed that the additional gradients needed to learn the variational parameters can be calculated or approximated efficiently using backpropagation. We analyzed the underlying ODE system of the resulting stochastic gradient descent algorithm using Lyapunov stability theory in Theorem~\ref{thm:stability} and obtained theoretical conditions under which a variational parameter corresponding to a unit/filter and its corresponding weights converge to $0$. These results were tied to the proposed learning/pruning algorithm using stochastic approximation theory in Theorem~\ref{thm:Alg_conv} and then used to suggest practical pruning conditions as part of our algorithm.

We evaluated the proposed learning/pruning algorithm on the MNIST and CIFAR-10 data sets using common LeNet and the VGG16 architectures. Our structured pruning method is able to reduce the total number of weights to a level on par or better than competing state-of-the-art methods while achieving higher test-accuracy. Most importantly, our experiments confirm that this performance is achieved in a robust way with respect to weight initialization and initial architecture size and that our algorithm can identify and prune irrelevant structures of the network during the early stages of training. Thus, significant computational load during the remainder of the training as well as during inference can be saved.
%number only refered eqs (in appendix), comment out to show all numbers in appendix
\mathtoolsset{showonlyrefs,showmanualtags}

\appendix
%renewcomand, but need to renew again before \subsection
%careful! now we refer by using just \ref{app:A} insead  of Appendix \ref... but not for the subsections in appendix
\renewcommand\thesection{\appendixname\ \Alph{section}}
\section{Conditions (\ref{eq:kbnd}) and (\ref{eq:ebnd}) in Theorem~\ref{thm:stability}}\label{app:A} %\eqref{eq:kbnd} and \eqref{eq:ebnd}
%renew again
\renewcommand\thesection{\Alph{section}}
In this Appendix,  we establish the existence of constants $\kappa$ and $\eta$ such that
\begin{align}\label{eq:bnds}
&|C_1-C_0| \leq \kappa\cdot \phi\quad{\rm and}\quad
%\bar \sigma\left(\E\left[ \delta_l z_b^\top \right]\right)\leq \eta
\bar\sigma( M_1+ M_2^\top)\leq \eta
\end{align}
%with $M_1$ and $M_2$ defined in (\ref{eq:Sys_GrDes})
hold as required by the assumptions of Theorem~\ref{thm:stability}. We assume that all activation functions $a_l(\cdot)$ of the network have derivatives bounded by 1, i.e., $|a_l'(\cdot)|\leq 1$ and satisfy $a_l(0)=0$. Typical activation functions such as ReLU and the hyperbolic tangent conform with this assumption. We also assume that the given data has bounded moments as follows: $\E_{x\sim p(x)}\norm{x}^k \leq S_{xk} < \infty$, $k=1,2,3,4$ and  $\E_{y\sim p(y|x)}\norm{y}^k\leq S_{yk} < \infty$, $k=1,2$.

\subsection{$|C_1-C_0| \leq \kappa\cdot \phi$ Bound}\label{sec:kappabound}
First, note that under the assumption that the gradients of all activation functions in the network are bounded by $1$ in absolute value, the Lipschitz constant $\mathcal{L}_{NN}$ of a neural network with weights ${W^l},\, {1\leq l\leq L}$ with $L$ layers exists and can be bounded by (\cite{szegedy_intriguing_2014, scaman_lipschitz_2019})
	\begin{align} \label{eq:NN_lipschitz}
	\mathcal{L}_{NN} \leq \left(\frac{1}{L}\sum_{l=1}^{L} \norm{W^l}_F^2\right)^{\frac{L}{2}}.
	\end{align}
%\end{lemma}
Next, consider the unit $z$ in layer $l$ of the neural network as depicted in Figure~\ref{fig:NN_struct}. The
activation signal in the last layer, i.e., the network's output before applying the output activation function can be thought of as the output of a neural network consisting of the last $L-l$ layers (and without output activation) of the full network, specifically
\begin{align}
\zeta^L(\xi) =NN_f\left(\bar W_f \left(\bar z\odot\bar \xi\right) + w_f\cdot z \cdot \xi\right),
\end{align}
with $z=a_{l-1}(w_b^\top z_b)$, $\bar W_f$ consisting of the columns in $W^l$ except $w_f$ and similarly $\bar z$ consisting of the elements in $z^l$ except $z$ and $\bar\xi$ the RV multiplying $\bar z$.
 Similarly $z_b=NN_b([x ; 1])$ can be thought of as the output of a neural network $NN_b$ consisting of the first $l-1$ layers of the full network with input $[x ; 1]$.
Therefore,  we can bound
\begin{align}
\norm{\zeta^L(\xi=1, \bar \Xi )-\zeta^L(\xi=0 , \bar \Xi )} \leq \mathcal{L}_{NN_f} \norm{w_f \cdot z}
\end{align}
and also
\begin{align}\label{eq:zb_bound}
\norm{z_b-NN_b([0;1])}
\leq \mathcal{L}_{NN_b}\cdot \norm{x}\ \Rightarrow \norm{z_b}\leq \mathcal{L}_{NN_b}\cdot \norm{x} + B_1
\end{align}
where $B_1$ is a uniform bound on $\norm{NN_b([0;1])}$ for all realizations of the RVs $\Xi$ and bounded weights $W^l$.
Furthermore,
\begin{align}
\norm{w_f\cdot z} = \norm{w_f}\cdot |z| \leq \norm{z_b}\cdot \norm{w_b}\cdot \norm{w_f} \leq \norm{z_b} \cdot\underbrace{\frac{1}{2}(\|w_b\|^2+\|w_f\|^2)}_{\equalhat\phi}
\end{align}
and combining the previous bounds yields
\begin{align}\label{eq:phibnd1}
\norm{\zeta^L(\xi=1, \bar \Xi )-\zeta^L(\xi=0 , \bar \Xi )} \leq \mathcal{L}_{NN_f}\left(\mathcal{L}_{NN_b} \norm{x}+B_1\right) \cdot \phi.
\end{align}
We can also bound the output of the full network when the input is $x$ as follows:
\begin{align}\label{eq:yhat_bound}
	\norm{\hat y(x,\Xi)} \leq \mathcal{L}_{NN}\cdot\norm{x} + B_2
\end{align}
where $B_2$ is a uniform bound on $\norm{NN([0;1])}$ for all realizations of the RVs $\Xi$ and bounded weights $W^l$. Indeed, for the regression case with linear output activation $\hat y\equiv\zeta^L$ and \eqref{eq:yhat_bound} follows in a similar manner with \eqref{eq:zb_bound}. For the $K$-class classification case with softmax output activation $\norm{\hat y}\leq K$, which implies a fortiori \eqref{eq:yhat_bound} by setting $B_2=\max\{B_2,K\}$.

Next, we show that the network output activation together with the loss function $l(y,\hat y)=-\log p(y \mid \hat y)$ satisfy for both the regression and $K$-class classification problems considered:
\begin{align}\label{eq:lbnd}
	\left|l_1-l_0\right| &= \left|l\left(y,\hat y_1\right) - l\left(y,\hat y_0\right)\right|\leq \left(\mathcal{L}_{NN}\|x\|+\|y\|+B_2\right)\cdot\norm{\zeta_1^L - \zeta_0^L},
\end{align}
where we defined $\hat y_1\equalhat y(x, \xi=1, \bar \Xi)$, $\hat y_0\equalhat y(x, \xi=0, \bar \Xi)$
and $\zeta_1^L\equalhat \zeta^L(x, \xi=1, \bar \Xi)$, $\zeta_0^L\equalhat \zeta^L(x, \xi=0, \bar \Xi)$ for brevity.  More specifically, in the regression case, we assume linear output activation and $p(y \mid \hat y) \sim \mathcal{N}(y;\hat y, I)$.
Then,
\begin{align}
	\left|l_1-l_0\right| = \frac{1}{2}\left|\norm{y-\hat y_1}^2- \norm{y-\hat y_0}^2\right|
\leq \frac{1}{2}\norm{\hat y_1+\hat y_0-2y}\cdot\norm{\hat y_1 - \hat y_0}
\end{align}
and \eqref{eq:lbnd} follows using \eqref{eq:yhat_bound} and since in this case $\hat y\equiv\zeta^L$. In the $K$-class classification case, we consider the softmax output activation and $p(y \mid \hat y) \sim \prod_{k=1}^K \hat y^{y{i,k}}$. Then, viewing the loss function as a function of $\zeta^L$ allows to write
\begin{align}
	\left|l_1-l_0\right| \leq \max_{\zeta^L}\norm{\nabla l(\zeta^L)}\cdot\norm{\zeta_1^L-\zeta_0^L}
\end{align}
and using the well known expression $\nabla l(\zeta^L)=y-\hat y$  and \eqref{eq:yhat_bound} gives \eqref{eq:lbnd}.

Next, by substituting \eqref{eq:phibnd1} in \eqref{eq:lbnd}, we obtain
\begin{align}
	|l_1-l_0| \leq  \left(\mathcal{L}_{NN}\|x\|+\|y\|+B_2\right)
\mathcal{L}_{NN_f}\left(\mathcal{L}_{NN_b} \norm{x}+B_1\right) \cdot \phi
\end{align}
and further (see \eqref{eq:Chats_sampling}):
\begin{align}\label{eq:Kappa_instance}
\begin{split}
%|\hat C_1- \hat C_0| = \left|\frac{N}{B}\sum_{i=1}^B l_1 - \frac{N}{B}\sum_{i=1}^B l_0\right|\leq N \left(\mathcal{L}_{NN}\|x\|+\|y\|+B_2\right)
%\mathcal{L}_{NN_f}\left(\mathcal{L}_{NN_b} \norm{x}+B_1\right) \cdot \phi,
\hspace{-0.5cm}|\hat C_1- \hat C_0| = \left|\frac{N}{B}\sum_{i=1}^B \left(l_1-l_0\right)\right|\leq N \left(\mathcal{L}_{NN}\|x\|+\|y\|+B_2\right)
\mathcal{L}_{NN_f}\left(\mathcal{L}_{NN_b} \norm{x}+B_1\right) \cdot \phi,
\end{split}
\end{align}
which holds for any instance of $\bar \Xi$ and data-point $x,y$.
Finally, taking expectation with respect to the RV $\bar \Xi$ and the data $\mathcal{D}$
gives
\begin{align}\label{eq:Kappa_ex}
|C_1-C_0|=\left|E_{\bar\Xi,\mathcal{D}}[\hat C_1-\hat C_0]\right|\leq E_{\bar\Xi,\mathcal{D}}[|\hat C_1-\hat C_0|] \leq \underbrace{N\mathcal{L}_{NN_f}\cdot\gamma_0}_\kappa\cdot \phi
\end{align}
as required, where we defined
\begin{align}\label{eq:gam0}
\gamma_0&\equalhat E_{\mathcal{D}}[\left(\mathcal{L}_{NN}\|x\|+\|y\|+B_2\right)
\left(\mathcal{L}_{NN_b} \norm{x}+B_1\right)]<\infty
%\\&=\mathcal{L}_{NN}S_x +\mathcal{L}_{NN_b}\mu_x\mu_y+(B_1\mathcal{L}_{NN}+B_2\mathcal{L}_{NN_b})\mu_x+B_1\mu_y+B_1B_2.
 \end{align}
given our assumption on the moments of $x$ and $y$.

\subsection{$\bar \sigma(M_1+M_2^\top)\leq \eta$ Bound}\label{sec:deltabound}
Consider $\hat M_1(\xi=1, \bar \Xi)=A_1 \delta_f z_b^\top$,  $\hat M_2(\xi=1, \bar \Xi)=A_2 \delta_f z_b^\top$ with $A_1 = \frac{a_{l-1}(w_b^\top z_b)}{w_b^\top z_b}\leq \gamma_1$, $A_2 =a_{l-1}'(w_b^\top z_b)~\leq~\gamma_2$ being sector-bounded and positive for the choices of activation functions under consideration.
%First, recall the neural network equations in \eqref{eq:NN_def}
%as well as the back-propagated error $\delta_f=\delta^l$ and notice that $\delta^{L+1}=\hat y-y$ for the output loss functions and activations considered. Then, we can bound the norm of the back-propagated error in layer $l$ as follows
From the back-propagation equations
\begin{align}\label{eq:bp}
\delta^l={\rm diag}\{a_l'(\zeta^l)\}\cdot\left(W^{l+1}\right)^\top\cdot\left(\delta^{l+1}\odot\xi^{l+1}\right)
\end{align}
and since $\delta^{L+1}=\hat y-y$ for the output loss and activation functions considered, we can bound the norm of the back-propagated error $\delta_f=\delta^l$ in layer $l$ as follows
\begin{align}
\norm{\delta_f(\xi=1, \bar \Xi)} \leq \norm{W^{l+1}}_F\cdot \norm{W^{l+2}}_F \dots \norm{W^{L}}_F\cdot \norm{\hat y(x,\xi=1,\bar \Xi)-y}
\end{align}
by recursively applying \eqref{eq:bp} and using the submultiplicativity of the Frobenius norm.
We then obtain
\begin{align}
\begin{split}
%\bar \sigma \left(   \delta_f(\xi=1, \bar \Xi) |z_b|^\top   \right) &\leq
%%\norm{\E\left[\norm{\delta_f\Bigr|_{(\xi=1)}}_F \norm{z_b}\right]}_F\\
\bar \sigma\left(\delta_f(\xi=1, \bar \Xi)z_b^\top\right) = \norm{\delta_f(\xi=1, \bar \Xi)}\cdot\norm{z_b}
&\leq 	\prod_{l+1}^{L}\norm{W^l}_F \cdot\norm{\hat{y}(x,\xi=1,\bar \Xi)-y}\cdot \norm{z_b}\\
&\leq   \gamma_3\cdot\left(\norm{\hat{y}(x,\xi=1,\bar \Xi)}+\norm y\right)\cdot
 \left(\mathcal{L}_{NN_b}\norm{x}+B_1\right)
\end{split}
\end{align}
using \eqref{eq:zb_bound} and where $\gamma_3$ bounds $\prod_{l+1}^{L}\norm{W^l}_F$.
Further, the sector-bounds on the activation function give
\begin{align}
	\bar \sigma(\hat M_1) &= \bar \sigma \left(     A_1\delta_f  z_b^\top          \right) \leq
	\gamma_1 \bar \sigma\left(\delta_f z_b^\top\right)
	 \quad \text{and}\quad
	 \bar \sigma(\hat M_2) = \bar \sigma \left(     A_2\delta_f  z_b^\top          \right) \leq
	 \gamma_2 \bar \sigma\left(\delta_f z_b^\top\right)
\end{align}
and combining these results yields
\begin{align}\label{eq:barsig_bnd}
		\bar \sigma(\hat M_k) &\leq  \gamma_k
\gamma_3\cdot\left(\norm{\hat{y}(x,\xi=1,\bar \Xi)}+\norm y\right)\cdot
 \left(\mathcal{L}_{NN_b}\norm{x}+B_1\right)\nonumber\\
 &\leq \gamma_k
\gamma_3\cdot\left(\mathcal{L}_{NN}\norm{x}+\norm{y}+B_2\right)\cdot
 \left(\mathcal{L}_{NN_b}\norm{x}+B_1\right),\quad k=1,2
\end{align}
using \eqref{eq:yhat_bound}.
Next, taking expectation with respect to the RV $\bar \Xi$ and the data $\mathcal{D}=\lbrace (x_i,y_i)\rbrace_{i=1}^N$ gives
\begin{align}\label{eq:M12h_bnd}
 \E_{\bar \Xi, \mathcal{D}}\left[\bar\sigma(\hat M_k)\right] \leq \gamma_0\gamma_k\gamma_3 \equalhat \eta_k,\quad k=1,2,
\end{align}
with $\gamma_0\equalhat E_{\mathcal D}[\left(\mathcal{L}_{NN}\norm{x}+\norm{y}+B_2\right)\cdot
 \left(\mathcal{L}_{NN_b}\norm{x}+B_1\right)]$ from \eqref{eq:gam0} and further
\begin{align}\label{eq:Esig_bnd}
 \bar \sigma(M_k) &= \bar \sigma\left(\E_{\bar \Xi, \mathcal{D}}\left[\hat M_k\right]\right)\leq \E_{\bar \Xi, \mathcal{D}}\left[\bar\sigma(\hat M_k)\right]\leq \eta_k,\quad k=1,2.
\end{align}
Finally, we arrive at
\begin{align}\label{eq:M12_bnd}
 \bar \sigma(M_1+M_2^\top) \leq \bar\sigma(M_1)+\bar\sigma(M_2)\leq\eta_1+\eta_2\equalhat\eta,
\end{align}
as required.
\renewcommand\thesection{\appendixname\ \Alph{section}}
\section{Proof of Theorem~\ref{thm:Alg_conv}}\label{app:B}
\begin{proof}
Consider the update rule from Algorithm~\ref{alg:learning_algo} for the fan-out weights $w_b(n)$, fan-in weights $w_b(n)$ and their update rate $\theta(n)$, expressed as
\begin{align}\label{eq:alg}
	x({n+1}) =  x(n)+a(n)g(n),\ n\geq 0
\end{align}
where $x(n)\equalhat [w_f(n)^T,\ w_b(n)^T, \theta(n)]^T$ as in the statement of Theorem~\ref{thm:Alg_conv},
\begin{align}
	g(n) = \begin{bmatrix}
	-\hat \xi\cdot\left(\frac{N}{B}\sum_{i=1}^{B} \frac{a_l(z_{b,i}^\top w_b)}{z_{b,i}^\top w_b}\delta_{f,i}z_{b,i}^\top\right) w_b - \lambda w_f \\[2ex] -\hat \xi\cdot\left(\frac{N}{B}\sum_{i=1}^{B}  a_l'(z_{b,i}^\top w_b)z_{b,i}\delta_{f,i}^\top\right) w_f - \lambda w_b \\[2ex]
	-\frac{N}{B} \sum_{i=1}^B \log \left(\frac{p(y_i \mid x_i,W,\hat \xi=0, \hat{\bar\Xi})}{p(y_i \mid x_i,W,\hat \xi=1, \hat{\bar\Xi})}\right)-\log \left(\frac{\theta(1-\pi^\star)}{(1-\theta)\pi^\star}\right)
	\end{bmatrix},
\end{align}
is a realization of the negative gradient of $L(W,\Theta)$ with respect to $x(n)$  and $a(n)$ is the stepsize. Notice that we omitted the iteration index $n$ in the notation for the weights $w_b(n), w_f(n)$ and parameters $\theta(n), \pi^\star(n)$ in the above equation for $g(n)$ for the sake of brevity. Also, notice that $z_{b,i}$ and $\delta_{f,i}$ are functions of the samples $\hat{\bar\Xi}$ and the data, and $\delta_{f,i}$ is computed with $\xi=1$.
We express \eqref{eq:alg} equivalently as
\begin{align}\label{eq:SA_borkar}
	x({n+1}) =  x(n)+a(n)\left(h(x({n})) +  M({n+1})\right),\ n\geq 0
\end{align}
where
\begin{align}\label{eq:happ}
	h(x)={\begin{bmatrix}
	-\left(\frac{\partial L}{\partial w_f}\right)^\top, & -\left(\frac{\partial L}{\partial w_b}\right)^\top,
    & -\left(\frac{\partial L}{\partial \theta}\right)^\top
	\end{bmatrix}}^\top
\end{align}
and we have implicitly defined
\begin{align}
	M({n+1})\equalhat g(n) - h(x).
\end{align}
Then, Theorem~\ref{thm:Alg_conv} follows immediately  from a result in \cite{borkar2009stochastic}. p.15  for the stochastic recursion (\ref{eq:SA_borkar}). First, we quote a set of assumptions for this result to hold from \cite{borkar2009stochastic}, pp.10-11:
\begin{itemize}
	\item[A1.] The map $h: \mathbb{R}^d \rightarrow \mathbb{R}^d$ is Lipschitz: $\norm{h(x)-h(y)}\leq L\norm{x-y}$ for some $0<L<\infty$.
	\item[A2.] Stepsizes $\{a(n)\}$ are positive scalars satisfying \begin{align}
		\sum_n a(n)=\infty, \quad \sum_n a(n)^2 < \infty.
	\end{align}
	\item[A3.] $\{M({n})\}$ is a martingale differene sequence with respect to the increasing family of $\sigma$-fields
	\begin{align}
		\mathcal{F}(n) \equalhat \sigma(x(m), M(m), m\leq n) = \sigma(x(0),M(1),\dots,M(n)), \,n\geq 0.
	\end{align}
	That is,\begin{align}
		\E\left[M({n+1})\rvert \mathcal{F}(n)\right]=0 \,\, a.s., \, n\geq0.
	\end{align}
	Furthermore,  $\{M(n)\}$ are square integrable with
	\begin{align}
		\E\left[\norm{M({n+1})}^2\rvert \mathcal{F}(n)\right] \leq K_B\left(1+\norm{x(n)}^2\right) \, a.s., \, n\geq 0,
	\end{align}
	for some constant $K_B>0$.
	\item[A4.] The iterates of \eqref{eq:SA_borkar} remain bounded $a.s.$, i.e,
	\begin{align}
		\sup_n \norm{x(n)}<\infty, \, a.s.
	\end{align}
\end{itemize}
Next, we state from \cite{borkar2009stochastic}, p.15:
\begin{theorem}{[Theorem 2 in \cite{borkar2009stochastic}, p.15}]\label{thm:bokar_convergence}
Assume that Conditions A1-A4 hold. Almost surely, the sequence $\{x(n)\}$ generated by \eqref{eq:SA_borkar} converges to a (possibly sample path dependent) compact connected internally chain transitive invariant set of the ODE:
\begin{align}\label{eq:ODE_borkar1}
	\mathit{\dot x(t) = h(x(t)),\ \ t\geq 0.}
\end{align}
%\hspace*{\fill} $\blacksquare$
\end{theorem}
First, notice that with $h(x)$ defined as in \eqref{eq:happ}, ODE \eqref{eq:ODE_borkar1}
%$\dot x(t) = h(x(t)), t\geq 0$ in the above theorem
matches the ODE system considered and analyzed in Theorem~\ref{thm:stability} (compare with \eqref{eq:dyn_sys}). %with region of attraction $\mathcal{A}$ as in \eqref{eq:RoA}.
We note further that under the assumption in Theorem~\ref{thm:Alg_conv} that the sequence $x(n)$ enters and remains within a region of attraction of the asymptotically stable equilibrium point $x^*=[w_f=0, w_b=0, \theta=\epsilon_1]^T$ of \eqref{eq:ODE_borkar1}, $x^*$ becomes by Theorem~\ref{thm:stability} the only chain transitive invariant set of the ODE in the aforementioned region of attraction. Therefore,  the almost sure convergence of $x(n)$ to $x^*$ is established once the assumptions~\eqref{eq:kbnd} to~\eqref{eq:cond_stability} of Theorem~\ref{thm:stability} and Assumptions A1 to A4 above are shown to hold for our algorithm. We also remark that by assuming the region of attraction to be contained within the projection region, we assure that the projection step does not interfere with the stochastic recursion as stated above.

Assumptions \eqref{eq:kbnd} and \eqref{eq:ebnd} hold for Algorithm~\ref{alg:learning_algo} as shown in \ref{app:A}. Assumption~\eqref{eq:cond_prior} is satisfied for both, the Beta and Flattening hyper-prior by Lemma~\ref{lemma:hyper_prior_characterization}.
Assumption~\eqref{eq:cond_stability} is satisfied by appropriate choice of $\epsilon_1$ in Algorithm~\ref{alg:learning_algo}.
%, in Algorithm \ref{alg:learning_algo}.}\\
Assumption~A2 is satisfied by choosing the stepsize $a(n)$ as required. Also, Assumption~A4 is clearly satisfied since the iterates $x(n)$ are assumed to remain within the bounded region $\phi\leq\phi_{max}$, $\theta\in[\theta_l,\ \theta_h]$. A close examination of the proof of Theorem~\ref{thm:bokar_convergence} in \cite{borkar2009stochastic} shows that it is sufficient to establish the Lipschitz property in Assumption~A1 only locally in the bounded region that the iterates $x(n)$ remain. Further, from (\ref{eq:dyn_sys}) we have:
	\begin{align}\label{eq:lip}
    h(x)=
	\begin{bmatrix}
	 \begin{bmatrix}
	-\lambda I &  -\theta M_1 \\ -\theta M_2 & -\lambda I
	\end{bmatrix}\begin{bmatrix}
	w_f \\  w_b
	\end{bmatrix}\\
     -(C_1-C_0) + \log\left[\frac{(1-\theta)\pi^\star}{\theta(1-\pi^\star)}\right]
    \end{bmatrix},
	\end{align}
and since sums of (locally) Lipschitz functions or products of bounded (locally) Lipschitz functions are (locally) Lipshitz, it suffices to show that $M_1$, $M_2$, $C_0$, $C_1$ and $f(\theta)\equalhat\log\left[\frac{(1-\theta)\pi^\star}{\theta(1-\pi^\star)}\right]$ are locally Lipshitz functions of $w_f$, $w_b$, and $\theta$. Clearly, $f(\theta)$ is locally Lipschitz since it has a bounded derivative in $\theta\in[\theta_l,\ \theta_h]$ for both the Beta and Flattening hyper-priors. Also, $M_k=E_{\Xi,\mathcal{D}}[A_k\delta_fz_b^T]$, $k=1,2$ are Lipschitz as the expectation of products of bounded Lipschitz functions. Finally, $C_0$ is independent of $w_f$, $w_b$, and $\theta$, while $\nabla_{w_f}C_1$ and $\nabla_{w_b}C_1$ are bounded functions of $w_f$ and $w_b$, therefore Lipschitz.

In the following, we verify Assumption~A3 in the context of our algorithm.
Let $M({n+1})=[M_f^\top,\ M_b^\top,\ M_\theta]^\top$ where $M_f$, $M_b$, and $M_\theta$ are defined in an obvious way. Then using the law of total expectation,
\begin{align}\label{eq:Mwf}
\begin{split}
&\E_{\xi,\bar\Xi,{\cal D}}[M_f|{\cal F}(n)]= \E_\xi \E_{\bar\Xi,{\cal D}}[M_f\mid\xi,{\cal F}(n)]=\\
&\E_\xi\left[\rule{0cm}{0.75cm}\right.-\xi\cdot\underbrace{ \E_{\bar\Xi,{\cal D}} \left[\frac{N}{B}\sum_{i=1}^{B} \frac{a_l(z_{b,i}^\top w_b)}{z_{b,i}^\top w_b}\delta_{f,i}z_{b,i}^\top\mid\xi,{\cal F}(n)\right]}_{\equalhat M_1}\left.\rule{0cm}{0.75cm}\right] w_b - \lambda w_f-\left(-\theta M_1 w_b-\lambda w_f\right)=\\
&-\underbrace{E[\xi]}_{\equalhat\theta}M_1w_b+\theta M_1 w_b=0.
\end{split}
\end{align}
Also,
\begin{align}
\begin{split}
&\E_{\xi,\bar\Xi,{\cal D}}[\|M_f\|^2\mid{\cal F}(n)]=
E_{\bar\Xi,{\cal D}}\E_\xi [\|M_f\|^2\mid\bar\Xi,{\cal D},{\cal F}(n)]=\\
&\E_{\bar\Xi,{\cal D}}\E_\xi\left[\|\-\xi\hat M_1 w_b - \lambda w_f-\left(-\theta M_1 w_b-\lambda w_f\right)\|^2\mid\bar\Xi,{\cal D},{\cal F}(n)\right]=\\
&\E_{\bar\Xi,{\cal D}}\left[\E_\xi[\xi^2]\|\hat M_1 w_b\|^2 -2\theta \E_\xi[\xi]w_b^\top M_1^\top\hat M_1w_b +\theta^2\|M_1 w_b\|^2\mid{\cal D},{\cal F}(n)\right]=\\
%&\theta \E_{\bar\Xi,{\cal D}}\left[\|\hat M_1 w_b\|^2\mid\bar\Xi,{\cal D},{\cal F}(n)\right] -\theta^2\|M_1 w_b\|^2.
&\theta \E_{\bar\Xi,{\cal D}}\left[\|\hat M_1 w_b\|^2\mid{\cal F}(n)\right] -\theta^2\|M_1 w_b\|^2,
\end{split}
\end{align}
where $\hat M_1$ is a sample of $\frac{a_l(z_{b}^\top w_b)}{z_{b}^\top w_b}\delta_{f}z_{b}^\top$ from the distribution of $\Xi$ and $\cal D$ with $\xi=1$ having expected value equal to $M_1$.
Next, using \eqref{eq:barsig_bnd} and that $\|\hat M_1 w_b\|^2 \leq \bar\sigma(\hat M_1)^2\norm{w_b}^2$, we bound
%\begin{align}\label{eq:Mfbnd}
%	\E_{\bar\Xi,{\cal D}}\left[\|\hat M_1 w_b\|^2\mid\bar\Xi,{\cal D},{\cal F}(n)\right] &\leq \E_{{\cal D}}\left[(\gamma_1\gamma_3)^2\left(\mathcal{L}_{NN}\norm x +\norm y +B_2\right)^2\left(\mathcal{L}_{NN_b}\norm x+B_1\right)^2\mid{\cal D}\right]\norm{w_b}^2 \\&\leq (\gamma_1\gamma_3)^2\gamma_{00}\norm{w_b}^2,
%\end{align}
\begin{align}\label{eq:Mfbnd}
	\E_{\bar\Xi,{\cal D}}\left[\|\hat M_1 w_b\|^2\mid{\cal F}(n)\right] &\leq \E_{{\cal D}}\left[(\gamma_1\gamma_3)^2\left(\mathcal{L}_{NN}\norm x +\norm y +B_2\right)^2\left(\mathcal{L}_{NN_b}\norm x+B_1\right)^2\right]\norm{w_b}^2\nonumber \\
&\leq (\gamma_1\gamma_3)^2\gamma_{00}\norm{w_b}^2,
\end{align}
where
\begin{align}\label{eq:gamma00}
\gamma_{00}&\equalhat E_{\cal D}\left[\left(\mathcal{L}_{NN}\|x\|+\|y\|+B_2\right)^2
\left(\mathcal{L}_{NN_b} \norm{x}+B_1\right)^2\right]<\infty
\end{align}
given our assumption on the moments of $x$ and $y$.
Then, \eqref{eq:Mfbnd} and \eqref{eq:Esig_bnd} yield
%it holds $E_{\bar\Xi,\mathcal D}\left[\bar\sigma(\hat M_1)\right] <\eta_1$ and also $\bar\sigma(M_1)\leq\eta_1$.  Therefore, we obtain
\begin{equation}
 \E_{\xi,\bar\Xi,{\cal D}}[\|M_f\|^2\mid{\cal F}(n)]\leq \left((\gamma_1\gamma_3)^2\gamma_{00}+\eta_1^2\right)\cdot\|w_b\|^2
\end{equation}
and $M_f$ satisfies Assumption~A3 by taking $K_B= \left((\gamma_1\gamma_3)^2\gamma_{00}+\eta_1^2\right)$. In a parallel manner, it can be shown that
\begin{align}
\begin{split}
&\E_{\xi,\bar\Xi,{\cal D}}[M_b|{\cal F}(n)]=0\\
&\E_{\xi,\bar\Xi,{\cal D}}[\|M_b\|^2\mid{\cal F}(n)]\leq \left((\gamma_2\gamma_3)^2\gamma_{00}+\eta_2^2\right)\cdot\|w_f\|^2
\end{split}
\end{align}
and $M_b$ satisfies Assumption~A3 by taking $K_B= \left((\gamma_1\gamma_2)^2\gamma_{00}+\eta_2^2\right)$. Next,
\begin{align}
\begin{split}
&\E_{\xi,\bar\Xi,{\cal D}}[M_\theta|{\cal F}(n)]=\\
& \E_{\xi,\bar\Xi,{\cal D}}\left[-(\hat C_1 - \hat C_0)+\log \left(\frac{\theta(1-\pi^\star)}{(1-\theta)\pi^\star}\right) - \left(-(C_1 - C_0) + \log \left(\frac{(1-\theta)\pi^\star}{\theta(1-\pi^\star)}\right)\right)\mid {\cal F}(n)\right]=\\
& \E_{\bar\Xi,{\cal D}}[\hat C_0-C_0\mid{\cal F}(n)] -\E_{\bar\Xi,{\cal D}}[\hat C_1-C_1\mid{\cal F}(n)]=0
\end{split}
\end{align}
where $\hat C_0$ and $\hat C_1$ as in  \eqref{eq:Chats_sampling} are samples of $-\log p(y_i \mid x_i,W, \xi=\lbrace 0,1\rbrace, {\bar\Xi})$ from the distribution of $\Xi$ and $\cal D$ when $\xi=0$ and $\xi=1$ and having expected value equal to $C_0$ and $C_1$, respectively.
It follows
\begin{align}
\E_{\xi,\bar\Xi,{\cal D}}[M_\theta^2\mid{\cal F}(n)]&= \text{Var}\left(-(\hat{C_1}-\hat{C_0})\mid \mathcal{F}(n)\right) \\&= \left| \E_{\bar\Xi,{\cal D}}\left[(\hat C_1-\hat C_0)^2\mid{\cal F}(n)\right]-(C_1-C_0)^2\right| \leq (N^2\mathcal{L}_{NN_f}^2\gamma_{00}+\kappa^2)\cdot \phi^2,
\end{align}
from  \eqref{eq:Kappa_instance} and \eqref{eq:Kappa_ex} in \ref{app:A} and using $\gamma_{00}$ from \eqref{eq:gamma00}. Thus, $M_\theta$ also satisfies Assumption~A3 by taking $K_B= (N^2\mathcal{L}_{NN_f}^2\gamma_{00}+\kappa^2)\cdot \phi_{max}^2$ and the proof is complete.
\end{proof}

%references go at the end
%\bibliography{citations}
%\bibliographystyle{ieeetr}

\end{document}